\documentclass{article}

\usepackage{microtype}
\usepackage{graphicx}
\usepackage{subfigure}
\usepackage{booktabs} % for professional tables
\usepackage{algorithm}
\usepackage{algorithmic}
\usepackage{amsmath}
\usepackage{amsfonts}
\usepackage{amssymb}
\usepackage{amsbsy}
\usepackage{fixmath}
\usepackage{bm}
\usepackage{amsthm}
\usepackage[T1]{fontenc}

\usepackage{enumitem}
\setlist[enumerate]{itemsep=0mm}

\DeclareMathOperator*{\argmin}{arg\,min}

\usepackage{hyperref}

\newtheorem{theorem}{Theorem}

\newtheorem{lemma}{Lemma}
\newtheorem{definition}{Definition}

\usepackage[accepted]{icml2018}

\allowdisplaybreaks

\begin{document}

\onecolumn
%\icmltitle{Multi-Learner Stackelberg Game for Adversarial Regression
%Problems}
\icmltitle{Adversarial Regression with Multiple Learners}



\icmlsetsymbol{equal}{*}

\begin{icmlauthorlist}
\icmlauthor{Liang Tong}{equal,to} 
\icmlauthor{Sixie Yu}{equal,to}
\icmlauthor{Scott Alfeld}{goo}
\icmlauthor{Yevgeniy Vorobeychik}{to}
\end{icmlauthorlist}

%\icmlaffiliation{equal}{Tong and Yu have an equal contribution}
\icmlaffiliation{to}{Department of EECS, Vanderbilt University, Nashville, TN, USA}
\icmlaffiliation{goo}{Computer Science Department, Amherst College, Amherst, MA, USA}

\icmlcorrespondingauthor{Yevgeniy Vorobeychik}{yevgeniy.vorobeychik@vanderbilt.edu}

\icmlkeywords{Machine Learning, ICML}

\vskip 0.3in


\printAffiliationsAndNotice{\icmlEqualContribution} % otherwise use the standard text.

\begin{abstract}
Despite the considerable success enjoyed by machine learning
techniques in practice, numerous studies demonstrated that many
approaches are vulnerable to attacks.
An important class of such attacks involves
adversaries changing features at test time to cause incorrect
predictions.
Previous investigations of this problem pit a single learner against
an adversary.
However, in many situations an adversary's decision is aimed at a
collection of learners, rather than specifically targeted at each
independently.
We study the problem of adversarial linear regression with multiple
learners.
We approximate the resulting game by exhibiting an upper bound on learner
loss functions, and show that the resulting game has a unique symmetric
equilibrium.
We present an algorithm for computing this equilibrium, and show
through extensive experiments that equilibrium models are
significantly more robust than conventional regularized linear regression.
\end{abstract}

\section{Introduction}
\label{sec:introduction}

Increasing use of machine learning in adversarial settings has
motivated a series of efforts investigating the extent to which
learning approaches can be subverted by malicious parties.
An important class of such attacks involves adversaries changing their
behaviors, or features of the environment, to effect an incorrect
prediction.
Most previous efforts study this problem as an interaction between a
single learner and a single
attacker~\citep{Bruckner11,Dalvi04,li2014feature,kdd2012}.
However, in reality attackers often target a broad array of potential
victim organizations.
For example, they craft generic spam templates and generic malware, and then disseminate
these widely to maximize impact.
The resulting ecology of attack targets reflects not a single learner,
but many such learners, all making autonomous decisions about how to
detect malicious content, although these decisions often rely on
similar training datasets.

We model the resulting game as an interaction between multiple
learners, who simultaneously learn linear regression models, and an
attacker, who observes the learned models (as in white-box
attacks~\cite{laskov2014practical}), and modifies the original feature
vectors at test time in order to induce incorrect predictions.
Crucially, rather than customizing the attack to each learner (as in
typical models), the attacker chooses a single attack for \emph{all}
learners.
We term the resulting game a \emph{Multi-Learner Stackelberg Game}, to
allude to its two stages, with learners jointly acting as Stackelberg
leaders, and the attacker being the follower.
Our first contribution is the formal model of this game.
Our second contribution is to approximate this game by deriving upper
bounds on the learner loss functions.
The resulting approximation yields a game in which there always exists
a symmetric equilibrium, and this equilibrium is unique.
In addition, we prove that this unique equilibrium can be computed
by solving a convex optimization problem.
Our third contribution is to show that the equilibrium of the
approximate game is robust, both theoretically (by showing it to be
equivalent to a particular robust optimization problem), and through
extensive experiments, which demonstrate it to be much more robust to
attacks than standard regularization approaches.

\noindent{\bf Related Work } 
Both attacks on and defenses of machine learning approaches have been studied within the literature on \emph{adversarial machine learning}~\citep{Bruckner11,Dalvi04,li2014feature,kdd2012,lowd2005adversarial}.
These approaches commonly assume a single learner, and consider either the problem of finding evasions against a fixed model~\citep{Dalvi04,lowd2005adversarial,laskov2014practical}, or algorithmic approaches for making learning more robust to attacks~\citep{aisec2016,Bruckner11,Dalvi04,li2014feature,Li15b}.
Most of these efforts deal specifically with classification learning, but several consider adversarial tampering with regression models~\citep{Alfeld16,Grosshans13}, although still within a single-learner and single-attacker framework.
\citet{Stevens13} study the algorithmic problem of attacking multiple linear classifiers, but did not consider the associated game among classifiers.

Our work also has a connection to the literature on security games with multiple defenders~\citep{Laszka16,Smith17,Vorobeychik11}.
The key distinction with our paper is that in multi-learner games, the learner strategy space is the space of possible models in a given model class, whereas prior research has focused on significantly simpler strategies (such as protecting a finite collection of attack targets).

%% AML
%% adversarial regression
%% multi-learner attack paper
%% multi-defender security games

\section{Model}
\label{sec:model}

We investigate the interactions between a collection of learners
$\mathcal{N}=\{1,2,...,n\}$ and an attacker in regression problems,
modeled as a \emph{Multi-Learner Stackelberg Game (MLSG)}.
At the high level, this game involves two stages: first, all learners
choose (train) their models from data, and second, the attacker
transforms test data (such as features of the environment, at
prediction time) to achieve malicious goals.
Below, we first formalize the model of the learners and the attacker,
and then formally describe the full game.

\subsection{Modeling the Players}
\label{subsec:players}

At training time, a set of training data $(\mathbf{X}, \mathbf{y})$ is drawn from an unknown distribution $\mathcal{D}$.
$\mathbf{X} \in \mathbb{R}^{m \times d}$ is the training sample and $\mathbf{y} \in \mathbb{R}^{m \times 1}$ is a vector of values of each data in $\mathbf{X}$.
We let $\mathbf{x}_j \in \mathbb{R}^{d \times 1}$ denote the $j$th instance in the training sample, associated with a corresponding value $y_j \in \mathbb{R}$ from $\mathbf{y}$.
Hence, $\mathbf{X} = [\mathbf{x}_1,...,\mathbf{x}_m]^\top$ and $\mathbf{y} = [y_1, y_2, ... , y_m]^\top$. 
On the other hand, test data can be generated either from
$\mathcal{D}$, the same distribution as the training data, or from
$\mathcal{D}^{'}$, a modification of $\mathcal{D}$ generated by an
attacker. 
The nature of such malicious modifications is described below.
We let $\beta\ (0 \leq \beta \leq 1)$ represent the probability that a
test instance is drawn from $\mathcal{D}^{'}$ (i.e., the malicious distribution), and $1-\beta$ be the probability that it is generated from $\mathcal{D}$.

The action of the $i$th learner is to select a $d \times 1$ vector $\bm{\theta}_i$ as the parameter of the linear regression function $\hat{\mathbf{y}}_i = \mathbf{X}\bm{\theta}_i$, where $\hat{\mathbf{y}}_i$ is the predicted values for data $\mathbf{X}$. 
The expected cost function of the $i$th learner at test time is then
\begin{equation}
\label{eq:learner_true_cost}
c_i(\bm{\theta}_i,\mathcal{D}^{'}) = \beta\mathbb{E}_{(\mathbf{X}^{'},\mathbf{y})\sim\mathcal{D}^{'}}[\ell(\mathbf{X}^{'}\bm{\theta}_i,\mathbf{y})] + (1-\beta)\mathbb{E}_{(\mathbf{X},\mathbf{y})\sim\mathcal{D}}[\ell(\mathbf{X}\bm{\theta}_i,\mathbf{y})].
\end{equation}
where $\ell(\hat{\mathbf{y}},\mathbf{y}) = ||\hat{\mathbf{y}}-\mathbf{y}||_2^2$.
That is, the cost function of a learner $i$ is a combination of its expected cost from both the attacker and the honest source. 

Every instance $(\mathbf{x}, y)$ generated according
to $\mathcal{D}$ is, with probability $\beta$, maliciously modified by the attacker into another,
$(\mathbf{x}', y)$, as follows.
We assume that the attacker has an instance-specific target
$z(\mathbf{x})$, and wishes that the prediction made by each learner
$i$ on the modified instance, $\hat{y} = \bm{\theta}_i^\top\mathbf{x}^{'}$, is close to this
target.
We measure this objective for the attacker by
$\ell(\hat{\mathbf{y}}',\mathbf{z}) =
||\hat{\mathbf{y}}'-\mathbf{z}||_2^2$ for a vector of predicted and
target values $\hat{\mathbf{y}}'$ and $\hat{\mathbf{z}}$,
respectively.
In addition, the attacker incurs a cost of transforming a distribution
$\mathcal{D}$ into $\mathcal{D}^{'}$, denoted by $R(\mathcal{D}^{'},\mathcal{D})$. 

%The attacker's action is to manipulate a distribution $\mathcal{D}^{'}$ deviated from $\mathcal{D}$ such that for $(\mathbf{X},\mathbf{y})\sim\mathcal{D}$ there exists corresponding $(\mathbf{X}^{'},\mathbf{y})\sim\mathcal{D}^{'}$ with target predicted values $\mathbf{z} \in \mathbb{R}^{m\times 1}$.
%That is, for any $(\mathbf{x}_i, \mathbf{y}_i)$ in test data, it has a probability $\beta$ to be transformed into $(\mathbf{x}_i^{'}, \mathbf{y}_i)$ with $z_i$ as the predicted target of $\mathbf{x}_i$.
%Note that $z_i$ is a target value.
%The true value of instance $\mathbf{x}_i^{'}$ is still $y_i$. 
%The transformation described above incurs a cost denoted as $R(\mathcal{D}^{'},\mathcal{D})$. 
%This cost represents the difference between the original and adversarial data distributions.

After a dataset $(\mathbf{X}^{'}, \mathbf{y})$ is generated in this way by the
attacker, it is used simultaneously against all the learners.
This is natural in most real attacks: for example, spam templates are
commonly generated to be used broadly, against many individuals and
organizations, and, similarly, malware executables are often produced to be
generally effective, rather than custom made for each target.
%the attacker applies these data to each of the learners.
%If the predicted value $\hat{\mathbf{z}} = \mathbf{X}^{'}\bm{\theta}_i$ deviates from the target $\mathbf{z}$, then the attacker obtains a cost $\ell(\mathbf{X}^{'}\bm{\theta}_i, \mathbf{z})$.
The expected cost function of the attacker is then a sum of its total
expected cost for all learners plus the cost of transforming $\mathcal{D}$ into $\mathcal{D}^{'}$ with coefficient $\lambda>0$:
\begin{equation}
\label{eq:attacker_true_cost}
c_a(\{\bm{\theta}_i\}_{i=1}^n, \mathcal{D}^{'}) = \sum_{i=1}^n \mathbb{E}_{(\mathbf{X}^{'}, \mathbf{y})\sim\mathcal{D}^{'}}[\ell(\mathbf{X}^{'}\bm{\theta}_i,\mathbf{z})] + \lambda R(\mathcal{D}^{'},\mathcal{D}).
\end{equation}

As is typical, we estimate the cost functions of the learners and
the attacker using training data $(\mathbf{X}, \mathbf{y})$, which is
also used to simulate attacks.
%the empirical cost functions of the learners and attacker, instead of expected cost functions.
%Particularly we apply training data $(\mathbf{X}, \mathbf{y})$ and its deviation $(\mathbf{X}^{'}, \mathbf{y})$ to replace $\mathcal{D}$ and $\mathcal{D}^{'}$ in Equation \ref{eq:learner_true_cost} and \ref{eq:attacker_true_cost}.
%Note that $\mathbf{X}^{'}$ is not the test data but is deviated from training data $\mathbf{X}$, as test data is unavailable at training time.
Consequently, the cost functions of each learner and the attacker
are estimated by
\begin{equation}
\label{eq:learner_cost}
c_i(\bm{\theta}_i, \mathbf{X}^{'}) = \beta\ell(\mathbf{X}^{'}\bm{\theta}_i, \mathbf{y})+(1-\beta)\ell(\mathbf{X}\bm{\theta}_i, \mathbf{y})
\end{equation}
and
\begin{equation}
\label{eq:attacker_cost}
c_a(\{\bm{\theta}_i\}_{i=1}^n, \mathbf{X}^{'}) = \sum_{i=1}^n \ell(\mathbf{X}^{'}\bm{\theta}_i, \mathbf{z}) +\lambda R(\mathbf{X}^{'},\mathbf{X})
\end{equation}
where the attacker's modification cost is measured by
$R(\mathbf{X}^{'},\mathbf{X}) = ||\mathbf{X}^{'}-\mathbf{X}||_F^2$,
the squared Frobenius norm.
% of $\mathbf{X}^{'}-\mathbf{X}$.

\subsection{The Multi-Learner Stackerlberg Game}
\label{subsec:game}
We are now ready to formally define the game between the $n$ learners
and the attacker.
The MLSG has two stages:
in the first stage, learners simultaneously select their model
parameters $\bm{\theta}_i$, and
in the second stage, the attacker makes its decision (manipulating
$\mathbf{X}^{'}$) after observing the learners' model choices $\{\bm{\theta}\}_{i=1}^n$.
We assume that the proposed game satisfies the following assumptions:

%\begin{assumption}
%\label{aspt:assumption}
%The following statements are held:
\begin{enumerate}
\item Players have complete information about parameters $\beta$ (common to all learners) and $\lambda$. This is a strong assumption, and we relax it in our experimental evaluation (Section \ref{sec:experiments}), providing guidance on how to deal with uncertainty about these parameters.
%Each learner knows the training data and $\beta$ of other learners, and the target and cost coefficient $\lambda$ of the attacker. 
%The attacker knows the training data of each learner.
%\item Each learner uses the same training data and coefficient $\beta$.
\item Each learner has the same action (model parameter) space $\bm{\Theta} \subseteq
  \mathbb{R}^{d\times 1}$ which is nonempty, compact and convex. The action space of the attacker is $\mathbb{R}^{m \times d}$.
%\item Each player chooses a pure strategy.
\item The columns of the training data $\mathbf{X}$ are linearly independent.
\end{enumerate}
%\end{assumption}

%The first assumption imposes strong limitations on the players, as in
%reality, each player can have incomplete information of
%others. However, this assumption allows us to investigate the best
%case as a baseline result. 
%We relax this assumption in our experiments detailed in Section \ref{sec:experiments}. 
%The second assumption can be satisfied, as many training data of machine learning systems are public and available.
%Hence, different entities with the same learning task can share the same training data.
%The third assumption is also satisfied by most of machine learning applications, as $\bm{\theta}$ cannot be infinite.
%The last assumption is self-evident, as it is used by many Stackelberg and static games.

We use \emph{Multi-Learner Stackelberg Equilibrium} (MLSE) as the
solution for the MLSG, defined as follows.
\begin{definition}[Multi-Learner Stackelberg Equilibrium (MLSE)] An action profile $(\{\bm{\theta}_i^{*}\}_{i=1}^n, \mathbf{X}^{*})$ is an MLSE if it satisfies 
\begin{equation}
\label{eq:mlse}
\begin{split}
  \bm{\theta}_i^{*} &= \argmin_{\bm{\theta}_i \in \bm{\Theta}} c_i(\bm{\theta}_i, \mathbf{X}^{*}(\bm{\theta})), \forall i \in \mathcal{N} \\
  &\texttt{s.t. } \mathbf{X}^{*}(\bm{\theta}) = \argmin_{\mathbf{X}^{'} \in \mathbb{R}^{m \times d}} c_a(\{\bm{\theta}_i\}_{i=1}^n, \mathbf{X}^{'}).
\end{split}
\end{equation}
where $\bm{\theta}=\{\bm{\theta}_i\}_{i=1}^n$ constitutes the joint actions of the learners.
%That is, each learner's action is a best response to others,
%given that the attacker always plays the best response to the
%models chosen by all learners.
\label{def:mlse}
\end{definition}
At the high level, the MLSE is a blend between a Nash equilibrium
(among all learners) and a Stackelberg equilibrium (between the learners
and the attacker), in which the
attacker plays a best response to the \emph{observed} models
$\bm{\theta}$ chosen by the learners, and given this behavior by the
attacker, all learners' models $\bm{\theta}_i$ are mutually optimal.

%We then have the best response of the attacker as follows.
The following lemma characterizes the best response of the attacker to
arbitrary model choices $\{\bm{\theta}_i\}_{i=1}^n$ by the learners.
\begin{lemma}[Best Response of the Attacker] Given $\{\bm{\theta}_i\}_{i=1}^n$, the best response of the attacker is 
\begin{equation}
\label{eq:best_response}
\mathbf{X}^{*} = (\lambda\mathbf{X}+\mathbf{z}\sum_{i=1}^n \bm{\theta}_i^\top)(\lambda \mathbf{I}+\sum_{i=1}^n\bm{\theta}_i\bm{\theta}_i^\top)^{-1}.
\end{equation} 
\label{lemma:best_response}
\end{lemma}
\begin{proof}
We derive the best response of the attacker by using the first order condition. 
Let $\nabla_{{X}^{'}} c_a(\{\bm{\theta}_i\}_{i=1}^n, \mathbf{X}^{'})$
denote the gradient of $c_a$ with respect to $\mathbf{X}^{'}$.
Then
\[ \nabla_{{X}^{'}} c_a = 2\sum_{i=1}^n(\mathbf{X}^{'}\bm{\theta}_i-\mathbf{z})\bm{\theta}_i^\top+2\lambda(\mathbf{X}^{'}-X).\]
Due to convexity of $c_a$, let $\nabla_{{X}^{'}} c_a = \mathbf{0}$, we have
\[ \mathbf{X}^{*} = (\lambda\mathbf{X}+\mathbf{z}\sum_{i=1}^n \bm{\theta}_i^\top)(\lambda \mathbf{I}+\sum_{i=1}^n\bm{\theta}_i\bm{\theta}_i^\top)^{-1}. \]
\end{proof}
Lemma \ref{lemma:best_response} shows that the best response of the
attacker, $\mathbf{X}^{*}$, has a closed form solution, as a function
of learner model parameters $\{\bm{\theta}_i\}_{i=1}^n$.
%, and it depends on the actions of the learners.
Let $\bm{\theta}_{-i} = \{\bm{\theta}_j\}_{j \neq i}$, then $c_i(\bm{\theta}_i, \mathbf{X}^{*})$ in Eq.~\eqref{eq:mlse} can be rewritten as
\begin{equation}
\label{eq:c_i}
c_i(\bm{\theta}_i, \bm{\theta}_{-i}) =  \beta\ell(\mathbf{X}^{*}(\bm{\theta}_i, \bm{\theta}_{-i})\bm{\theta}_i, \mathbf{y})+(1-\beta)\ell(\mathbf{X}\bm{\theta}_i, \mathbf{y}).
\end{equation}
Using Eq.~\eqref{eq:c_i}, we can then define a \emph{Multi-Learner
  Nash Game (MLNG)}:
% as follows.
\begin{definition}[Multi-Learner Nash Game (MLNG)]
A static game, denoted as $\langle \mathcal{N}, \bm{\Theta}, (c_i)
\rangle$ is a \emph{Multi-Learner Nash Game} if
\begin{enumerate}
\item The set of players is the set of learners $\mathcal{N}$,
\item the cost function of each learner $i$ is $c_i(\bm{\theta}_i, \bm{\theta}_{-i})$ defined in Eq.~\eqref{eq:c_i},
%\item the learners' strategy spaces are $\bm{\Theta}$,
\item all learners simultaneously select $\bm{\theta}_i \in \bm{\Theta}$.
\end{enumerate}
\label{def:mlng}
\end{definition}
We can then define \emph{Multi-Learner Nash Equilibrium (MLNE)} of the
game $\langle \mathcal{N}, \bm{\Theta}, (c_i) \rangle$:
\begin{definition}[Multi-Learner Nash Equilibrium (MLNE)]
An action profile $\bm{\theta}^{*}=\{\bm{\theta}_i^{*}\}_{i=1}^n$ is a \emph{Multi-Learner Nash Equilibrium} of the MLNG $\langle \mathcal{N}, \bm{\Theta}, (c_i) \rangle$ if it is the solution of the following set of coupled optimization problem:
\begin{equation}
\label{eq:mlne}
\min_{\bm{\theta}_i \in \bm{\Theta}} c_i(\bm{\theta}_i, \bm{\theta}_{-i}), \forall i \in \mathcal{N}.
\end{equation}
\label{def:mlne}
\end{definition}
Combining the results above, the following result is immediate.
%we immediately get the result as follows.
\begin{theorem}%[Equivalence of Games] 
An action profile $(\{\bm{\theta}_i^{*}\}_{i=1}^n, \mathbf{X}^{*})$ is
an MLSE of the multi-learner Stackelberg game if and only if
$\{\bm{\theta}_i^{*}\}_{i=1}^n$ is a MLNE of the multi-learner Nash
game $\langle \mathcal{N}, \bm{\Theta}, (c_i) \rangle$, with
$\mathbf{X}^{*}$ defined in Eq.~\eqref{eq:best_response} for
$\bm{\theta}_i = \bm{\theta}_i^{*}, \forall i \in \mathcal{N}$. 
\label{thm:equivalence_games}
\end{theorem}
%\begin{proof}
%We can get the equivalence of the two games by replacing $X^{*}$ in Equation \ref{eq:mlse} with its closed form in Equation \ref{eq:best_response}, then we use $c_i(\bm{\theta}_i, \bm{\theta}_{-i})$ to replace $c_i(\bm{\theta}_i, \mathbf{X}^{*})$. 
%The resulting optimization problem in Equation \ref{eq:mlse} is exactly the same with Equation \ref{eq:mlne}.
%\end{proof}

Theorem \ref{thm:equivalence_games} shows that we can reduce the
original $(n+1)$-player Stackelberg game to an $n$-player simultaneous-move game $\langle \mathcal{N}, \bm{\Theta}, (c_i) \rangle$. 
%We analyze the Nash equilibrium of $\langle \mathcal{N}, \bm{\Theta},
%(c_i) \rangle$ (Definition \ref{def:mlne}) in the following sections.  
In the remaining sections, we focus on analyzing the Nash equilibrium
of this multi-learner Nash game.

\section{Theoretical Analysis}
\label{sec:analysis}

In this section, we analyze the game $\langle \mathcal{N}, \bm{\Theta}, (c_i) \rangle$. 
As presented in Eq.~\eqref{eq:best_response}, there is an inverse of a complicated matrix to compute the best response of the attacker. 
Hence, the cost function $c_i(\bm{\theta}_i, \bm{\theta}_{-i})$ shown in Eq.~\eqref{eq:c_i} is intractable.
To address this challenge, we first derive a new game, $\langle \mathcal{N}, \bm{\Theta}, (\widetilde{c}_i) \rangle$ with tractable cost function for its players, to approximate $\langle \mathcal{N}, \bm{\Theta}, (c_i) \rangle$.
Afterward, we analyze existence and uniqueness of the \emph{Nash Equilibirum} of $\langle \mathcal{N}, \bm{\Theta}, (\widetilde{c}_i) \rangle$.

\subsection{Approximation of $\langle \mathcal{N}, \bm{\Theta}, (c_i) \rangle$}
We start our analysis by computing $(\lambda\mathbf{I}+\sum_{i=1}^n\bm{\theta}_i\bm{\theta}_i^\top)^{-1}$ presented in Eq.~\eqref{eq:best_response}.
Let matrix $\mathbf{A}_n = \lambda\mathbf{I}+\sum_{i=1}^n\bm{\theta}_i\bm{\theta}_i^\top$, and $\mathbf{A}_{-i}=\lambda\mathbf{I}+\sum_{j \neq i}\bm{\theta}_j\bm{\theta}_j^\top$. 
Then, 
\( \mathbf{A}_n = \mathbf{A}_{-i}+\bm{\theta_i}\bm{\theta_i}^\top. \)
Similarly, let matrix
$\mathbf{B}_n=\lambda\mathbf{X}+\mathbf{z}\sum_{i=1}^n\bm{\theta}_i^\top$,
and $\mathbf{B}_{-i}=\lambda\mathbf{X}+\mathbf{z}\sum_{j\neq
  i}\bm{\theta}_j^\top$, which implies that
%Then, 
\( \mathbf{B}_n = \mathbf{B}_{-i}+\mathbf{z}\bm{\theta}_i^\top \)
The best response of the attacker can then be rewritten as
\( \mathbf{X}^{*} = \mathbf{B}_n\mathbf{A}_n^{-1}. \)
%We then have results as follows.
We then obtain the following results.
\begin{lemma}
\label{lemma:A_minus}
$\mathbf{A}_n$ and $\mathbf{A}_{-i}$ satisfy
\begin{enumerate}
\item $\mathbf{A}_n$ and $\mathbf{A}_{-i}$ are invertible, and the corresponding invertible matrices, $\mathbf{A}_n^{-1}$ and $\mathbf{A}_{-i}^{-1}$, are positive definite.
\item $\mathbf{A}_n^{-1} = \mathbf{A}_{-i}^{-1} - \frac{\mathbf{A}_{-i}^{-1}\bm{\theta}_i\bm{\theta}_i^\top\mathbf{A}_{-i}^{-1}}{1+\bm{\theta}_i^\top\mathbf{A}_{-i}^{-1}\bm{\theta}_i}$.
\item $\bm{\theta}_i^\top\mathbf{A}_{-i}^{-1}\bm{\theta}_i \leq \frac{1}{\lambda}\bm{\theta_i}^\top\bm{\theta}_i$.
\end{enumerate}
\end{lemma}

\begin{proof}
\begin{enumerate}

\item First, we prove that $\mathbf{A}_n= \lambda\mathbf{I}+\sum_{i=1}^n\bm{\theta}_i\bm{\theta}_i^\top$ is invertible, and its inverse matrix, $\mathbf{A}_n^{-1}$, is positive definite by using mathematical induction.

When $n = 1$, $\mathbf{A}_1=\lambda\mathbf{I}+\bm{\theta}_1\bm{\theta}_1^\top$.
As $\lambda\mathbf{I}$ is an invertible square matrix and $\bm{\theta}_1$ is a column vector, by using \emph{Sherman-Morrison formula}, $\mathbf{A}_1$ is invertible.
\[ \mathbf{A}_1^{-1} = \frac{1}{\lambda}(\mathbf{I}-\frac{\bm{\theta}_1\bm{\theta}_1^\top}{\lambda+\bm{\theta}_1^\top\bm{\theta}_1}). \] 
For any non-zero column vector $\mathbf{u}$, we have
\[ \mathbf{u}^\top\mathbf{A}_1^{-1}\mathbf{u} = \frac{\lambda\mathbf{u}^\top\mathbf{u}+\mathbf{u}^\top\mathbf{u}\bm{\theta}_1^\top\bm{\theta}_1-\mathbf{u^\top}\bm{\theta}_1\bm{\theta}_1^\top\mathbf{u}}{\lambda(\lambda+\bm{\theta}_1^\top\bm{\theta}_1)}.\]
As $\mathbf{u}^\top\mathbf{u} > 0$ and $\lambda > 0$, according to \emph{Cauchy-Schwarz inequality},
\[ \mathbf{u}^\top\mathbf{u}\bm{\theta}_1^\top\bm{\theta}_1 \geq \mathbf{u^\top}\bm{\theta}_1\bm{\theta}_1^\top\mathbf{u}, \] 
Then, $\mathbf{u}^\top\mathbf{A}_1^{-1}\mathbf{u} > 0$. 
Thus, $\mathbf{A}_1^{-1}$ is a positive definite matrix.

We then assume that when $n=k(k \geq 1)$, $\mathbf{A}_k$ is invertible and $\mathbf{A}_k^{-1}$ is positive definite.
Then, when $n=k+1$,
\[ \mathbf{A}_{k+1} = \mathbf{A}_k + \bm{\theta}_{k+1}\bm{\theta}_{k+1}^\top. \]
As $\mathbf{A}_{k}$ is invertible, $\bm{\theta}_{k+1}$ is a column vector. By using \emph{Sherman-Morrison formula}, we have that $\mathbf{A}_{k+1}$ is invertible, and
\[  \mathbf{A}_{k+1}^{-1} = \mathbf{A}_k^{-1}-\frac{\mathbf{A}_k^{-1}\bm{\theta}_{k+1}\bm{\theta}_{k+1}^\top\mathbf{A}_k^{-1}}{1+\bm{\theta}_{k+1}^\top\mathbf{A}_k^{-1}\bm{\theta}_{k+1}} .\]
Then,
\begin{equation*}
\mathbf{u}^\top\mathbf{A}_{k+1}^{-1}\mathbf{u} 
= \frac{\mathbf{u}^\top\mathbf{A}_k^{-1}\mathbf{u} + \mathbf{u}^\top\mathbf{A}_k^{-1}\mathbf{u}\cdot\bm{\theta}_{k+1}^\top\mathbf{A}_k^{-1}\bm{\theta}_{k+1} - \mathbf{u}^\top\mathbf{A}_k^{-1}\bm{\theta}_{k+1}\cdot\bm{\theta}_{k+1}^\top\mathbf{A}_k^{-1}\mathbf{u}}{1+\bm{\theta}_{k+1}^\top \mathbf{A}_k^{-1} \bm{\theta}_{k+1}}
\end{equation*}
As $\mathbf{A}_{k}^{-1}$ is a positive definite matrix, we have 
$\mathbf{u}^\top\mathbf{A}_k^{-1}\mathbf{u} > 0$ 
and 
$\bm{\theta}_{k+1}^\top\mathbf{A}_k^{-1}\bm{\theta}_{k+1} > 0$.
By using \emph{Extended Cauchy-Schwarz inequality}, we have
\[ \mathbf{u}^\top\mathbf{A}_k^{-1}\mathbf{u}\bm{\theta}_{k+1}^\top\mathbf{A}_k^{-1}\bm{\theta}_{k+1} > \mathbf{u}^\top\mathbf{A}_k^{-1}\bm{\theta}_{k+1}\bm{\theta}_{k+1}^\top\mathbf{A}_k^{-1}\mathbf{u}. \]
Then, $\mathbf{A}_{k+1}^{-1}$ is positive definite.
Hence, $\mathbf{A}_n= \lambda\mathbf{I}+\sum_{i=1}^n\bm{\theta}_i\bm{\theta}_i^\top$ is invertible, and $\mathbf{A}_n^{-1}$ is positive definite. Similarly, we can prove that $\mathbf{A}_{-i}$ is invertible, and $\mathbf{A}_{-i}^{-1}$ is positive definite.

\item We have proved that $\mathbf{A}_n$ and $\mathbf{A}_{-i}$ are invertible. Then, the result can be obtained by using \emph{Sherman-Morrison formula}.

\item Let $\mathbf{A}_{-i,-j} = \mathbf{A}_{-i} - \bm{\theta}_j\bm{\theta}_j^\top$. As $\mathbf{A}_{-i,-j}$ is a symmetric matrix, its inverse, $\mathbf{A}_{-i,-j}^{-1}$ is also symmetric. 
Using a similar approach to the one above, we can prove that $\mathbf{A}_{-i,-j}$ is invertible and $\mathbf{A}_{-i,-j}^{-1}$ is positive definite.
By using \emph{Sherman-Morrison formula}, we have
\[ \mathbf{A}_{-i}^{-1} =  \mathbf{A}_{-i,-j}^{-1} - \frac{\mathbf{A}_{-i,-j}^{-1}\bm{\theta}_j\bm{\theta}_j^\top\mathbf{A}_{-i,-j}^{-1}}{1+\bm{\theta}_j^\top\mathbf{A}_{-i,-j}^{-1}\bm{\theta}_j}.\]
Then,
\begin{equation*}
\begin{split}
\bm{\theta}_i^\top\mathbf{A}_{-i}^{-1}\bm{\theta}_i &= \bm{\theta}_i^\top\mathbf{A}_{-i,-j}^{-1}\bm{\theta}_i - \frac{\bm{\theta}_i^\top\mathbf{A}_{-i,-j}^{-1}\bm{\theta}_j\cdot{•}\bm{\theta}_j^\top\mathbf{A}_{-i,-j}^{-1}\bm{\theta}_i}{1+\bm{\theta}_j^\top\mathbf{A}_{-i,-j}^{-1}\bm{\theta}_j}\\
&= \bm{\theta}_i^\top\mathbf{A}_{-i,-j}^{-1}\bm{\theta}_i - \frac{(\bm{\theta}_i^\top\mathbf{A}_{-i,-j}^{-1}\bm{\theta}_j)^2}{1+\bm{\theta}_j^\top\mathbf{A}_{-i,-j}^{-1}\bm{\theta}_j} \\
&\leq \bm{\theta}_i^\top\mathbf{A}_{-i,-j}^{-1}\bm{\theta}_i.
\end{split}
\end{equation*}
We then iteratively apply \emph{Sherman-Morrison formula} and get
\begin{equation*}
\begin{split}
\bm{\theta}_i^\top\mathbf{A}_{-i}^{-1}\bm{\theta}_i &\leq \bm{\theta}_i^\top\mathbf{A}_0^{-1}\bm{\theta}_i \\
&= \frac{1}{\lambda} \bm{\theta}_i^\top\bm{\theta}_i.
\end{split}
\end{equation*}
\end{enumerate}\end{proof}

Lemma \ref{lemma:A_minus} allows us to relax $\ell(\mathbf{X}^{*}(\bm{\theta}_i, \bm{\theta}_{-i}) \bm{\theta}_i,\mathbf{y})$ as follows:
\begin{lemma}
\label{lemma:1st_bound}
%$\ell(\mathbf{X}^{*}(\bm{\theta}_i, \bm{\theta}_{-i}),\mathbf{y})$ can be relaxed by deriving its upper bound:
\begin{align}
\label{eq:1st_bound}
%\ell(\mathbf{X}^{*}\bm{\theta}_i, \mathbf{y}) 
\nonumber \ell(\mathbf{X}^{*}(\bm{\theta}_i, \bm{\theta}_{-i}) \bm{\theta}_i,\mathbf{y})
\leq \ell(&\mathbf{B}_{-i}\mathbf{A}_{-i}^{-1}\bm{\theta}_i,
                                               \mathbf{y})\\
&+\frac{1}{\lambda^2}||\mathbf{z}-\mathbf{y}||_2^2(\bm{\theta}_i^\top\bm{\theta}_i)^2.
\end{align}
\end{lemma}

\begin{proof}
Firstly, by using \emph{Sherman-Morrison formula} we have
\begin{equation*}
\begin{split}
&\mathbf{X}^{*}\bm{\theta}_i = \mathbf{B}_n\mathbf{A}_n^{-1}\bm{\theta}_i \\
&= (\mathbf{B}_{-i}+\mathbf{z}\bm{\theta}_i^\top)(\mathbf{A}_{-i}^{-1} - \frac{\mathbf{A}_{-i}^{-1}\bm{\theta}_i\bm{\theta}_i^\top\mathbf{A}_{-i}^{-1}}{1+\bm{\theta}_i^\top\mathbf{A}_{-i}^{-1}\bm{\theta}_i})\bm{\theta}_i \\
&= \mathbf{B}_{-i}\mathbf{A}_{-i}^{-1}\bm{\theta}_i + \frac{\mathbf{z}\bm{\theta}_i^\top\mathbf{A}_{-i}^{-1}\bm{\theta}_i - \mathbf{B}_{-i}\mathbf{A}_{-i}^{-1}\bm{\theta}_i\bm{\theta}_i^\top\mathbf{A}_{-i}^{-1}\bm{\theta}_i }{ 1+\bm{\theta}_i^\top\mathbf{A}_{-i}^{-1}\bm{\theta}_i} \\
&= \frac{(\mathbf{B}_{-i}+\mathbf{z}\bm{\theta}_i^\top)\mathbf{A}_{-i}^{-1}\bm{\theta}_i}{1+\bm{\theta}_i^\top\mathbf{A}_{-i}^{-1}\bm{\theta}_i} \\
&= \frac{\mathbf{B}_n\mathbf{A}_{-i}^{-1}\bm{\theta}_i}{1+\bm{\theta}_i^\top\mathbf{A}_{-i}^{-1}\bm{\theta}_i}.
\end{split}
\end{equation*}
Then,
\begin{equation*}
\begin{split}
&\ell(\mathbf{X}^{*}\bm{\theta}_i, \mathbf{y}) = || \frac{\mathbf{B}_n\mathbf{A}_{-i}^{-1}\bm{\theta}_i}{1+\bm{\theta}_i^\top\mathbf{A}_{-i}^{-1}\bm{\theta}_i} - \mathbf{y} ||_2^2 \\
&= || \frac{ \mathbf{B}_n\mathbf{A}_{-i}^{-1}\bm{\theta}_i -\mathbf{y} - \bm{\theta}_i^\top\mathbf{A}_{-i}^{-1}\bm{\theta}_i\mathbf{y} }{ 1+\bm{\theta}_i^\top\mathbf{A}_{-i}^{-1}\bm{\theta}_i } ||_2^2 \\
&\leq || \mathbf{B}_n\mathbf{A}_{-i}^{-1}\bm{\theta}_i -\mathbf{y} - \bm{\theta}_i^\top\mathbf{A}_{-i}^{-1}\bm{\theta}_i\mathbf{y} ||_2^2 \\
&= || (\mathbf{B}_{-i}+\mathbf{z}\bm{\theta}_i^\top)\mathbf{A}_{-i}^{-1}\bm{\theta}_i -\mathbf{y} - \bm{\theta}_i^\top\mathbf{A}_{-i}^{-1}\bm{\theta}_i\mathbf{y} ||_2^2 \\
&= || \mathbf{B}_{-i}\mathbf{A}_{-i}^{-1}\bm{\theta}_i - \mathbf{y} + (\mathbf{z}-\mathbf{y})\bm{\theta}_i^\top\mathbf{A}_{-i}^{-1}\bm{\theta}_i ||_2^2 \\
&\leq \ell(\mathbf{B}_{-i}\mathbf{A}_{-i}^{-1}\bm{\theta}_i,\mathbf{y}) + ||\mathbf{z}-\mathbf{y}||_2^2(\bm{\theta}_i^\top\mathbf{A}_{-i}^{-1}\bm{\theta}_i)^2 
\end{split}
\end{equation*}
By using Lemma \ref{lemma:A_minus}, we have
$(\bm{\theta}_i^\top\mathbf{A}_{-i}^{-1}\bm{\theta}_i)^2 \leq \frac{1}{\lambda^2}(\bm{\theta}_i^\top\bm{\theta}_i)^2$ which completes the proof.
\end{proof}

Note that in Eq. ~\eqref{eq:1st_bound}, $\mathbf{B}_{-i}$ and $\mathbf{A}_{-i}$ only depend on $\{\bm{\theta}_j\}_{j\neq i}$.
Hence, the RHS of Eq.~\eqref{eq:1st_bound} is a strictly convex function with respect to $\bm{\theta}_i$.
Lemma \ref{lemma:1st_bound} shows that $\ell(\mathbf{X}^{*}(\bm{\theta}_i, \bm{\theta}_{-i}) \bm{\theta}_i,\mathbf{y})$ can be relaxed by moving $\bm{\theta}_i$ out of $\mathbf{X}^{*}(\bm{\theta}_i, \bm{\theta}_{-i})$ and adding a regularizer $(\bm{\theta}_i^\top\bm{\theta}_i)^2$ with its coefficient $\frac{||\mathbf{z}-\mathbf{y}||_2^2}{\lambda^2}$.
Motivated by this method, we iteratively relax $\ell(\mathbf{X}^{*}(\bm{\theta}_i, \bm{\theta}_{-i}) \bm{\theta}_i,\mathbf{y})$ by adding corresponding regularizers.
We now identify a tractable upper bound function for $c_i(\bm{\theta}_i, \bm{\theta}_{-i})$.
\begin{theorem}
\label{thm:2nd_bound}
\begin{equation}
\label{eq:2nd_bound}
\begin{split}
c_i(\bm{\theta}_i, \bm{\theta}_{-i}) &\leq \bar{c}_i(\bm{\theta}_i, \bm{\theta}_{-i}) \\
&= \ell(\mathbf{X}\bm{\theta}_i,\mathbf{y})+\frac{\beta}{\lambda^2}||\mathbf{z}-\mathbf{y}||_2^2\sum_{j=1}^n(\bm{\theta}_j^\top\bm{\theta}_i)^2 + \epsilon,
\end{split}
\end{equation}
where $\epsilon$ is a positive constant and $\epsilon < +\infty$.
\end{theorem} 

\begin{proof}
We prove by extending the results in Lemma \ref{lemma:1st_bound} and iteratively relaxing the cost function. As presented in Lemma 3, we have 
\[ \ell(\mathbf{X}^{*}\bm{\theta}_i, \mathbf{y}) \leq \ell(\mathbf{B}_{-i}\mathbf{A}_{-i}^{-1}\bm{\theta}_i,\mathbf{y}) + \frac{1}{\lambda^2}||\mathbf{z}-\mathbf{y}||_2^2(\bm{\theta}_i^\top\bm{\theta}_i)^2 .\]
By using \emph{Sherman-Morrison formula},
\begin{equation*}
\begin{split}
\ell(\mathbf{B}_{-i}\mathbf{A}_{-i}^{-1}\bm{\theta}_i,\mathbf{y}) &=  || \mathbf{B}_{-i}(\mathbf{A}_{-i,-j}^{-1}-\frac{\mathbf{A}_{-i,-j}^{-1}\bm{\theta}_j\bm{\theta}_j^\top\mathbf{A}_{-i,-j}^{-1}}{1+\bm{\theta}_j^\top\mathbf{A}_{-i,-j}^{-1}\bm{\theta}_j})\bm{\theta}_i - \mathbf{y} ||_2^2 \\
&\leq || \frac{\mathbf{B}_{-i}\mathbf{A}_{-i,-j}^{-1}\bm{\theta}_i}{1+\bm{\theta}_j^\top\mathbf{A}_{-i,-j}^{-1}\bm{\theta}_j} - \mathbf{y} ||_2^2 + \bigtriangleup_1(\bm{\theta})
\end{split}
\end{equation*}
where $j \neq i$, and $\bigtriangleup_1(\bm{\theta})$ is a continuous function of $\bm{\theta}=\{\bm{\theta}_i\}_{i=1}^n$. 
As the action space $\bm{\Theta}$ is bounded, then $0 \leq \bigtriangleup_1(\bm{\theta}) < \infty$.
Hence, we have
\begin{equation*}
\begin{split}
\ell(\mathbf{B}_{-i}\mathbf{A}_{-i}^{-1}\bm{\theta}_i,\mathbf{y}) &\leq || \frac{\mathbf{B}_{-i}\mathbf{A}_{-i,-j}^{-1}\bm{\theta}_i}{1+\bm{\theta}_j^\top\mathbf{A}_{-i,-j}^{-1}\bm{\theta}_j} - \mathbf{y} ||_2^2 + \bigtriangleup_1(\bm{\theta}) \\
& = || \frac{\mathbf{B}_{-i}\mathbf{A}_{-i,-j}^{-1}\bm{\theta}_i - \mathbf{y} - \bm{\theta}_j^\top\mathbf{A}_{-i,-j}^{-1}\bm{\theta}_j\mathbf{y} }{ 1+\bm{\theta}_j^\top\mathbf{A}_{-i,-j}^{-1}\bm{\theta}_j } ||_2^2 + \bigtriangleup_1(\bm{\theta})  \\
&\leq || \mathbf{B}_{-i}\mathbf{A}_{-i,-j}^{-1}\bm{\theta}_i - \mathbf{y} - \bm{\theta}_j^\top\mathbf{A}_{-i,-j}^{-1}\bm{\theta}_j\mathbf{y} ||_2^2 + \bigtriangleup_1(\bm{\theta})  \\
&= || (\mathbf{B}_{-i,-j}+\mathbf{z}\bm{\theta}_j^\top)\mathbf{A}_{-i,-j}^{-1}\bm{\theta}_i - \mathbf{y} - \bm{\theta}_j^\top\mathbf{A}_{-i,-j}^{-1}\bm{\theta}_j\mathbf{y} ||_2^2 + \bigtriangleup_1(\bm{\theta}) \\
&= || (\mathbf{B}_{-i,-j}\mathbf{A}_{-i,-j}^{-1}\bm{\theta}_i - \mathbf{y}) + (\mathbf{z} - \mathbf{y}) \bm{\theta}_j^\top\mathbf{A}_{-i,-j}^{-1}\bm{\theta}_i + \bm{\theta_j}^\top\mathbf{A}_{-i,-j}^\top(\bm{\theta}_i-\bm{\theta}_j)\mathbf{y}||_2^2 + \bigtriangleup_1(\bm{\theta}) \\
&\leq \ell(\mathbf{B}_{-i,-j}\mathbf{A}_{-i,-j}^{-1}\bm{\theta}_i, \mathbf{y}) + || (\mathbf{z} - \mathbf{y})||_2^2 (\bm{\theta}_j^\top\mathbf{A}_{-i,-j}^{-1}\bm{\theta}_i )^2 + \bigtriangleup_2(\bm{\theta})
\end{split}
\end{equation*}
where $\bigtriangleup_2(\bm{\theta})$ is a continuous function of $\bm{\theta}$ and $0\leq\bigtriangleup_2(\bm{\theta})<\infty$.
Let $\mathbf{A}_{-i,-j,-k}=\mathbf{A}_{-i,-j}-\bm{\theta}_k\bm{\theta}_k^\top$, then,
similarly, $( \bm{\theta}_j^\top\mathbf{A}_{-i,-j}^{-1}\bm{\theta}_i)^2$ can be further relaxed as follows.
\begin{equation*}
\begin{split}
( \bm{\theta}_j^\top\mathbf{A}_{-i,-j}^{-1}\bm{\theta}_i )^2 
&= ( \bm{\theta}_j^\top (\mathbf{A}_{-i,-j,-k}^{-1} - \frac{\mathbf{A}_{-i,-j,-k}^{-1}\bm{\theta}_k\bm{\theta}_k^\top\mathbf{A}_{-i,-j,-k}^{-1}}{1+\bm{\theta}_k^\top\mathbf{A}_{-i,-j,-k}^{-1}\bm{\theta}_k}) \bm{\theta}_i )^2\\
&\leq (\bm{\theta}_j^\top\mathbf{A}_{-i,-j,-k}^{-1}\bm{\theta}_i)^2 + \bigtriangleup_3(\bm{\theta}) 
\end{split}
\end{equation*}
where $0\leq\bigtriangleup_3(\bm{\theta})<\infty$, using the same approach,$( \bm{\theta}_j^\top\mathbf{A}_{-i,-j}^{-1}\bm{\theta}_i )^2 $ can be further and iteratively relaxed as follows,
\begin{equation*}
\begin{split}
( \bm{\theta}_j^\top\mathbf{A}_{-i,-j}^{-1}\bm{\theta}_i )^2  &\leq ( \bm{\theta}_j^\top\mathbf{A}_0^{-1}\bm{\theta}_i )^2 + \bigtriangleup_4(\bm{\theta}) \\
&= \frac{1}{\lambda^2}(\bm{\theta}_j^\top\bm{\theta}_i)^2 + \bigtriangleup_4(\bm{\theta})  
\end{split}
\end{equation*}
where $0\leq\bigtriangleup_4(\bm{\theta})<\infty$.  
Combining the results above, we can iteratively relax $\ell(\mathbf{B}_{-i}\mathbf{A}_{-i}^{-1}\bm{\theta}_i,\mathbf{y})$ as follows,
\begin{equation*}
\begin{split}
\ell(\mathbf{B}_{-i}\mathbf{A}_{-i}^{-1}\bm{\theta}_i,\mathbf{y}) &\leq \ell(\mathbf{B}_{-i,-j}\mathbf{A}_{-i,-j}^{-1}\bm{\theta}_i, \mathbf{y}) + \frac{1}{\lambda^2}||\mathbf{z}-\mathbf{y}||_2^2(\bm{\theta}_j^\top\bm{\theta}_i)^2 + \bigtriangleup_5(\bm{\theta}) \\
&\leq \ell(\mathbf{X}\bm{\theta}_i,\mathbf{y}) + \frac{1}{\lambda^2}||\mathbf{z}-\mathbf{y}||_2^2\sum_{j\neq i} (\bm{\theta}_j^\top\bm{\theta}_i)^2 + \bigtriangleup(\bm{\theta})
\end{split}
\end{equation*}
where $0\leq\bigtriangleup_5(\bm{\theta})<\infty$ and $0\leq\bigtriangleup(\bm{\theta})<\infty$. 
Then,
\begin{equation*}
\begin{split}
\ell(\mathbf{X}^{*}\bm{\theta}_i, \mathbf{y}) &\leq \ell(\mathbf{B}_{-i}\mathbf{A}_{-i}^{-1}\bm{\theta}_i,\mathbf{y}) + \frac{1}{\lambda^2}||\mathbf{z}-\mathbf{y}||_2^2(\bm{\theta}_i^\top\bm{\theta}_i)^2 \\
&\leq \ell(\mathbf{X}\bm{\theta}_i,\mathbf{y}) + \frac{1}{\lambda^2}||\mathbf{z}-\mathbf{y}||_2^2\sum_{j=1}^n (\bm{\theta}_j^\top\bm{\theta}_i)^2 + \bigtriangleup(\bm{\theta}).
\end{split}
\end{equation*}
Hence,
\begin{equation*}
\begin{split}
c_i(\bm{\theta}_i, \bm{\theta}_{-i}) &= \beta\ell(\mathbf{X}^{*}\bm{\theta}_i, \mathbf{y}) + (1-\beta)\ell(\mathbf{X}\bm{\theta}_i,\mathbf{y})\\
&\leq \ell(\mathbf{X}\bm{\theta}_i,\mathbf{y})+\frac{\beta}{\lambda^2}||\mathbf{z}-\mathbf{y}||_2^2\sum_{j=1}^n(\bm{\theta}_j^\top\bm{\theta}_i)^2 + \epsilon
\end{split}
\end{equation*}
where $\epsilon$ is a constant such that $\epsilon=\beta\ast\max_{\bm{\theta}}\{\bigtriangleup(\bm{\theta})\}<\infty$.
\end{proof}

As represented in Eq.~\eqref{eq:2nd_bound}, $\bar{c}_i(\bm{\theta}_i, \bm{\theta}_{-i})$ is strictly convex with respect to $\bm{\theta}_i$ and $\bm{\theta}_j(\forall j \neq i)$.
We then use the game $\langle \mathcal{N}, \bm{\Theta}, (\bar{c}_i) \rangle$ as an approximation of $\langle \mathcal{N}, \bm{\Theta}, (c_i) \rangle$.
Let
\begin{equation}
\label{eq:3rd_bound} 
\begin{split}
\widetilde{c}_i(\bm{\theta}_i, \bm{\theta}_{-i}) &= \bar{c}_i(\bm{\theta}_i, \bm{\theta}_{-i}) - \epsilon \\
&= \ell(\mathbf{X}\bm{\theta}_i,\mathbf{y})+\frac{\beta}{\lambda^2}||\mathbf{z}-\mathbf{y}||_2^2\sum_{j=1}^n(\bm{\theta}_j^\top\bm{\theta}_i)^2,
\end{split}
\end{equation}
then $\langle \mathcal{N}, \bm{\Theta}, (\widetilde{c}_i) \rangle$ has
the same Nash equilibrium with $\langle \mathcal{N}, \bm{\Theta},
(\bar{c}_i) \rangle$ if one exists, as adding or deleting a constant term does not affect the optimal solution.
Hence, we use $\langle \mathcal{N}, \bm{\Theta}, (\widetilde{c}_i) \rangle$ to approximate $\langle \mathcal{N}, \bm{\Theta}, (c_i) \rangle$, and analyze the Nash equilibrium of $\langle \mathcal{N}, \bm{\Theta}, (\widetilde{c}_i) \rangle$ in the remaining sections.

\subsection{Existence of Nash Equilibrium}
As introduced in Section \ref{sec:model}, each learner has identical action spaces, and they are trained with the same dataset.
We exploit this symmetry to analyze the existence of a Nash equilibrium of the approximation game $\langle \mathcal{N}, \bm{\Theta}, (\widetilde{c}_i) \rangle$.

We first define a \emph{Symmetric Game}~\citep{Cheng04}:
\begin{definition}[Symmetric Game] An n-player game is symmetric if the players have the same action space, and their cost functions $c_i(\bm{\theta}_i, \bm{\theta}_{-i})$ satisfies
\begin{equation}
\label{eq:symmetric}
c_i(\bm{\theta}_i, \bm{\theta}_{-i}) = c_j(\bm{\theta}_j, \bm{\theta}_{-j}),\forall i,j \in \mathcal{N}
\end{equation}
if $\bm{\theta}_i = \bm{\theta}_j$ and $\bm{\theta}_{-i} = \bm{\theta}_{-j}$.
\label{def:symmetric}
\end{definition}

In a symmetric game $\langle \mathcal{N}, \bm{\Theta}, (\widetilde{c}_i) \rangle$ it is natural to consider a \emph{Symmetric Equilibrium}:
\begin{definition}[Symmetric Equilibrium] An action profile $\{\bm{\theta}_i^{*}\}_{i=1}^n$ of $\langle \mathcal{N}, \bm{\Theta}, (\widetilde{c}_i) \rangle$ is a symmetric equilibrium if it is a Nash equilibrium and $\bm{\theta}_i^{*} = \bm{\theta}_j^{*}, \forall i,j \in \mathcal{N}$.
\label{def:se}
\end{definition}

%We now derive the existence of Nash equilibrium.
We now show that our approximate game is symmetric, and always has a
symmetric Nash equilibrium.
%As shown by~\citet{Fey12}, a symmetric game may have either a
%symmetric or asymmetric equilibria.
%Although sufficient conditions are identified under which there exists
%Nash equilibrium in previous works
%Although prior literature has provided sufficient conditions for Nash
%equilibrium existence (see, e.g.,~\citep{Rosen65}), these 
%they are inadequate to distinguish the existence of symmetric and asymmetric equilibrium.  
%Using Theorem 3 from~\citet{Cheng04}, we have the result as follows.
\begin{theorem}[Existence of Nash Equilibrium] $\langle \mathcal{N}, \bm{\Theta}, (\widetilde{c}_i) \rangle$ is a symmetric game and it has at least one symmetric equilibrium.
\label{thm:existence}
\end{theorem}
\begin{proof}
As described above, the players of $\langle \mathcal{N}, \bm{\Theta}, (\widetilde{c}_i) \rangle$ use the same action space and complete information of others. 
Hence, the cost function $c_i$ is symmetric, making $\langle \mathcal{N}, \bm{\Theta}, (\widetilde{c}_i) \rangle$ a symmetric game.
As $\langle \mathcal{N}, \bm{\Theta}, (\widetilde{c}_i) \rangle$ has nonempty, compact and convex action space, and the cost function $\widetilde{c}_i$ is continuous in $\{\bm{\theta}_i\}_{i=1}^n$ and convex in $\bm{\theta}_i$,
according to Theorem 3 in~\citet{Cheng04}, $\langle \mathcal{N}, \bm{\Theta}, (\widetilde{c}_i) \rangle$ has at least one symmetric Nash equilibrium.
\end{proof}
\vspace{-0.1in}
\subsection{Uniqueness of Nash Equilibrium}
%So far we have proved the existence of at least one symmetric Nash
%equilibrium. 
While we showed that the approximate game always admits a symmetric
Nash equilibrium, it leaves open the possibility that there may be
multiple symmetric equilibria, as well as equilibria which are not symmetric.
%Two questions naturally arise from the results above:
%(1) Is the symmetric Nash equilibrium unique?
%(2) Does there exist any asymmetric equilibrium?
%We answer the two questions above by investigating the existence of Nash equilibrium of $\langle \mathcal{N}, \bm{\Theta}, (\widetilde{c}_i) \rangle$.
%If there is a unique Nash equilibrium, then according to Theorem \ref{thm:existence}, this unique Nash equilibrium is symmetric, and there exists no asymmetric Nash equilibrium.
We now demonstrate that this game in fact has a \emph{unique}
equilibrium (which must therefore be symmetric).
\begin{theorem}[Uniqueness of Nash Equilibrium] $\langle \mathcal{N}, \bm{\Theta}, (\widetilde{c}_i) \rangle$ has a unique Nash equilibrium, and this unique NE is symmetric.
\label{thm:uniqueness_sym_MLNE}
\end{theorem}

\begin{proof}
We have known that $\langle \mathcal{N}, \bm{\Theta}, (\widetilde{c}_i) \rangle$ has at least NE, and each learner has an nonempty, compact and convex action space $\bm{\Theta}$. 
Hence, we can apply Theorem 2 and Theorem 6 of~\citet{Rosen65}. 
That is, for some fixed $\{r_i\}_i^n (0<r_i<1, \sum_{i=1}^n r_i= 1)$, if the matrix in Eq.~\eqref{eq:jr} is positive definite, then $\langle \mathcal{N}, \bm{\Theta}, (\widetilde{c}_i) \rangle$ has a unique NE. 
\begin{equation}
\label{eq:jr}
Jr(\bm{\theta}) =
\begin{bmatrix}
    &r_1\nabla_{\bm{\theta}_1,\bm{\theta}_1}\widetilde{c}_1(\bm{\theta}) & \dots  &  r_1\nabla_{\bm{\theta}_1,\bm{\theta}_n}\widetilde{c}_1(\bm{\theta})\\
    &\vdots  &   & \vdots \\
    &r_n\nabla_{\bm{\theta}_n,\bm{\theta}_1}\widetilde{c}_n(\bm{\theta}) & \dots  &  r_n\nabla_{\bm{\theta}_n,\bm{\theta}_n}\widetilde{c}_n(\bm{\theta})
\end{bmatrix}
\end{equation}
By taking second-order derivatives, we have
\begin{equation*}
\nabla_{\bm{\theta}_i,\bm{\theta}_i}\widetilde{c}_i(\bm{\theta})=2\mathbf{X}^\top\mathbf{X}+\frac{2\beta||\mathbf{z}-\mathbf{y}||_2^2}{\lambda^2}(4\bm{\theta}_i\bm{\theta}_i^\top+2\bm{\theta}_i^\top\bm{\theta}_i\mathbf{I}+\sum_{j\neq i}\bm{\theta}_j\bm{\theta}_j^\top)
\end{equation*}
and
\begin{equation*}
\nabla_{\bm{\theta}_i,\bm{\theta}_j}\widetilde{c}_i(\bm{\theta})=\frac{2\beta||\mathbf{z}-\mathbf{y}||_2^2}{\lambda^2}(\bm{\theta}_i^\top\bm{\theta}_j\mathbf{I}+\bm{\theta}_j\bm{\theta}_i^\top)
\end{equation*}

We first let $r_1=r_2=...=r_n=\frac{1}{n}$ and decompose $Jr(\bm{\theta})$ as follows,
\begin{equation}\label{eq:decompose}
Jr(\bm{\theta}) = \frac{2}{n}\mathbf{P} + \frac{2\beta||\mathbf{z}-\mathbf{y}||_2^2}{\lambda^2 n}(\mathbf{Q}+\mathbf{S}+\mathbf{T}),
\end{equation}
where $\mathbf{P}$ and $\mathbf{Q}$ are \emph{block diagonal matrices} such that $\mathbf{P}_{ii}=\mathbf{X}^\top\mathbf{X}$, $\mathbf{P}_{ij}=\mathbf{0}$, $\mathbf{Q}_{ii} = 4\bm{\theta}_i\bm{\theta}_i^\top+\bm{\theta}_i^\top\bm{\theta}_i\mathbf{I}$ and $\mathbf{Q}_{ij} = \mathbf{
0}$, $\forall i,j \in \mathcal{N}, j\neq i$.
$\mathbf{S}$ and $\mathbf{T}$ are \emph{block symmetric matrices} such that $\mathbf{S}_{ii} = \bm{\theta}_i^\top\bm{\theta}_i\mathbf{I}$, $\mathbf{S}_{ij}=\bm{\theta}_i^\top\bm{\theta}_j\mathbf{I}$, $\mathbf{T}_{ii}=\sum_{j\neq i}\bm{\theta}_j\bm{\theta}_j^\top$ and $\mathbf{T}_{ij}=\bm{\theta}_j\bm{\theta}_i^\top$, $\forall i,j \in \mathcal{N}, j\neq i$.

Next, we prove that $\mathbf{P}$ is \emph{positive definite}, and $\mathbf{Q}$, $\mathbf{S}$ and $\mathbf{T}$ are \emph{positive semi-definite}.
Let $\mathbf{u} = [\mathbf{u}_1^\top,...,\mathbf{u}_n^\top]^\top$ be an $nd\times 1$ vector, where $\mathbf{u}_i\in\mathbb{R}^{d\times 1}(i\in \mathcal{N})$ are not all zero vectors.

\begin{enumerate}

\item $\mathbf{u}^\top\mathbf{P}\mathbf{u} = \sum_{i=1}^n \mathbf{u}_i^\top\mathbf{X}^\top\mathbf{X}\mathbf{u}_i=\sum_{i=1}^n||\mathbf{X}\mathbf{u}_i||_2^2$.
As the columns of $\mathbf{X}$ are linearly independent and $\mathbf{u}_i$ are not all zero vectors, there exists at least one  $\mathbf{u}_i$ such that $\mathbf{X}\mathbf{u}_i \neq \mathbf{0}$.
Hence, $\mathbf{u}^\top\mathbf{P}\mathbf{u} >0 $ which indicates that $\mathbf{P}$ is positive definite.

\item Similarly, $\mathbf{u}^\top\mathbf{Q}\mathbf{u} \geq 0$ which indicates that $\mathbf{Q}$ is a positive semi-definite matrix.

\item Let's $\mathbf{S}^{*} \in \mathbb{R}^{n \times n}$ be a symmetric matrix such that $\mathbf{S}_{ii}^{*}=\bm{\theta}_i^\top\bm{\theta}_i$ and $\mathbf{S}_{ij}^{*}=\bm{\theta}_i^\top\bm{\theta}_j$, $\forall i,j\in \mathcal{N}, j\neq i$.
Hence, $\mathbf{S}_{ij} = \mathbf{S}_{ij}^{*}\mathbf{I}$, $\forall i,j \in \mathcal{N}$.
Note that $\mathbf{S}^{*}=[\bm{\theta}_1,\bm{\theta}_2,...,\bm{\theta}_n]^\top[\bm{\theta}_1,\bm{\theta}_2,...,\bm{\theta}_n]$ is a positive semi-definite matrix, as it is also symmetric, there exists at least one lower triangular matrix $\mathbf{L}^{*}\in\mathbb{R}^{n\times n}$ with non-negative diagonal elements~\citep{Higham1990} such that
\begin{equation*}
\mathbf{S}^{*} = \mathbf{L}^{*}{\mathbf{L}^{*}}^\top \text{(Cholesky Decomposition)}
\end{equation*}
Let $\mathbf{L}$ be a block matrix such that $\mathbf{L}_{ij} = \mathbf{L}_{ij}^{*}\mathbf{I}$, $\forall i,j\in\mathcal{N}$.
Therefore, $(\mathbf{L}\mathbf{L}^\top)_{ij}=(\mathbf{L}^{*}{\mathbf{L}^{*}}^\top)_{ij}\mathbf{I}=\mathbf{S}_{ij}^{*}\mathbf{I}= \mathbf{S}_{ij}$
which indicates that $\mathbf{S}=\mathbf{L}\mathbf{L}^\top$ is a positive semi-definite matrix. 

\item Since
\begin{equation*}
\begin{split}
\mathbf{u}^\top\mathbf{T}\mathbf{u} &= \sum_{i=1}^n\sum_{j\neq i}(\mathbf{u}_i^\top\bm{\theta}_j)^2 + \sum_{i=1}^n\sum_{j \neq i}(\mathbf{u}_i^\top\bm{\theta}_j)(\mathbf{u}_j^\top\bm{\theta}_i) \\
&= \sum_{i=1}^n\sum_{j\neq i}[\frac{1}{2}(\mathbf{u}_i^\top\bm{\theta}_j)^2+\frac{1}{2}(\mathbf{u}_j^\top\bm{\theta}_i)^2 + (\mathbf{u}_i^\top\bm{\theta}_j)(\mathbf{u}_j^\top\bm{\theta}_i)] \\
&= \frac{1}{2}\sum_{i=1}^n\sum_{j\neq i}(\mathbf{u}_i^\top\bm{\theta}_j + \mathbf{u}_j^\top\bm{\theta}_i)^2 \\
&\geq 0,
\end{split}
\end{equation*}
$\mathbf{T}$ is a positive semi-definite matrix.
\end{enumerate}

Combining the results above, $Jr(\bm{\theta})$ is a positive definite matrix which indicates that $\langle \mathcal{N}, \bm{\Theta}, (\widetilde{c}_i) \rangle$ has a unique NE.
As Theorem 3 points out, the game has at least one symmetric NE. 
Therefore, the NE is unique and must be symmetric.
\end{proof}

\section{Computing the Equilibrium}
\label{sec:find_solution}

Having shown that $\langle \mathcal{N}, \bm{\Theta}, (\widetilde{c}_i) \rangle$ has a unique symmetric Nash equilibrium, we now consider computing its solution.
We exploit the symmetry of the game which enables to reduce the search space of the game to only symmetric solutions.
Particularly, we derive the symmetric Nash equilibrium of $\langle \mathcal{N}, \bm{\Theta}, (\widetilde{c}_i) \rangle$ by solving a single convex optimization problem.
We obtain the following result.

\begin{theorem}
\label{thm:solution}
Let 
\begin{equation}
\label{eq:cvx}
f(\bm{\theta}) = \ell(\mathbf{X}\bm{\theta},\mathbf{y})+\frac{\beta(n+1)}{2\lambda^2}||\mathbf{z}-\mathbf{y}||_2^2(\bm{\theta}^\top\bm{\theta})^2,
\end{equation}
Then, the unique symmetric  NE of $\langle \mathcal{N}, \bm{\Theta}, (\widetilde{c}_i) \rangle$, $\{\bm{\theta}_i^{*}\}_{i=1}^n$, can be derived by solving the following convex optimization problem
\begin{equation}
\label{eq:solution}
\min_{\bm{\theta} \in \bm{\Theta}} f(\bm{\theta})
\end{equation}
and then letting $\bm{\theta}_i^{*} = \bm{\theta}^{*}, \forall i \in \mathcal{N}$, where $\bm{\theta}^{*}$ is the solution of Eq.~\eqref{eq:solution}.
\end{theorem}
\begin{proof}
We prove this theorem by characterizing the first-order optimality conditions of each learner's minimization problem in Eq.~\eqref{eq:mlne} with $c_i$ being replaced with its approximation $\widetilde{c}_i$. 
Let $\{\bm{\theta}_i^{*}\}_{i=1}^n$ be the NE, then it satisfies
\begin{equation}
\label{eq:first_order}
(\bm{\eta} - \bm{\theta}_i^{*})^\top\nabla_{\bm{\theta}_i}\widetilde{c}_i(\bm{\theta}_i^{*},\bm{\theta}_{-i}^{*}) \geq 0, \forall \bm{\eta} \in \bm{\Theta}, \forall i \in \mathcal{N}  
\end{equation}
where $\nabla_{\bm{\theta}_i}\widetilde{c}_i(\bm{\theta}_i^{*},\bm{\theta}_{-i}^{*})$ is the gradient of $\widetilde{c}_i(\bm{\theta}_i, \bm{\theta}_{-i})$ with respect to $\bm{\theta}_i$ and is evaluated at $\{\bm{\theta}_i^{*}\}_{i=1}^n$.
Then, Eq.~\eqref{eq:first_order} is equivalent to the equations as follows:
\begin{equation}
\label{eq:first_order2}
\left\{
\begin{aligned}
&(\bm{\eta} - \bm{\theta}_1^{*})^\top\nabla_{\bm{\theta}_1}\widetilde{c}_1(\bm{\theta}_1^{*},\bm{\theta}_{-1}^{*}) \geq 0, \forall \bm{\eta} \in \bm{\Theta},\\
&\bm{\theta}_1^{*} = \bm{\theta}_j^{*}, \forall j \in \mathcal{N}\setminus \{1\}  \\
\end{aligned}
\right.
\end{equation}
The reasons are: first, any solution of Eq.~\eqref{eq:first_order} satisfies Eq.~\eqref{eq:first_order2}, as $\{\bm{\theta}_i^{*}\}_{i=1}^n$ is symmetric;
Second, any solution of Eq.~\eqref{eq:first_order2} also satisfies Eq.~\eqref{eq:first_order}.
By definition of symmetric game, if $\bm{\theta}_1^{*} = \bm{\theta}_j^{*}$, then
$\nabla_{\bm{\theta}_1}\widetilde{c}_1(\bm{\theta}_1^{*},\bm{\theta}_{-1}^{*}) = \nabla_{\bm{\theta}_j}\widetilde{c}_j(\bm{\theta}_j^{*},\bm{\theta}_{-j}^{*})$,
and we have 
\[ (\bm{\eta} - \bm{\theta}_j^{*})^\top\nabla_{\bm{\theta}_j}\widetilde{c}_j(\bm{\theta}_j^{*},\bm{\theta}_{-j}^{*}), \forall \bm{\eta} \in \bm{\Theta}, \forall j \in \mathcal{N}\setminus\{1\} \]
Hence, Eq.~\eqref{eq:first_order} and Eq.~\eqref{eq:first_order2} are equivalent.
Eq.~\eqref{eq:first_order2} can be further rewritten as 
\begin{equation}
\label{eq:minimum_principle}
(\bm{\eta} - \bm{\theta}_1^{*})^\top\nabla_{\bm{\theta}_1}\widetilde{c}_1(\bm{\theta}_1^{*},\bm{\theta}_{-1}^{*})|_{\bm{\theta}_1^{*}=...=\bm{\theta}_n^{*}} \geq 0, \forall \bm{\eta} \in \bm{\Theta}.
\end{equation}
We then let
\begin{equation}
\label{eq:F}
\begin{split}
F(\bm{\theta}_1^{*}) &= \nabla_{\bm{\theta}_1}\widetilde{c}_1(\bm{\theta}_1^{*},\bm{\theta}_{-1}^{*})|_{\bm{\theta}_1^{*}=...=\bm{\theta}_n^{*}}\\
&= 2\mathbf{X}^\top(\mathbf{X}\bm{\theta}_1^{*}-\mathbf{y})+\frac{2\beta(n+1)}{\lambda^2}||\mathbf{z}-\mathbf{y}||_2^2{\bm{\theta}_1^{*}}^\top\bm{\theta}_1^{*}\bm{\theta}_1^{*}.
\end{split}
\end{equation} 
Then, $F(\bm{\theta}_1^{*}) = \nabla_{\bm{\theta}_1}f(\bm{\theta}_1^{*})$ where $f(\cdot)$ is defined in Eq.~\eqref{eq:cvx}.
Hence, we have
\begin{equation}
\label{eq:vi}
(\bm{\eta} - \bm{\theta}_1^{*})^\top\nabla_{\bm{\theta}_1}f(\bm{\theta}_1^{*}) \geq 0, \forall \bm{\eta} \in \bm{\Theta},
\end{equation}
This means that $\bm{\theta}_1^{*}$ is the solution of the optimization problem in Eq.\eqref{eq:solution} which finally completes the proof.      
\end{proof}

%The convex optimization problem presented in Eq.~\eqref{eq:solution} can be solved in polynomial time.
%Hence, our approach to derive the Nash equilibrium of $\langle \mathcal{N}, \bm{\Theta}, (\widetilde{u}_i) \rangle$  achives both effectiveness and efficiency.
A deeper look at Eq.~\eqref{eq:cvx} reveals that the Nash equilibrium can be obtained by each learner independently,  without knowing others' actions.
This means that the Nash equilibrium can be computed in a distributed manner while the convergence is still guaranteed.
Hence, our proposed approach is highly scalable, as increasing the number of learners does not impact the complexity of finding the Nash equilibrium.
We investigate the robustness of this equilibrium both using theoretical analysis and experiments in the remaining sections.
% \bm{\theta}^{\ast}

\section{Robustness Analysis}
\label{sec:robust}

We now draw a connection between the multi-learner equilibrium in the adversarial setting, derived above, and robustness, in the spirit of the analysis by \citet{xu2009robust}.
%We then analyze the robustness of the learners using the techniques introduced in \cite{xu2009robust}. 
Specifically, we prove the equivalence between Eq.~\eqref{eq:solution} and a robust linear regression problem where data is maliciously corrupted by some disturbance $\boldsymbol{\triangle}$. Formally, a robust linear regression solves the following problem:
	\begin{equation}\label{eq:min_max}
		\min_{\bm{\theta} \in \bm{\Theta}} \max_{\boldsymbol{\triangle} \in \mathcal{U} } {|| \mathbf{y} - (\mathbf{X} + \boldsymbol{\triangle})\bm{\theta}   ||_2^2},
	\end{equation}
where the uncertainty set $\mathcal{U}=\{ \boldsymbol{\triangle} \in \mathbb{R}^{m \times d} \, | \, {\boldsymbol{\triangle}}^T{\boldsymbol{\triangle}}=\mathbf{G}:  |\mathbf{G}_{ij}| \le c|{\theta}_i {\theta}_j| \,\,\, \forall i,j \}$, with $c=\frac{\beta(n+1)}{2\lambda^2}{|| \mathbf{z}- \mathbf{y} ||_2^2}$. Note that $\bm{\theta}$ is a vector and ${\theta}_i$ is the $i$-th element of $\bm{\theta}$.

%which results in a min-max problem shown in Eq.\eqref{eq:min_max}. 
From a game-theoretic point of view, in training phase the defender is simulating an attacker. The attacker maximizes the training error by adding disturbance to $\mathbf{X}$. The magnitude of the disturbance is controlled by a parameter $c=\frac{\beta(n+1)}{2\lambda^2}{|| \mathbf{z}- \mathbf{y} ||_2^2}$. Consequently, the robustness of Eq.~\eqref{eq:min_max} is guaranteed if and only if the magnitude reflects the uncertainty interval. This sheds some light on how to choose $\lambda$, $\beta$ and $\mathbf{z}$ in practice. One strategy is to over-estimate the attacker's strength, which amounts to choosing small values of $\lambda$, large values of $\beta$ and exaggerated target $\mathbf{z}$. The intuition of this strategy is to enlarge the uncertainty set so as to cover potential adversarial behavior. In Experiments section we will show this strategy works well in practice. Another insight from Eq.~\eqref{eq:min_max} is that the fundamental reason \textit{MLSG} is robust is because it proactively takes adversarial behavior into account.

\begin{theorem}
\label{thm:robustness}
The optimal solution $\bm{\theta}^{\ast}$ of the problem in Eq.~\eqref{eq:solution} is an optimal solution to the robust optimization problem in Eq.~\eqref{eq:min_max}.
%the followring robust optimization problem with the uncertainty set $\mathcal{U}=\{ \boldsymbol{\triangle} \in \mathbb{R}^{m \times d} \, | \, {\boldsymbol{\triangle}}^T{\boldsymbol{\triangle}}=\mathbf{G}: \forall i\ne j, \,\, \mathbf{G}_{ii} \le c {|\bm{\theta}^{\ast}_i|}^2  \text{ and } sign(\tau)\mathbf{G}_{ij} \le  sign(\tau)\tau c \}$, where $c=\frac{\beta(n+1)}{2\lambda^2}{|| \mathbf{z}- \mathbf{y} ||_2^2}$, $\tau=\bm{\theta}^{\ast}_i \bm{\theta}^{\ast}_j$, $sign(\cdot)$ is the sign function and $\bm{\theta}^{\ast}$ is the optimal solution of Eq.~\eqref{eq:solution}.
\end{theorem}

\begin{proof}
Fix $\bm{\theta}^{\ast}$, we show that 
\[\max_{\boldsymbol{\triangle} \in \mathcal{U}} {|| \mathbf{y} - (\mathbf{X} + \boldsymbol{\triangle})\bm{\theta}^{\ast}   ||}_2^2 = {|| \mathbf{y} - \mathbf{X}\bm{\theta}^{\ast} ||_2^2} + c{( {\bm{\theta}^{\ast}}^T \bm{\theta}^{\ast}   )}^2.\]
The left-hand side can be expanded as: 
		\begin{align*}
%			\begin{split}
				& \max_{\boldsymbol{\triangle} \in \mathcal{U}} {|| \mathbf{y} - (\mathbf{X} + \boldsymbol{\triangle})\bm{\theta}^{\ast}   ||_2^2} \\
			 =  & \max_{\boldsymbol{\triangle} \in \mathcal{U}} {||  \mathbf{y} - \mathbf{X}\bm{\theta}^{\ast} - \boldsymbol{\triangle}\bm{\theta}^{\ast} ||_2^2} \\
			 \le &  \max_{\boldsymbol{\triangle} \in \mathcal{U}} {||\mathbf{y} - \mathbf{X}\bm{\theta}^{\ast} ||_2^2} +  \max_{\boldsymbol{\triangle} \in \mathcal{U}} { || \boldsymbol{\triangle} \bm{\theta}^{\ast} ||_2^2 } \\ 
			 = & \max_{\boldsymbol{\triangle} \in \mathcal{U}} {||\mathbf{y} - \mathbf{X}\bm{\theta}^{\ast} ||_2^2} + \max_{\boldsymbol{\triangle} \in \mathcal{U}} { {\bm{\theta}^{\ast}}^T {\boldsymbol{\triangle}}^T{\boldsymbol{\triangle}}  \bm{\theta}^{\ast}} \\
			 & \text{(substitute ${\boldsymbol{\triangle}}^T{\boldsymbol{\triangle}}=\mathbf{G}$)} \\
			 = &  {||\mathbf{y} - \mathbf{X}\bm{\theta}^{\ast} ||_2^2} + \max_{\mathbf{G} } {{\bm{\theta}^{\ast}}^T \mathbf{G} \bm{\theta}^{\ast}} \\
			 = &  {||\mathbf{y} - \mathbf{X}\bm{\theta}^{\ast} ||_2^2}  +  \max_{\mathbf{G} } {\sum_{i=1}^{d}{ {|{\theta}^{\ast}_i|}^2 \mathbf{G}_{ii}  }  } + {2\sum_{j=1}^{d}{   \sum_{i=1}^{j-1}{  {\theta}^{\ast}_i {\theta}^{\ast}_j \mathbf{G}_{ij}        }        }} \\
			 \le & {||\mathbf{y} - \mathbf{X}\bm{\theta}^{\ast} ||_2^2} +  c{\sum_{i=1}^{d}{ {|{\theta}^{\ast}_i|}^4  }  } + {2c\sum_{j=1}^{d}{   \sum_{i=1}^{j-1}{  ({\theta}^{\ast}_i {\theta}^{\ast}_j)^2      }        }} \\
			 = & {||\mathbf{y} - \mathbf{X}\bm{\theta}^{\ast} ||_2^2} + c\big( \sum_{i=1}^{d}{|{\theta}^{\ast}_i|^2} \big)^2 \\
			 = & {||\mathbf{y} - \mathbf{X}\bm{\theta}^{\ast} ||_2^2} + c{( {\bm{\theta}^{\ast}}^T \bm{\theta}^{\ast}   )}^2.
%			\end{split}
		\end{align*}

Now we define $\boldsymbol{\triangle}^{\ast}=[\sqrt{c}{\bm{\theta} }^{\ast}_1 \mathbf{u},\cdots,\sqrt{c}{\bm{\theta} }^{\ast}_n \mathbf{u}]$, where $\bm{\theta}^{\ast}_i$ is the $i$-th element of $\bm{\theta}^{\ast}$ and $\mathbf{u}$ is defined as:
  \begin{equation}
    \mathbf{u} \triangleq
    \begin{cases}
	       \frac{\mathbf{y} - \mathbf{X}\bm{\theta}^{\ast}}{|| \mathbf{y} - \mathbf{X}\bm{\theta}^{\ast} ||_2}, &  \textit{if}\ \mathbf{y} \ne \mathbf{X}\bm{\theta}^{\ast} \\
	      \textit{any vector with unit $L_2$ norm}, & \text{otherwise}
	    \end{cases}
  \end{equation}

Then we have:
		\begin{align}
			\begin{split}
			    & \max_{\boldsymbol{\triangle} \in \mathcal{U}} {\left\| \mathbf{y} - (\mathbf{X} + \boldsymbol{\triangle})\bm{\theta}^{\ast}   \right\| }_2^2 \\
			\ge &  {|| \mathbf{y} - (\mathbf{X} + \boldsymbol{\triangle}^{\ast})\bm{\theta}^{\ast}   ||}_2^2 \\
			=   & {|| \mathbf{y} - \mathbf{X}\bm{\theta}^{\ast} - \boldsymbol{\triangle}^{\ast}\bm{\theta}^{\ast}   ||}_2^2 \\
			= & {|| \mathbf{y} - \mathbf{X}\bm{\theta}^{\ast} - \sum_{i=1}^{d}{  \sqrt{c}|{\theta}^{\ast}_i|^2 \mathbf{u}    }  ||_2^2} \\
			& \text{($\mathbf{u}$ is in the same direction as $\mathbf{y} - \mathbf{X}\bm{\theta}^{\ast}$ )} \\
			= & {|| \mathbf{y} - \mathbf{X}\bm{\theta}^{\ast} ||_2^2} + {||  \sum_{i=1}^{d}{  \sqrt{c}|{\theta}^{\ast}_i|^2 \mathbf{u}    }  ||_2^2} \\
			= & {|| \mathbf{y} - \mathbf{X}\bm{\theta}^{\ast} ||_2^2} + c{( {\bm{\theta}^{\ast}}^T \bm{\theta}^{\ast}   )}^2
			\end{split}
		\end{align}
\end{proof} 

% \bm{\theta}^{\ast}

\section{Experiments}
\label{sec:experiments}

As previously discussed, a dataset is represented by $(\mathbf{X}, \mathbf{y})$, where $\mathbf{X}$ is the feature matrix and $\mathbf{y}$ are labels. 
We use $(\mathbf{x}_j, \mathbf{y}_j)$ to denote the $j$-th instance and its corresponding label. 
The dataset is equally divided into a training set $(\mathbf{X}_{\text{train}}, \mathbf{y}_{\text{train}})$ and a testing set $(\mathbf{X}_{\text{test}}, \mathbf{y}_{\text{test}})$. 
We conducted experiments on three datasets: Wine Quality (redwine), Boston Housing Market (boston), and PDF malware (PDF). 
The number of learners is set to 5. 
%Due to space limitation the experimental results for the boston dataset are included in supplement.

The Wine Quality dataset \cite{cortez2009modeling} contains 1599 instances and each instance has 11 features. 
Those features are physicochemical and sensory measurements for wine. 
The response variables are quality scores ranging from 0 to 10, where 10 represents for best quality and 0 for least quality. 
The Boston Housing Market dataset \cite{harrison1978hedonic} contains information collected by U.S Census Service in the area of Boston. 
This dataset has 506 instances and each instance has 13 features. 
Those features are, for example, per capita crime rate by town and average number of rooms per dwelling etc. 
The response variables are median home prices. 
The PDF malware dataset consists of 18658 PDF files collected from the internet.  
We employed an open-sourced tool \textit{mimicus}\footnote{https://github.com/srndic/mimicus} to extract 135 real-valued features from PDF files~\citep{laskov2014practical}. %This tool has been utilized to analyze PDF in \cite{laskov2014practical}. 
%The extracted features' dimensions are 135. 
We then applied \textit{peepdf}\footnote{https://github.com/rohit-dua/peePDF} to score each PDF between 0 and 10, with a higher score indicating greater likelihood of being malicious.

%A score ranges between 0 and 10, where a larger score indicates higher probability of being malicious. The distribution of PDF scoring is showed in Figure~\ref{fig:scoring_PDF}, where blue and orange represents benign and malicious PDF, respectively. Although there is some overlapping of the scoring, \textit{peepy} gives a reasonably good criterion to distinguish benign and malicious PDF.

%\begin{figure}[H]
%\centering
%\begin{tabular}{c}
%\includegraphics[width=2.4in]{figure/PDF_scoring.pdf}
%\end{tabular}
%\caption{Scoring of PDF files generated by \textit{peepy}}
%\label{fig:scoring_PDF}
%\end{figure}

%For simplicity we denote our proposed algorithm as \textit{MLSG} (\textit{\textbf{M}ulti-\textbf{L}earner \textbf{S}tackelberg \textbf{G}ame}). 
Throughout, we abbreviate our proposed approach as \textit{MLSG}, and compare it to 
%We compared \textit{MLSG} with 
three other algorithms: ordinary least squares (\textit{OLS}) regression, as well as \textit{Lasso}, and Ridge regression (\textit{Ridge}). \textit{Lasso} and \textit{Ridge} are ordinary least square with $L_1$ and $L_2$ regularizations. 
In our evaluation, we simulate the attacker for different values of $\beta$ (the probability that a given instance is maliciously manipulated).
% When evaluating each algorithm we simulated an attacker. 
The specific attack targets $\mathbf{z}$ vary depending on the dataset; we discuss these below.
For our evaluation, we compute model parameters (for the equilibrium, in the case of \textit{MLSG}) on training data.
We then use test data to compute optimal attacks, characterized by Eq.~\eqref{eq:best_response}.
Let $\mathbf{X}_{\text{test}}'$ be the test feature matrix after adversarial manipulation, $\hat{\mathbf{y}}_{\text{test}}^A$ the associated predicted labels on manipulated test data, $\hat{\mathbf{y}}_{\text{test}}$ predicted labels on untainted test data, and $\mathbf{y}_{\text{test}}$ the ground truth labels for test data.
We use root expected mean square error (RMSE) as an evaluation metric, where the expectation is with respect to the probability $\beta$ of a particular instance being maliciously manipulated:
%Thus, RMSE in our case is
$\sqrt{\frac{\beta {(\hat{\mathbf{y}}_{\text{test}}^A - \mathbf{y}_{\text{test}})}^T{(\hat{\mathbf{y}}_{\text{test}}^A - \mathbf{y}_{\text{test}})}+(1-\beta) {(\hat{\mathbf{y}}_{\text{test}} - \mathbf{y}_{\text{test}})}^T{(\hat{\mathbf{y}}_{\text{test}} - \mathbf{y}_{\text{test}})}}{N}}$, where $N$ is the size of the test data.

\textbf{The redwine dataset}: Recall that the response variables in redwine dataset are quality scores ranging from 0 to 10. 
We simulated an attacker whose target is to increase the overall scores of testing data. 
In practice this could correspond to the scenario that wine sellers try to manipulate the evaluation of third-party organizations.  
We formally define the attacker's target as $\mathbf{z}=\mathbf{y} + \Delta$, where $\mathbf{y}$ is the ground-truth response variables and $\Delta$  is a real-valued vector representing the difference between the attacker's target and the ground-truth. Since the maximum score is 10, any element of $\mathbf{z}$ that is greater than 10 is clipped to 10. 
We define $\Delta$ to be homogeneous (all elements are the same); generalization to heterogeneous values is direct.
%, although it is direct to generalize to heterogeneous ( elements can be different ) setting.  
The mean and standard deviation of $\mathbf{y}$ are $\mu_r=5.64$ and $\sigma_r=0.81$. We let $\Delta=5\sigma_r \times \mathbf{1}$, where $\mathbf{1}$ is a vector with all elements equal to one. 
The intuition for this definition is to simulate the generating process of adversarial data. 
Specifically, by setting the attacker's target to an unrealistic value (i.e. in current case outside the $3-\sigma_r$ of $\mu_r$), the generated adversarial data $\mathbf{X}^{'}$ is supposed to be intrinsically different from $\mathbf{X}$. 
%This results in $(\mathbf{X}^{'}, \mathbf{y}^{'}) \sim \mathcal{D}^{'}$, where as mentioned in Section.~\ref{sec:model} the distribution $\mathcal{D}^{'}$ is  different from the normal data generating distribution. Next we will show two experimental results to support the robustness of \textit{MLSG}. 
For ease of exposition we use the term \textit{defender} to refer to \textit{MLSG}.

Remember in Eq.\eqref{eq:3rd_bound} there are three hyper-parameters in the defender's loss function: $\lambda$, $\beta$, and $\mathbf{z}$. $\lambda$ is the regularization coefficient in the attacker's loss function Eq.\eqref{eq:attacker_true_cost}. It is negatively proportional to the attacker's strength. $\beta$ is the probability of a test data being malicious. In practice these three hyper-parameters are externally set by the attacker. In the first experiment we assume the defender knows the values of these three hyper-parameters, which corresponds to the \textit{best case}. 
%The motivation of this experiment is to see whether it is helpful to proactively consider adversarial behavior in training phase.  
The result is shown in Figure~\ref{fig:redwine_best_case}. Each bar is averaged over 50 runs, where at each run we randomly sampled training and test data. The regularization parameters of \textit{Lasso} and \textit{Ridge} were selected by cross-validation. Figure~\ref{fig:redwine_best_case} demonstrates that \emph{MLSG} approximate equilibrium solution is significantly more robust than conventional linear regression learning, with and without regularization.
%indicates that it is indeed useful to proactively take adversarial behavior into account. 

\begin{figure}[h]
\centering
\begin{tabular}{c}
\includegraphics[width=2.2in]{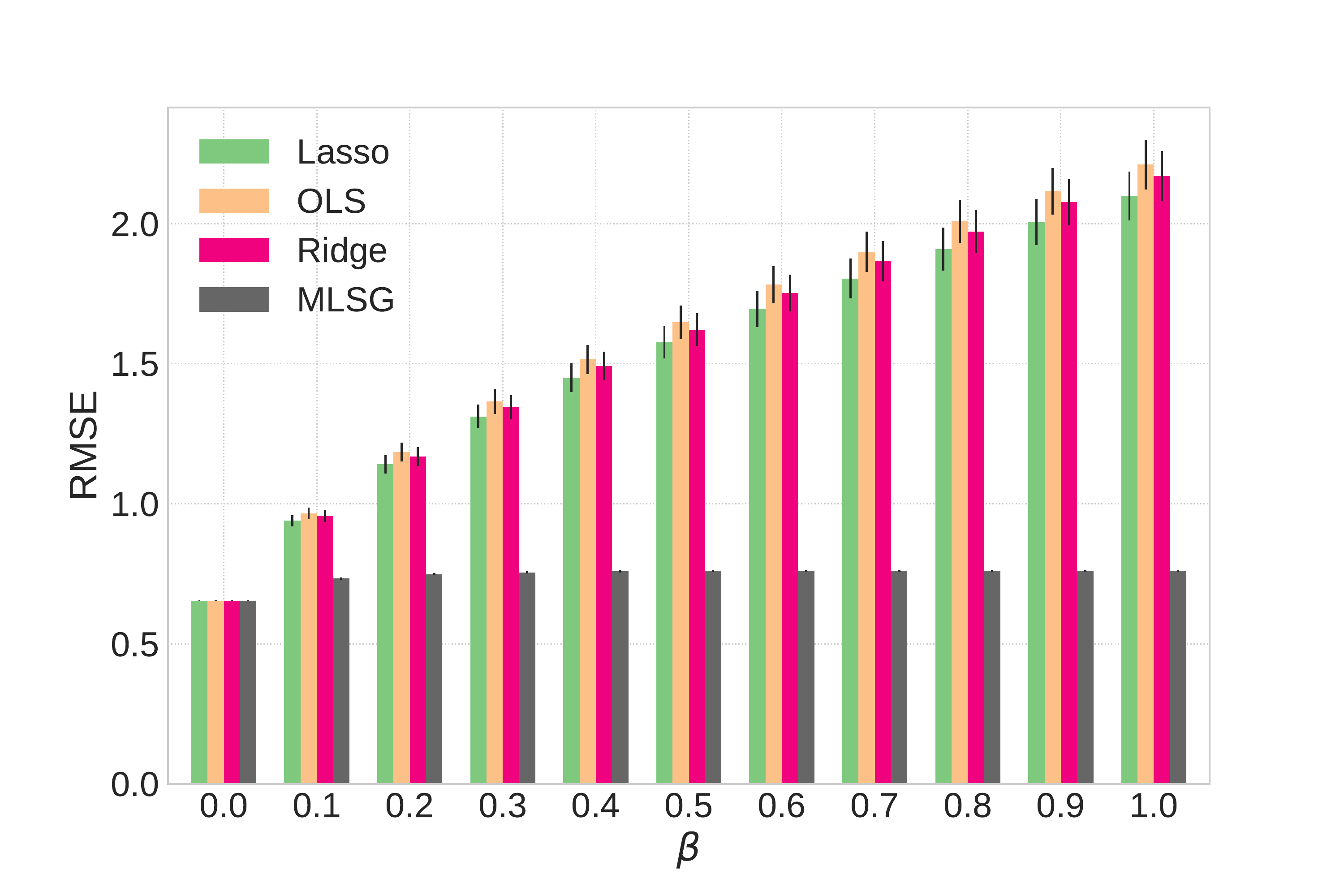}
\end{tabular}
\caption{RMSE of $\mathbf{y}^{'}$ and $\mathbf{y}$ on redwine dataset. The defender knows $\lambda$, $\beta$, and $\mathbf{z}$.}
\label{fig:redwine_best_case}
\end{figure}

In the second experiment we relaxed the assumption that the defender knows $\lambda$, $\beta$ and $\mathbf{z}$, and instead simulated the practical scenario that the defender obtains estimates for these (for example, from historical attack data), but the estimates have error.
%has some estimation about the attacker (i.e. from historical attacking data), but there is discrepancy between the estimation and the actual adversarial behavior. 
We denote by $\hat{\lambda}=0.5$ and $\hat{\beta}=0.8$ the defender's estimates of the true $\lambda$ and $\beta$.\footnote{We tried alternative values of $\hat{\lambda}$ and $\hat{\beta}$, and the results are consistent. Due to space limitations we include them in supplemental materials.} 
Remember that $\beta$ is the probability of an instance being malicious and $\lambda$ is negatively proportional to the attacker's strength. So the estimation characterizes a \textit{pessimistic} defender that is expecting very strong attacks. We experimented two kinds of estimation about $\mathbf{z}$: 1) the defender overestimates $\mathbf{z}$: $\hat{\mathbf{z}} = t\mathbf{1}$, where $t$ is a random variable sampled from a uniform distribution over $[5\sigma_r, 10]$; and 2) the defender underestimates $\mathbf{z}$: $\hat{\mathbf{z}} = t\mathbf{1}$, where $t$ is sampled from $[0, 5\sigma_r]$. 
Due to space limitations we only present the results for the latter; the former can be found in the Appendix.
%for underestimate setting  in Figure~\ref{fig:redwine_underestimate}. 
%The result for overestimate setting is included in supplemental materias. 
In Figure~\ref{fig:redwine_underestimate} the y-axis represents the actual values of $\lambda$, and the x-axis represents the actual values of $\beta$. The color bar on the right of each figure visualizes the average RMSE. Each cell is averaged over 50 runs. The result shows that even if there is a discrepancy between the defender's estimation and the actual adversarial behavior, \textit{MLSG} is consistently more robust than the other approaches. 

\begin{figure}[h]
% \centering
	\begin{tabular}{cc}
		\includegraphics[width=1.5in]{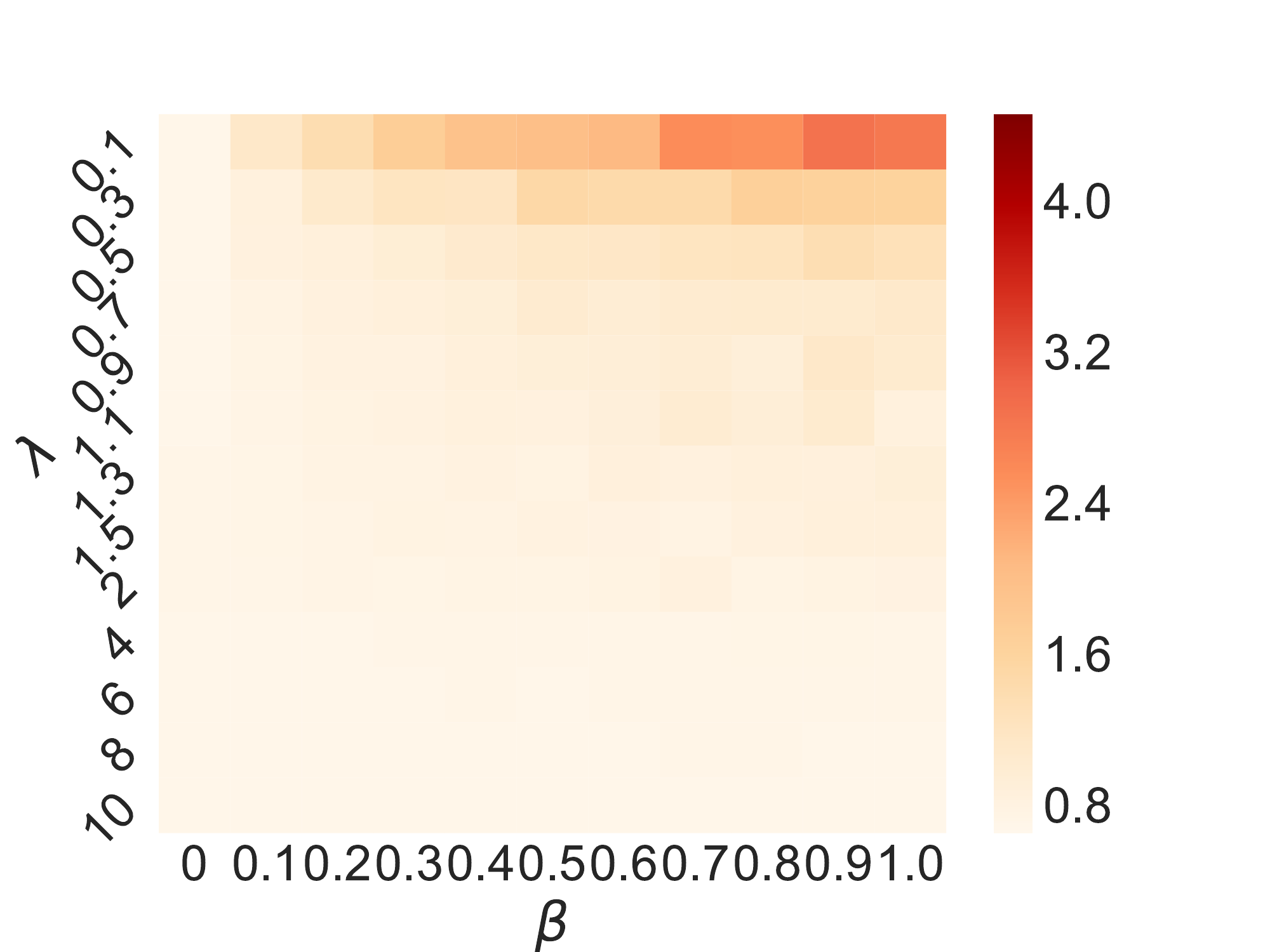} & \includegraphics[width=1.5in]{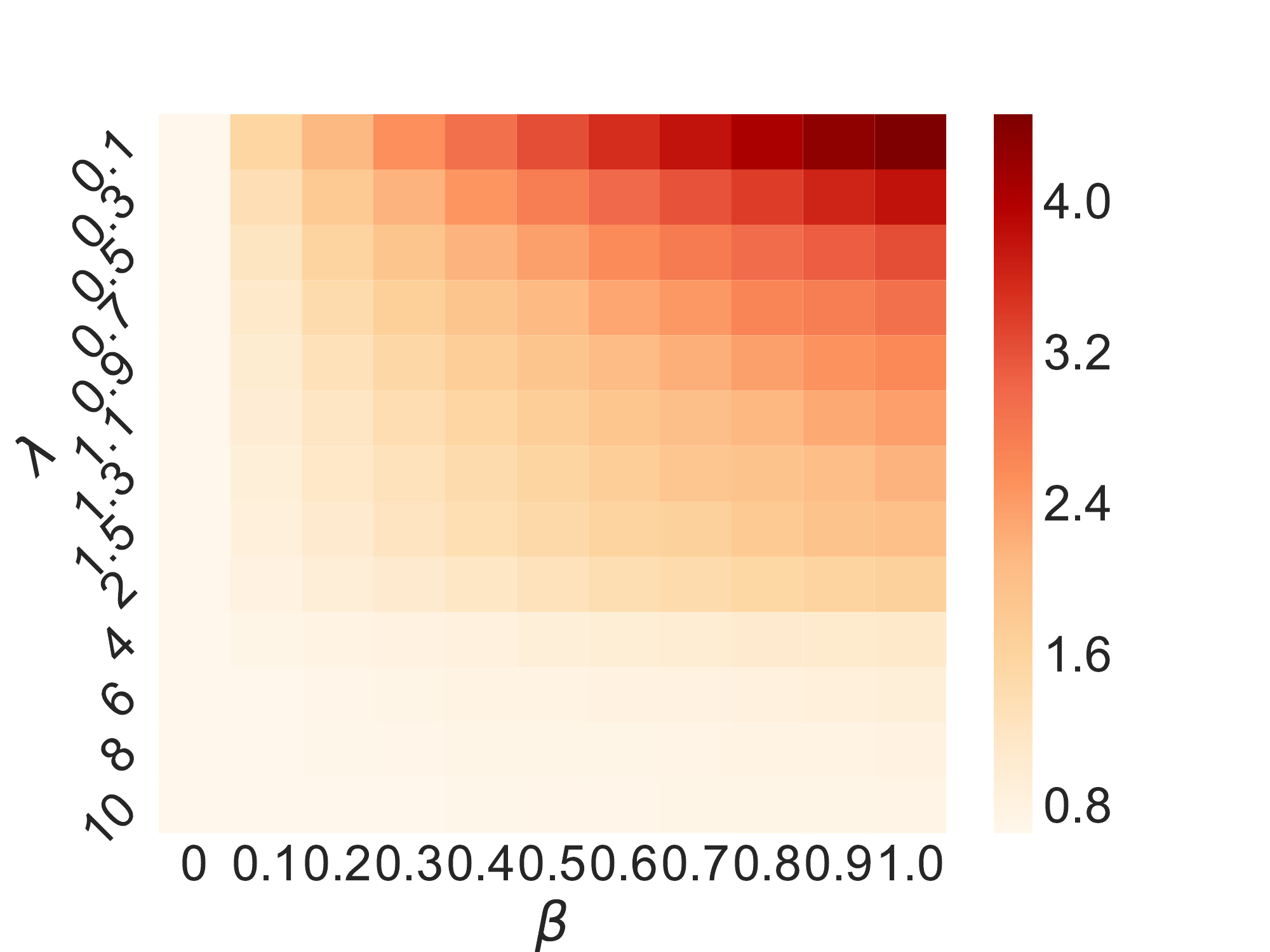}
	\end{tabular}
		\begin{tabular}{cc}
		\includegraphics[width=1.5in]{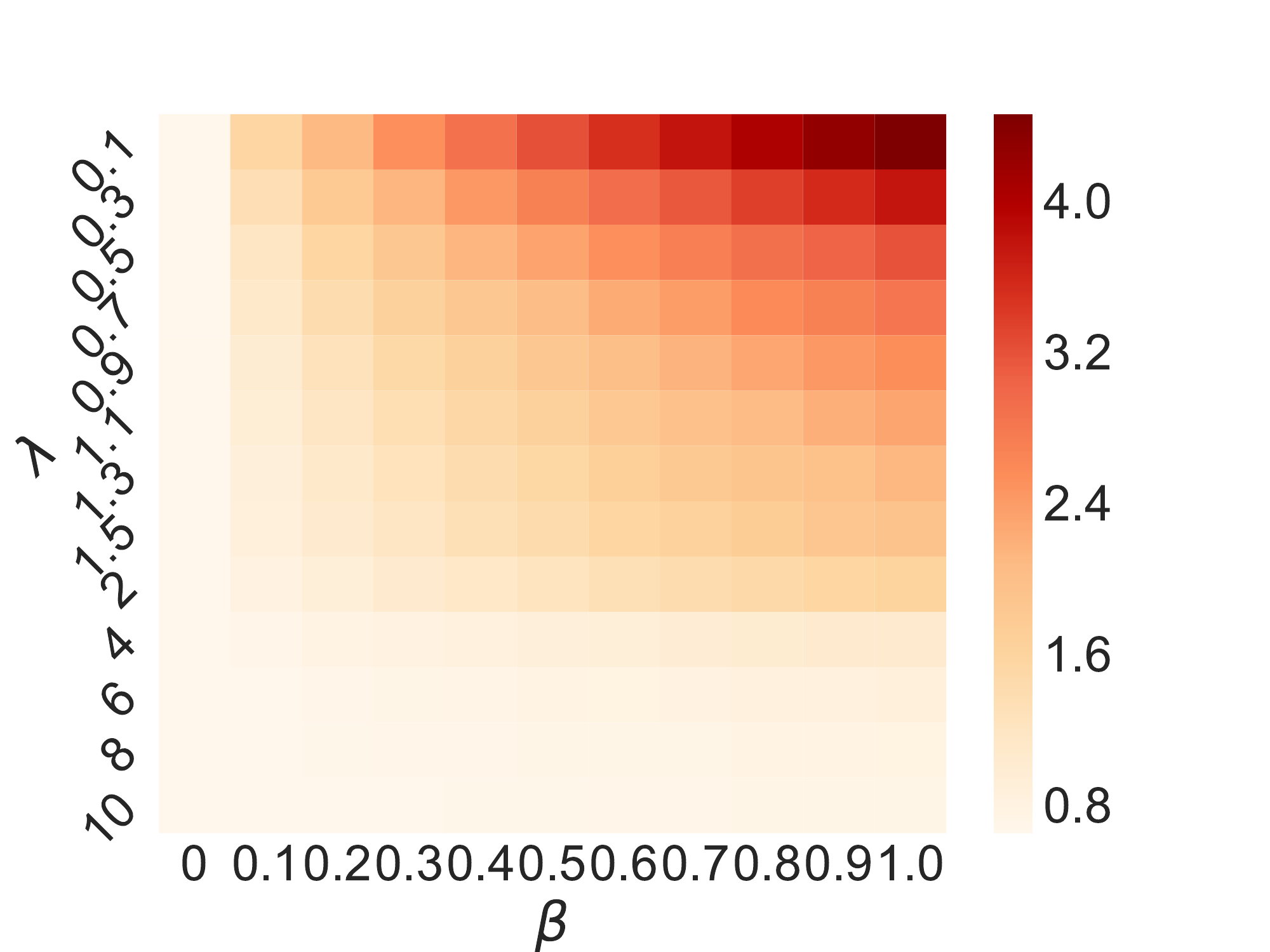} & \includegraphics[width=1.5in]{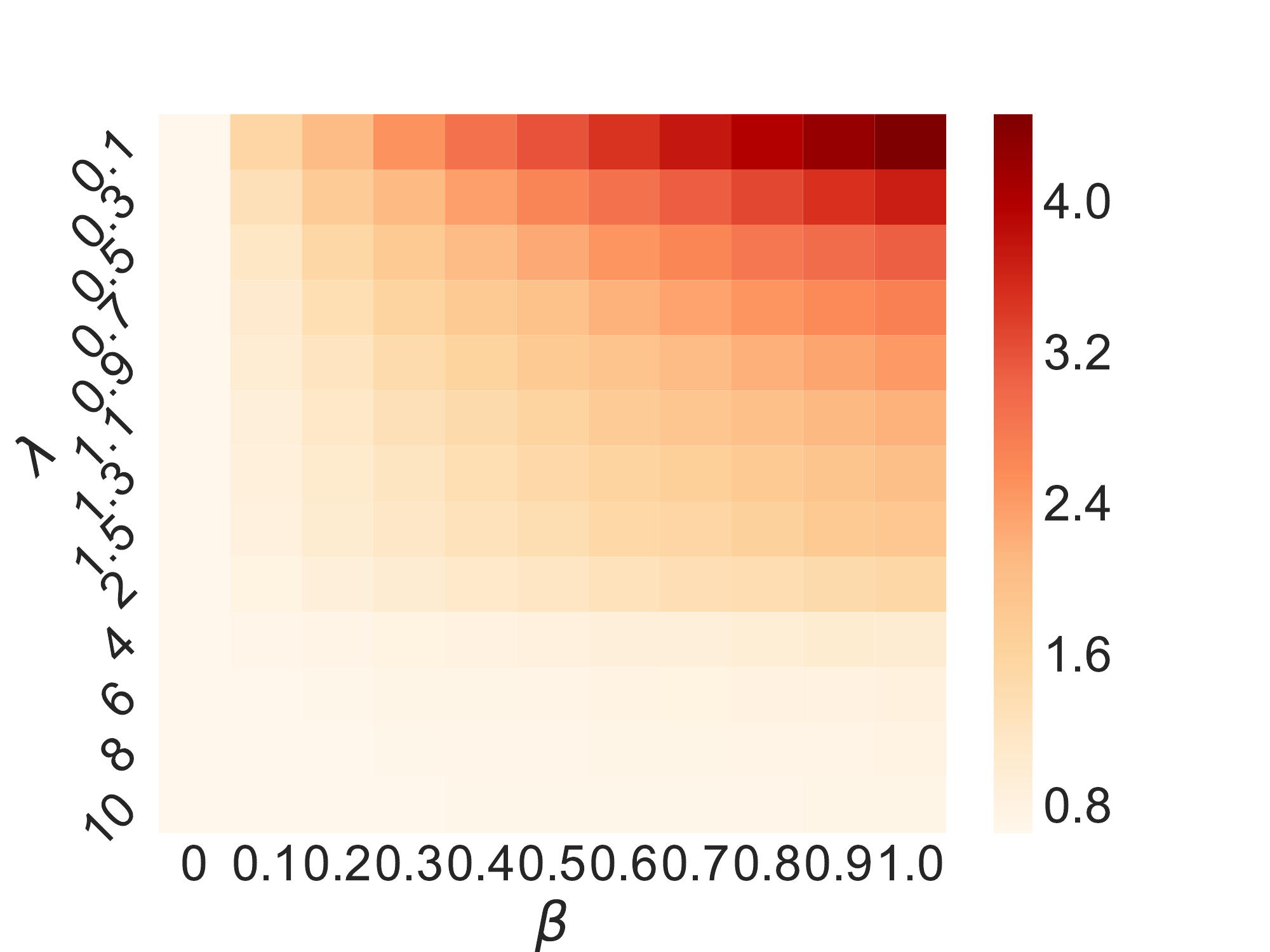}
	\end{tabular}
	\caption{The average RMSE across different values of actual $\lambda$ and $\beta$ on redwine dataset. From left to right: \textit{MLSG}, \textit{Lasso}, \textit{Ridge}, \textit{OLS}.}
	\label{fig:redwine_underestimate}
\end{figure}

 \textbf{The boston dataset}: The response variables in boston dataset are the median prices of houses in Boston area. We define the attacker's target as $\mathbf{z}=\mathbf{y} + \Delta$. The mean and standard deviation of $\mathbf{y}$  are $\mu_b=22.53$ and $\sigma_b=9.20$. We let $\Delta=2\sigma_b \times \mathbf{1}$. 
 %Unlike what we did on redwine dataset, we did not let $\mathbf{z}$ to be extremely unrealistic, for example outside of $3\sigma_b$ of $\mu_b$. The reason is for numerical stability concerns (a very unrealistic $\mathbf{z}$  causes challenges for the defender's optimization problem). In addition, setting $\mathbf{z}=2\sigma_b \times \mathbf{1}$ can  generate malicious enough data in our experiments. 
 We first present the experimental result for \textit{best case} in Figure~\ref{fig:boston_best_case}. The result further supports our claim that proactively considering adversarial behavior can improve robustness. 

 \begin{figure}[h]
 \centering
 \begin{tabular}{c}
 \includegraphics[width=2.6in]{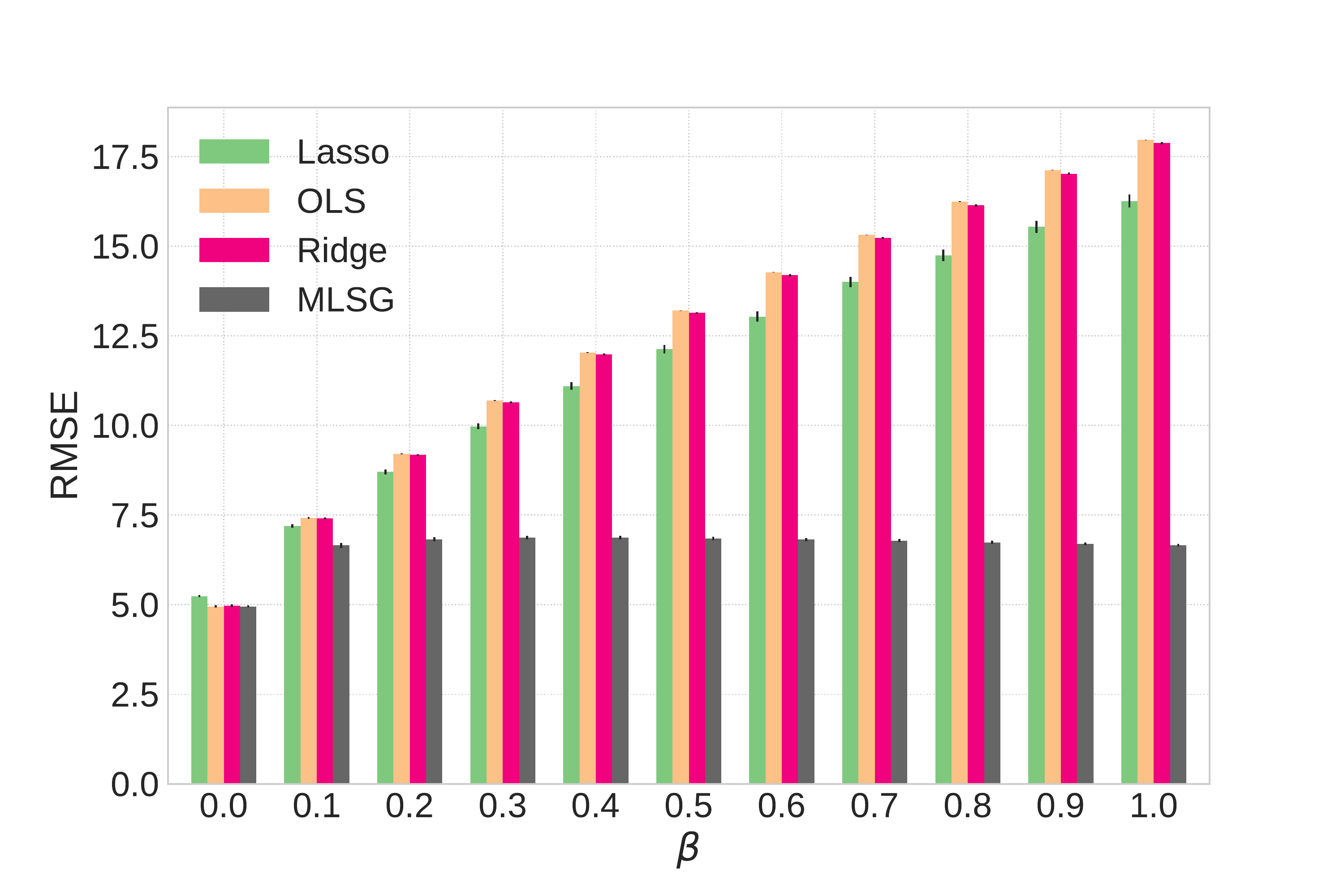}
 \end{tabular}
 \caption{RMSE of $\mathbf{y}^{'}$ and $\mathbf{y}$ on boston dataset. The defender knows $\lambda$, $\beta$, and $\mathbf{z}$.}
 \label{fig:boston_best_case}
 \end{figure}

We then relaxed the assumption that the defender knows $\lambda$, $\beta$ and $\mathbf{z}$ and set the defender's estimates of the true $\lambda$ and $\beta$ as $\hat{\lambda}=0.1$ and $\hat{\beta}=0.3$. Similar to the experimental setting used on redwine dataset we experimented with overestimation and underestimation of $\mathbf{z}$. 
We show the results for overestimation in Figure~\ref{fig:boston_overestimate},  where $\hat{\mathbf{z}} = t\mathbf{1}$ and $t$ is sampled from $[2\sigma_b, 3\sigma_b]$.  
From Figure~\ref{fig:boston_overestimate} we have several interesting observations.
First, \textit{MLSG} is more robust or equivalently robust compared with the three baselines, except when $\beta=0$.
This is to be expected, as $\beta = 0$ corresponds to non-adversarial data.
%This is actually expected because the defender's loss function as defined in \eqref{eq:learner_true_loss} is a tradeoff between the test error on adversarial data and the test error on normal data. When $\beta=0$ the defender's estimation are all overestimated, which means it should have been greatly attacked but the test data given to the defender is purely normal. Therefore in such cases \textit{MLSG} is not brought into fullest play; 
Second, we observe that for large $\beta$, \textit{Lasso} (middle left) and \textit{Ridge} (middle right) are less robust than \textit{OLS} (non-regularized linear regression) (rightmost).
This is surprising, as \citet{yang2013unified} showed that Lasso-like algorithms (including \textit{Lasso} and \textit{Ridge}) are robust in the sense that their original optimization problems can be converted to \textit{min-max} optimization problems, where the inner maximization takes malicious uncertainty about data into account. 
However, such robustness is only guaranteed if the regularization constant reflects the uncertainty interval; what our results suggest is when the interval bounds violated (i.e., regularization parameter is smaller than necessary), regularized models can actually degrade quickly.
%But from Figure~\ref{fig:boston_overestimate} when $\beta$ increases \textit{Lasso} (Upper Right) and \textit{Ridge} (Lower Left) are less robust than \textit{OLS} (Lower Right). This phenomenon can be explained from a game-theoretic perspective, where in training phase \textit{Lasso} and \textit{Ridge} are indeed rehearsing an attacker. This attacker nudges the training data within an uncertainty set in order to maximize the training error. Note that the uncertainty set is determined by the regularization. Therefore when some perturbation happens to the test data, \textit{Lasso} and \textit{Ridge} are robust as long as the perturbation is still within the uncertainty set. Unfortunately, our proposed attacker could generate test data outside of the uncertainty set so \textit{Lasso} and \textit{Ridge} backfire. 

 \begin{figure}[h]
 % \centering
 	\begin{tabular}{cc}
 		\includegraphics[width=1.6in]{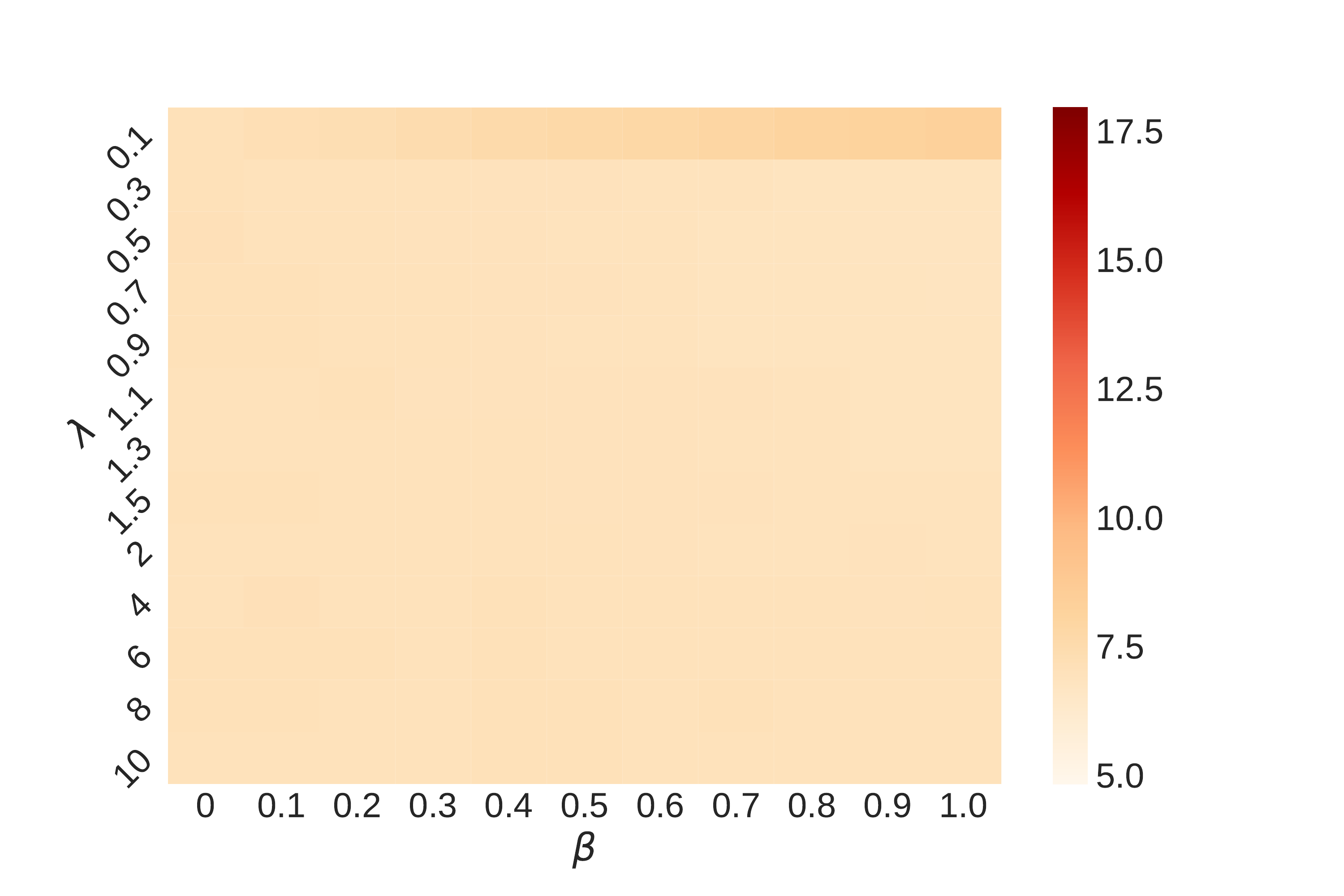} & \includegraphics[width=1.6in]{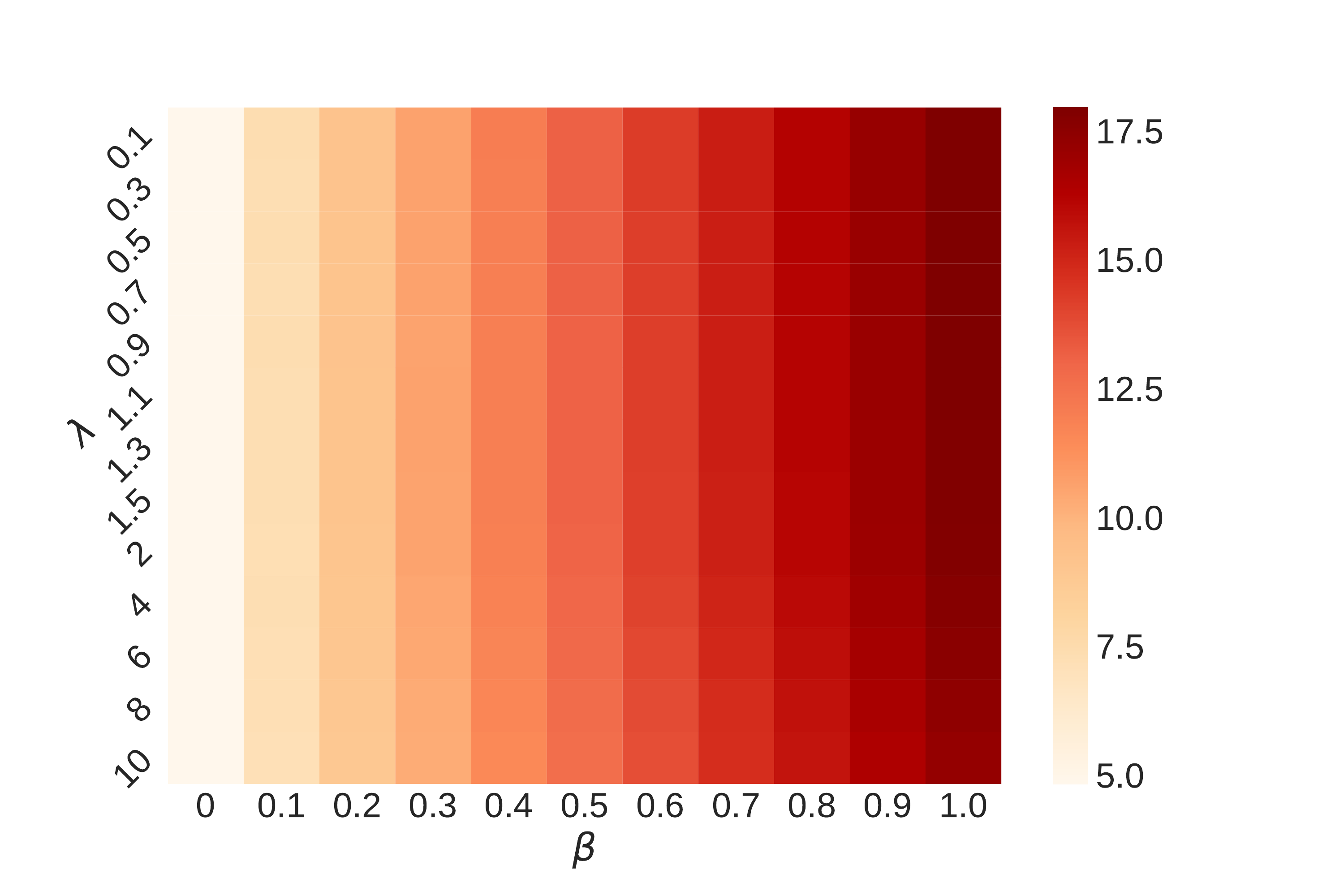}
 	\end{tabular}
 		\begin{tabular}{cc}
 		\includegraphics[width=1.6in]{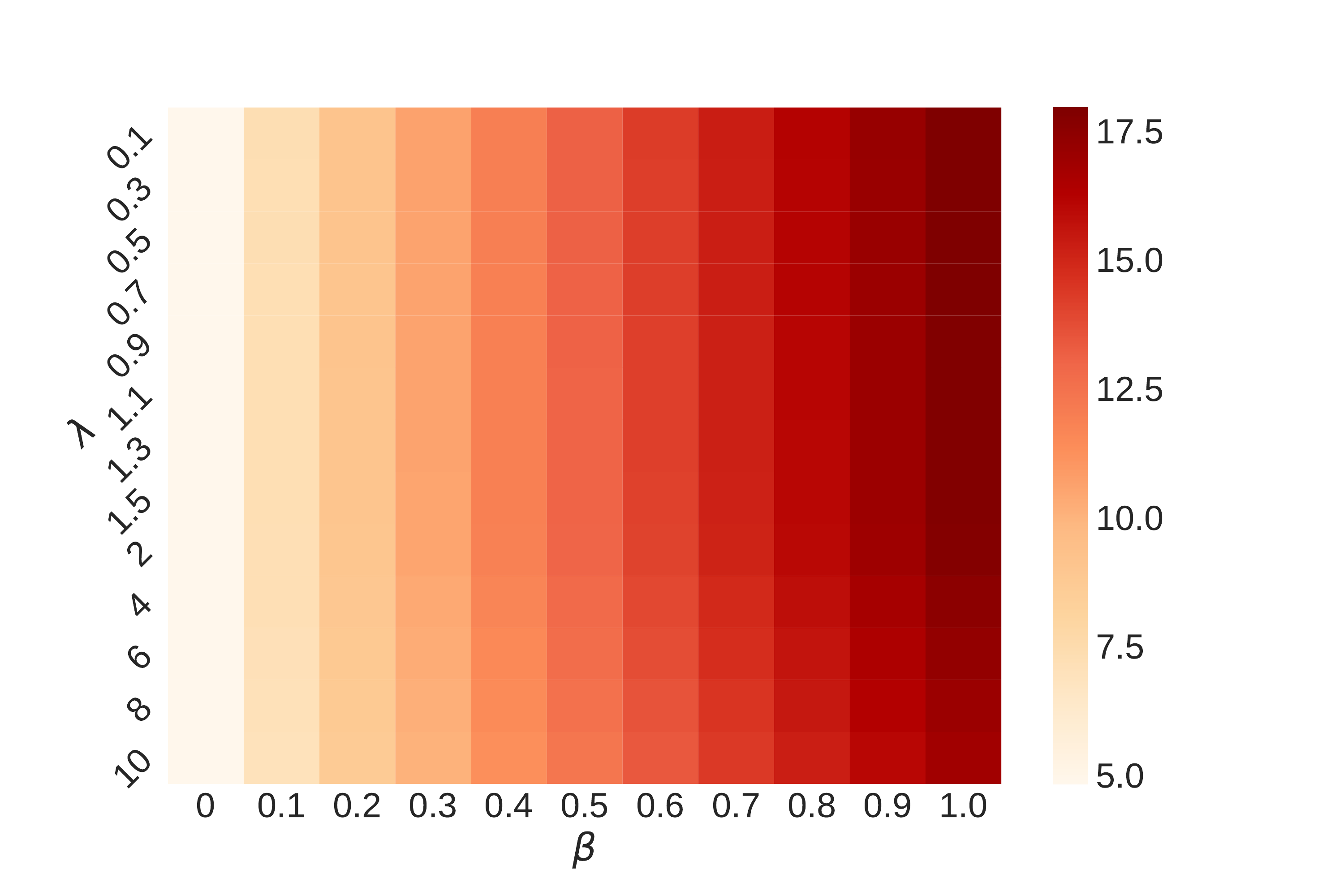} & \includegraphics[width=1.6in]{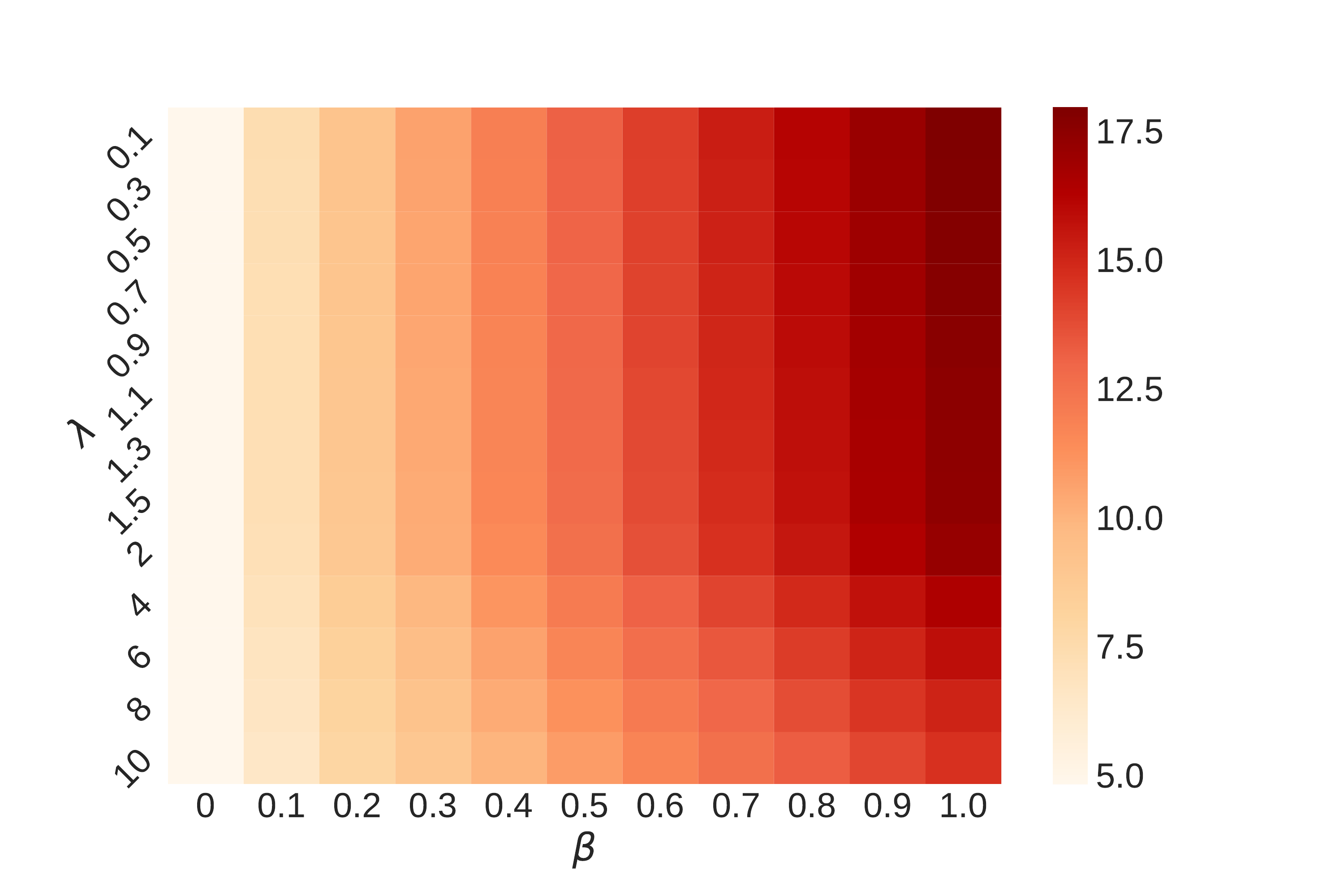}
 	\end{tabular}
 	\caption{The average RMSE across different values of actual $\lambda$ and $\beta$ on boston dataset. From left to right: \textit{MLSG}, \textit{Lasso}, \textit{Ridge}, \textit{OLS}.}
 	\label{fig:boston_overestimate}
 \end{figure}

\textbf{The PDF dataset}: The response variables of this dataset are malicious scores ranging between 0 and 10. The mean and standard deviation of $\mathbf{y}$ are $\mu_p=5.56$ and $\sigma_p=2.66$. Instead of letting the $\Delta$ be non-negative as in previous two datasets, the attacker's target is to descrease the scores of malicious PDFs. Consequently, we define $\Delta=-2\sigma_p \times \mathbf{1}_{\mathcal{M}}$, where $\mathcal{M}$ is the set of indices of malicious PDF and 
%$\mathbf{1}_{\mathcal{M}}=\{ \mathbf{1}_i=1, \forall i \in \mathcal{M}; \mathbf{1}_i=0 \forall i \notin \mathcal{M}  \}$
$\mathbf{1}_{\mathcal{M}}$ is a vector with only those elements indexed by $\mathcal{M}$ being one and others being zero. 
%Notice that the attacker knows the nature of each PDF (i.e. whether it is malicious).  
Our experiments were conducted on a subset (3000 malicious PDF and 3000 benign PDF) randomly sampled from the original dataset. We evenly divided the subset into training and testing sets. 
%The original feature dimensions are 135. 
We applied PCA to reduce dimensionality of the data and selected the top-10 principal components as features.  The result for \textit{best case} is displayed in Figure~\ref{fig:PDF_best_case}. Notice that when $\beta=0$, \textit{MLSG} is less robust than \textit{Lasso}, as we would expect. 
%This can be explained by the same argument for Figure~\ref{fig:boston_overestimate} as we made in previous paragraph.

Similarly as before we relaxed the assumption that the defender knows $\lambda$, $\beta$ and $\mathbf{z}$ and let the defender's estimation of the true $\lambda$ and $\beta$ be $\hat{\lambda}=1.5$ and $\hat{\beta}=0.5$. We also experimented with both overestimation and underestimation of $\mathbf{z}$. The defender's estimation is $\hat{\mathbf{z}}=-t \mathbf{1}_{\mathcal{M}}$. For overestimation setting $t$ is sampled from $[2\sigma_p, 3\sigma_p]$, and for underestimation setting it is sampled from $[\sigma_p, 2\sigma_p]$. 
The result for underestimated $\hat{\mathbf{z}}$ is showed in Figure~\ref{fig:PDF_underestimate}. Notice that in the leftmost plot of Figure~\ref{fig:PDF_underestimate} the area inside the blue rectangle corresponds to those cases where $\hat{\lambda}$ and $\hat{\beta}$ are overestimated and they are more robust than the remaining underestimated cases. Similar patterns can be observed in Figure~\ref{fig:redwine_underestimate}.
% Figure~\ref{fig:boston_overestimate}. 
This further supports our claim that it is advantageous to overestimate adversarial behavior.

\begin{figure}[h]
\centering
\begin{tabular}{c}
\includegraphics[width=2.2in]{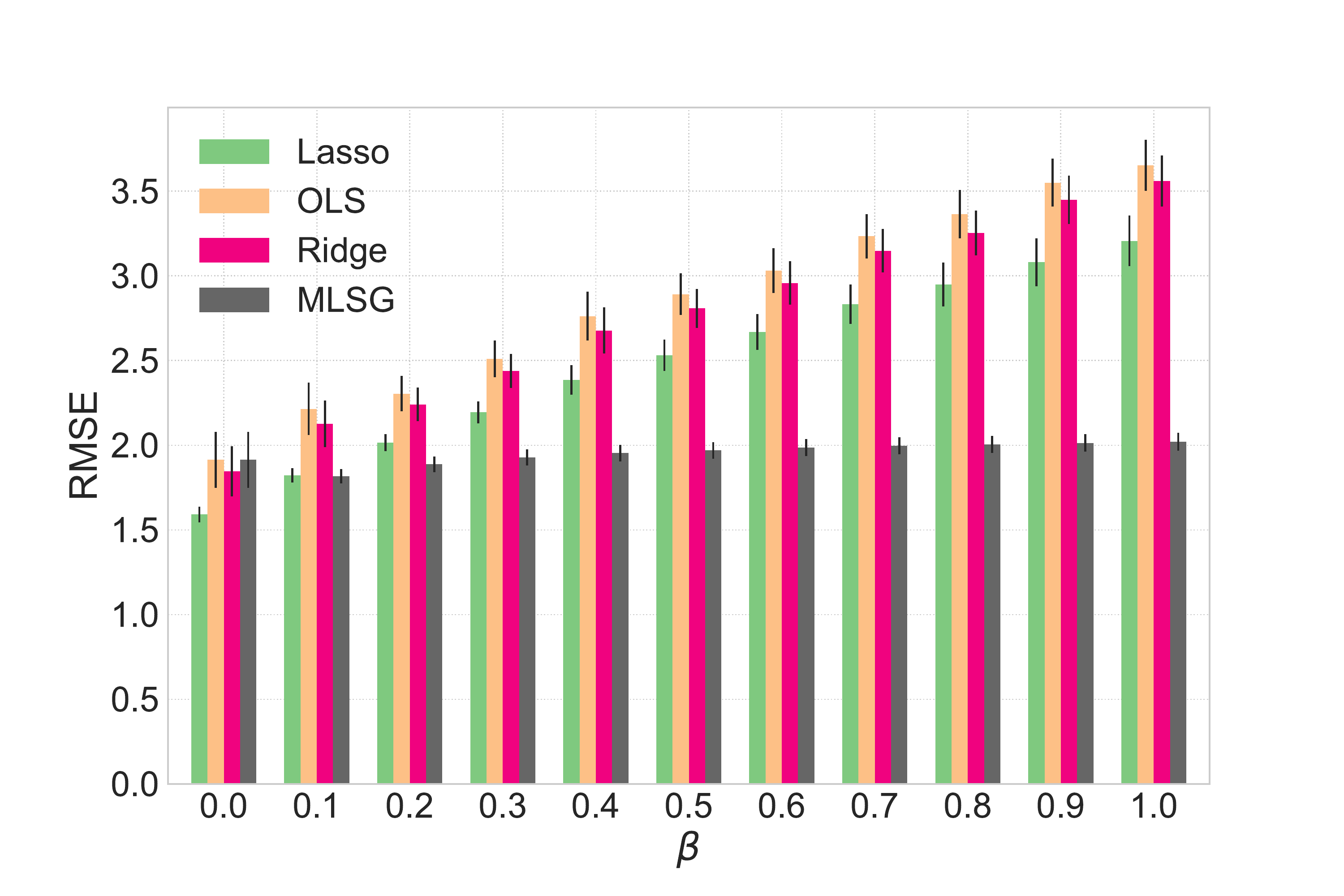}
\end{tabular}
\caption{RMSE of $\mathbf{y}^{'}$ and $\mathbf{y}$ on PDF dataset. The defender knows $\lambda$, $\beta$, and $\mathbf{z}$.}
\label{fig:PDF_best_case}
\end{figure}

\begin{figure}[h]
% \centering
	\begin{tabular}{cc}
		\includegraphics[width=1.6in]{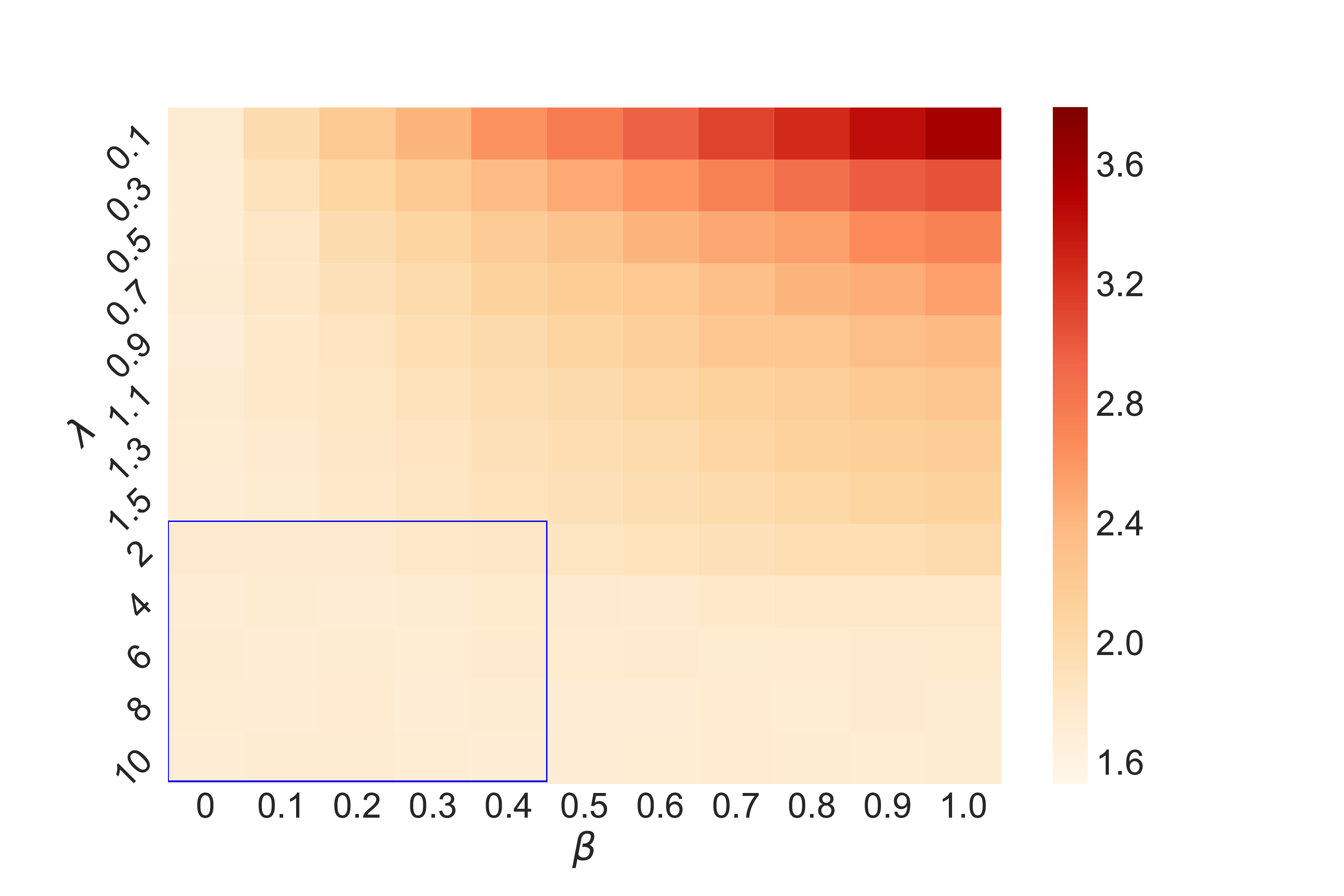} & \includegraphics[width=1.6in]{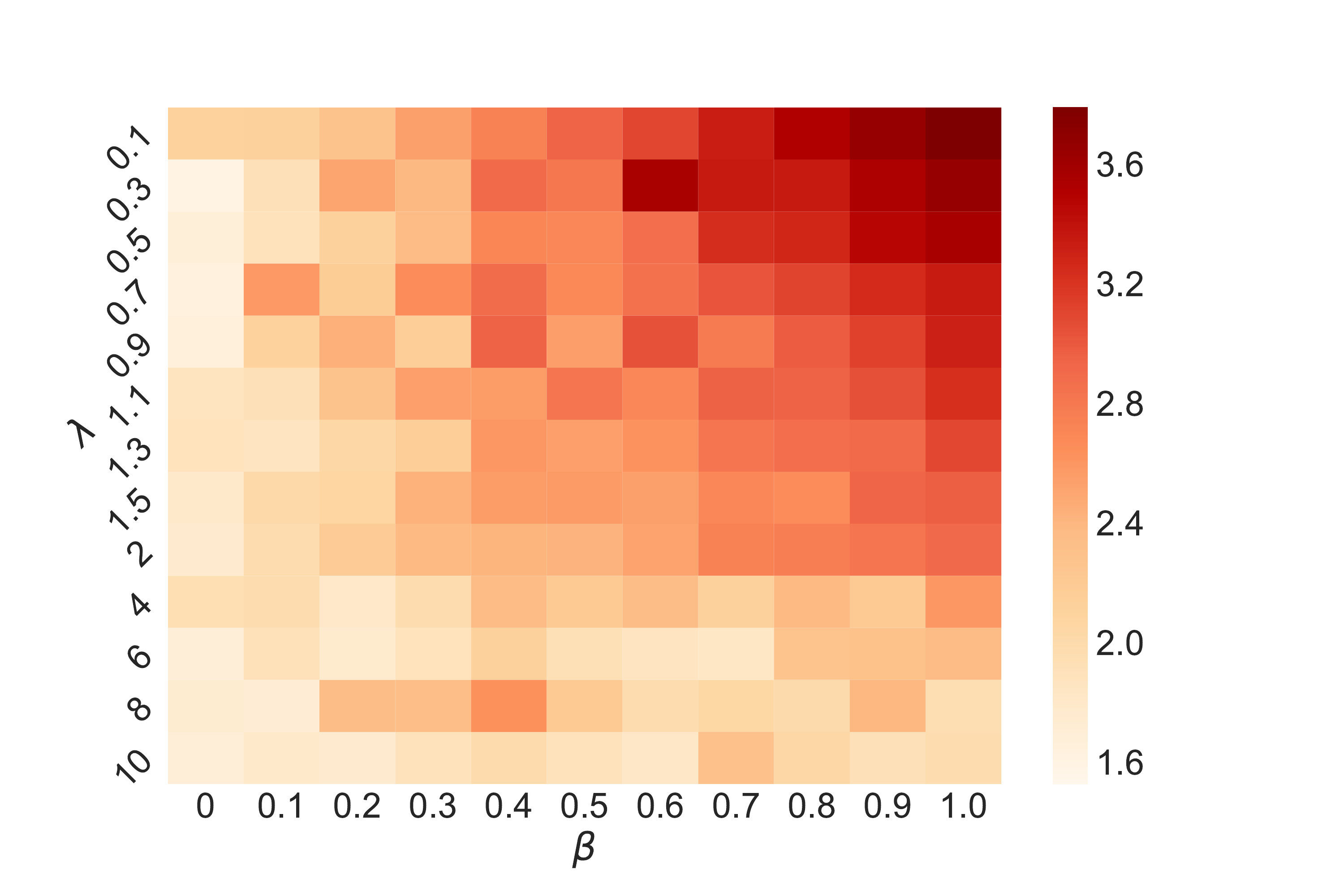}
	\end{tabular}
		\begin{tabular}{cc}
		\includegraphics[width=1.6in]{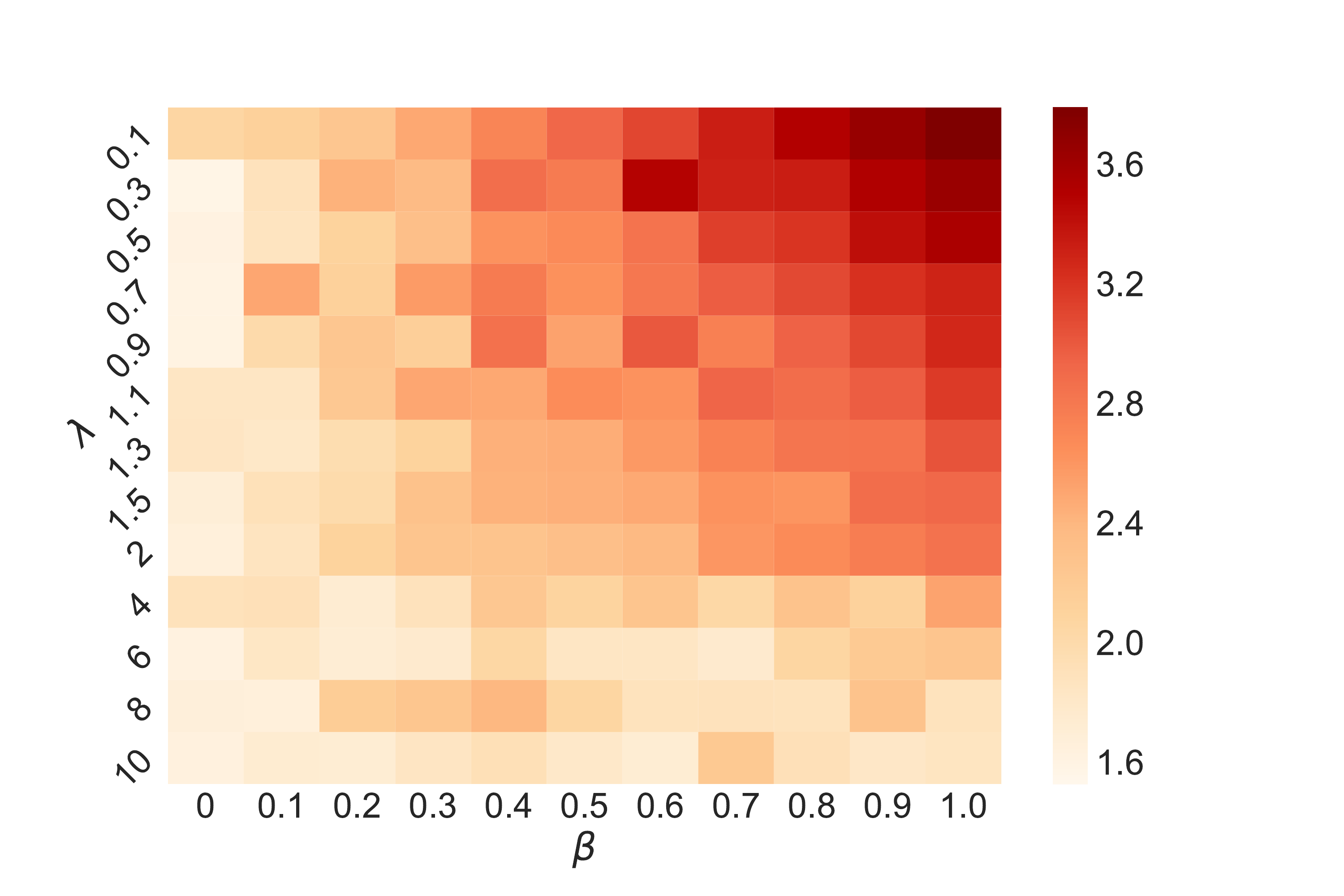} & \includegraphics[width=1.6in]{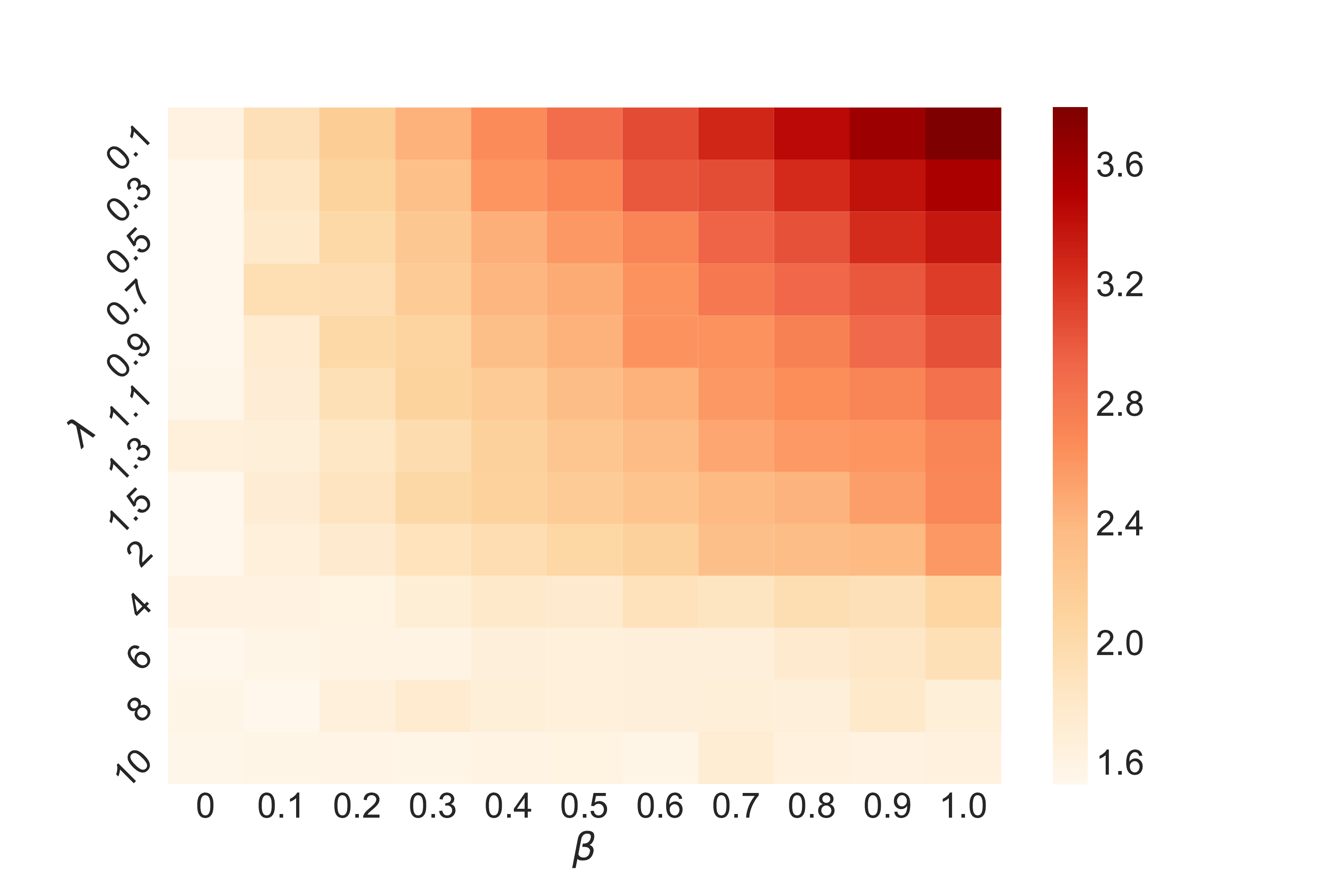}
	\end{tabular}
	\caption{The average RMSE across different values of actual $\lambda$ and $\beta$ on PDF dataset. From left to right: \textit{MLSG}, \textit{Lasso}, \textit{Ridge}, \textit{OLS}.}
	\label{fig:PDF_underestimate}
\end{figure}

\section{Conclusion}
\label{sec:conclusion}

We study the problem of linear regression in adversarial settings
involving multiple learners learning from the same or similar data.
In our model, learners first simultaneously decide on their models
(i.e., learn), and an attacker then modifies test instances to cause
predictions to err towards the attacker's target.
We first derive an upper bound on the cost functions of all learners,
and the resulting approximate game.
We then show that this game has a unique symmetric equilibrium, and
present an approach for computing this equilibrium by solving a convex
optimization problem.
Finally, we show that the equilibrium is robust, both theoretically,
and through an extensive experimental evaluation.

\section*{Acknowledgements}
We would like to thank Shiying Li in the Department of Mathematics at Vanderbilt University for her valuable and constructive suggestions for the proof of Theorem 4.
This work was partially supported by the National  Science Foundation (CNS-1640624, IIS-1526860, IIS-1649972), Office of  Naval Research (N00014-15-1-2621),  Army  Research  Office (W911NF-16-1-0069), and National Institutes of Health (R01HG006844-05).

\bibliography{ICML2018}

\begin{thebibliography}{21}
\providecommand{\natexlab}[1]{#1}
\providecommand{\url}[1]{\texttt{#1}}
\expandafter\ifx\csname urlstyle\endcsname\relax
  \providecommand{\doi}[1]{doi: #1}\else
  \providecommand{\doi}{doi: \begingroup \urlstyle{rm}\Url}\fi

\bibitem[Alfeld et~al.(2016)Alfeld, Zhu, and Barford]{Alfeld16}
Alfeld, S., Zhu, X., and Barford, P.
\newblock Data poisoning attacks against autoregressive models.
\newblock In \emph{AAAI Conference on Artificial Intelligence}, 2016.

\bibitem[Br{\"u}ckner \& Scheffer(2011)Br{\"u}ckner and Scheffer]{Bruckner11}
Br{\"u}ckner, M. and Scheffer, T.
\newblock Stackelberg games for adversarial prediction problems.
\newblock In \emph{ACM SIGKDD international conference on Knowledge discovery
  and data mining}, pp.\  547--555, 2011.

\bibitem[Cheng et~al.(2004)Cheng, Reeves, Vorobeychik, and Wellman]{Cheng04}
Cheng, S.-F., Reeves, D.~M., Vorobeychik, Y., and Wellman, M.~P.
\newblock Notes on equilibria in symmetric games.
\newblock In \emph{International Workshop On Game Theoretic And Decision
  Theoretic Agents}, pp.\  71--78, 2004.

\bibitem[Cortez et~al.(2009)Cortez, Cerdeira, Almeida, Matos, and
  Reis]{cortez2009modeling}
Cortez, P., Cerdeira, A., Almeida, F., Matos, T., and Reis, J.
\newblock Modeling wine preferences by data mining from physicochemical
  properties.
\newblock \emph{Decision Support Systems}, 47\penalty0 (4):\penalty0 547--553,
  2009.

\bibitem[Dalvi et~al.(2004)Dalvi, Domingos, Mausam, Sanghai, and
  Verma]{Dalvi04}
Dalvi, N., Domingos, P., Mausam, Sanghai, S., and Verma, D.
\newblock Adversarial classification.
\newblock In \emph{SIGKDD International Conference on Knowledge Discovery and
  Data Mining}, pp.\  99--108, 2004.

\bibitem[Grosshans et~al.(2013)Grosshans, Sawade, Br\"{u}ckner, and
  Scheffer]{Grosshans13}
Grosshans, M., Sawade, C., Br\"{u}ckner, M., and Scheffer, T.
\newblock Bayesian games for adversarial regression problems.
\newblock In \emph{International Conference on International Conference on
  Machine Learning}, pp.\  55--63, 2013.

\bibitem[Harrison~Jr \& Rubinfeld(1978)Harrison~Jr and
  Rubinfeld]{harrison1978hedonic}
Harrison~Jr, D. and Rubinfeld, D.~L.
\newblock Hedonic housing prices and the demand for clean air.
\newblock \emph{Journal of environmental economics and management}, 5\penalty0
  (1):\penalty0 81--102, 1978.

\bibitem[Higham(1990)]{Higham1990}
Higham, N.~J.
\newblock Analysis of the cholesky decomposition of a semi-definite matrix.
\newblock In \emph{Reliable Numerical Computation}, pp.\  161--185. University
  Press, 1990.

\bibitem[Laszka et~al.(2016)Laszka, Lou, and Vorobeychik]{Laszka16}
Laszka, A., Lou, J., and Vorobeychik, Y.
\newblock Multi-defender strategic filtering against spear-phishing attacks.
\newblock In \emph{AAAI Conference on Artificial Intelligence}, 2016.

\bibitem[Li \& Vorobeychik(2014)Li and Vorobeychik]{li2014feature}
Li, B. and Vorobeychik, Y.
\newblock Feature cross-substitution in adversarial classification.
\newblock In \emph{Advances in neural information processing systems}, pp.\
  2087--2095, 2014.

\bibitem[Li \& Vorobeychik(2015)Li and Vorobeychik]{Li15b}
Li, B. and Vorobeychik, Y.
\newblock Scalable optimization of randomized operational decisions in
  adversarial classification settings.
\newblock In \emph{Conference on Artificial Intelligence and Statistics}, 2015.

\bibitem[Lowd \& Meek(2005)Lowd and Meek]{lowd2005adversarial}
Lowd, D. and Meek, C.
\newblock Adversarial learning.
\newblock In \emph{Proceedings of the eleventh ACM SIGKDD international
  conference on Knowledge discovery in data mining}, pp.\  641--647. ACM, 2005.

\bibitem[Rosen(1965)]{Rosen65}
Rosen, J.~B.
\newblock Existence and uniqueness of equilibrium points for concave n-person
  games.
\newblock \emph{Econometrica: Journal of the Econometric Society}, pp.\
  520--534, 1965.

\bibitem[Russu et~al.(2016)Russu, Demontis, Biggio, Fumera, and
  Roli]{aisec2016}
Russu, P., Demontis, A., Biggio, B., Fumera, G., and Roli, F.
\newblock Secure kernel machines against evasion attacks.
\newblock In \emph{{ACM} Workshop on Artificial Intelligence and Security},
  pp.\  59--69, 2016.

\bibitem[Smith et~al.(2017)Smith, Lou, and Vorobeychik]{Smith17}
Smith, A., Lou, J., and Vorobeychik, Y.
\newblock Multidefender security games.
\newblock \emph{IEEE Intelligent Systems}, 32\penalty0 (1):\penalty0 50--60,
  2017.

\bibitem[{\v{S}}rndic \& Laskov(2014){\v{S}}rndic and
  Laskov]{laskov2014practical}
{\v{S}}rndic, N. and Laskov, P.
\newblock Practical evasion of a learning-based classifier: A case study.
\newblock In \emph{2014 IEEE Symposium on Security and Privacy}, pp.\
  197--211, 2014.

\bibitem[Stevens \& Lowd(2013)Stevens and Lowd]{Stevens13}
Stevens, D. and Lowd, D.
\newblock On the hardness of evading combinations of linear classifiers.
\newblock In \emph{ACM Workshop on Artificial Intelligence and Security}, 2013.

\bibitem[Vorobeychik et~al.(2011)Vorobeychik, Mayo, Armstrong, and
  Ruthruff]{Vorobeychik11}
Vorobeychik, Y., Mayo, J.~R., Armstrong, R.~C., and Ruthruff, J.
\newblock Noncooperatively optimized tolerance: Decentralized strategic
  optimization in complex systems.
\newblock \emph{Physical Review Letters}, 107\penalty0 (10):\penalty0 108702,
  2011.

\bibitem[Xu et~al.(2009)Xu, Caramanis, and Mannor]{xu2009robust}
Xu, H., Caramanis, C., and Mannor, S.
\newblock Robust regression and lasso.
\newblock In \emph{Advances in Neural Information Processing Systems}, pp.\
  1801--1808, 2009.

\bibitem[Yang \& Xu(2013)Yang and Xu]{yang2013unified}
Yang, W. and Xu, H.
\newblock A unified robust regression model for lasso-like algorithms.
\newblock In \emph{International Conference on Machine Learning}, pp.\
  585--593, 2013.

\bibitem[Zhou et~al.(2012)Zhou, Kantarcioglu, Thuraisingham, and Xi]{kdd2012}
Zhou, Y., Kantarcioglu, M., Thuraisingham, B.~M., and Xi, B.
\newblock Adversarial support vector machine learning.
\newblock In \emph{{ACM} {SIGKDD} International Conference on Knowledge
  Discovery and Data Mining}, pp.\  1059--1067, 2012.

\end{thebibliography}
\bibliographystyle{icml2018}

\section*{Appendix}
\subsection{Supplementary results for the redwine dataset}

\begin{figure}[H]
% \centering
	\begin{tabular}{cc}
		\includegraphics[width=1.5in]{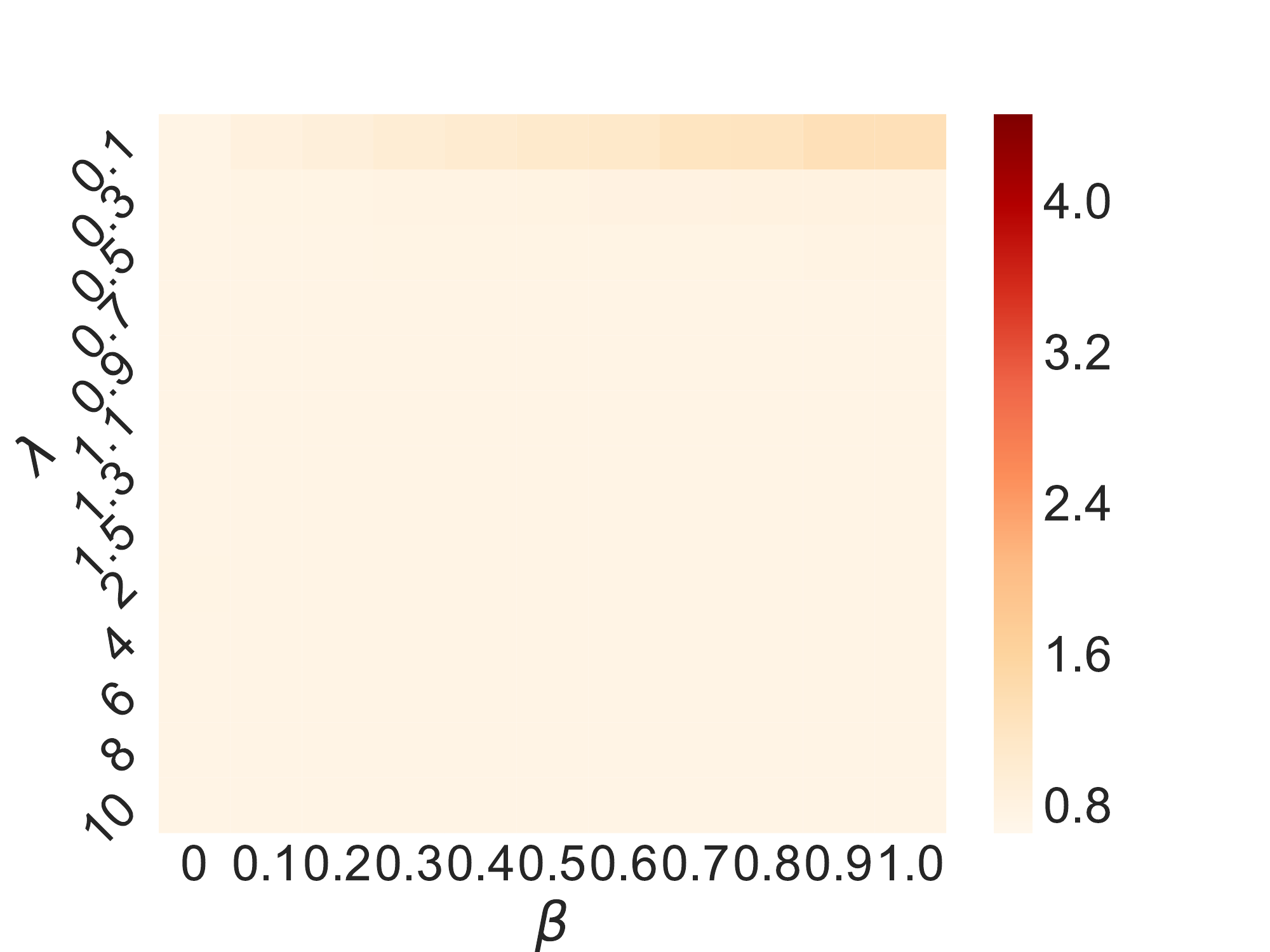} & \includegraphics[width=1.5in]{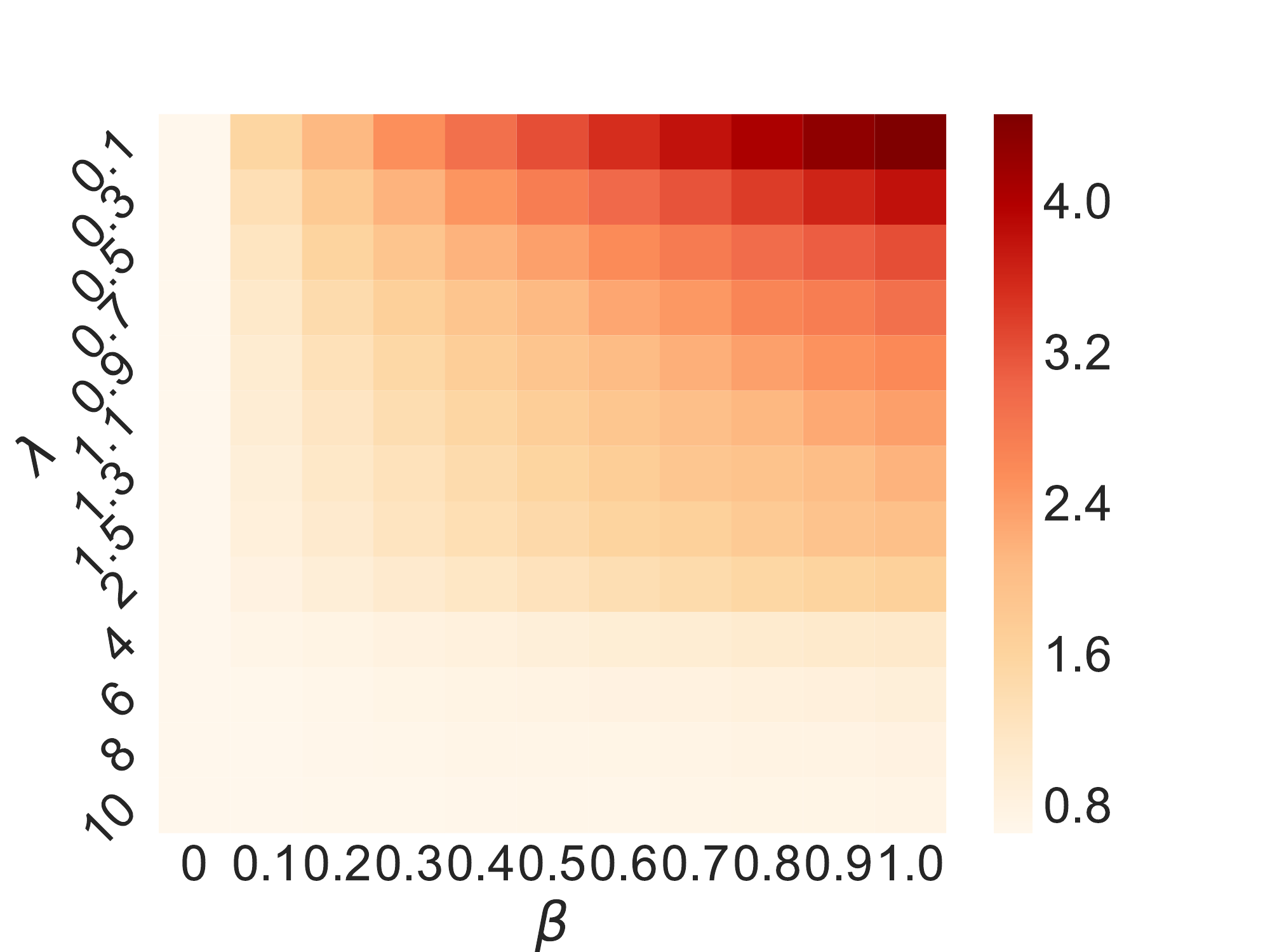}
	\end{tabular}
		\begin{tabular}{cc}
		\includegraphics[width=1.5in]{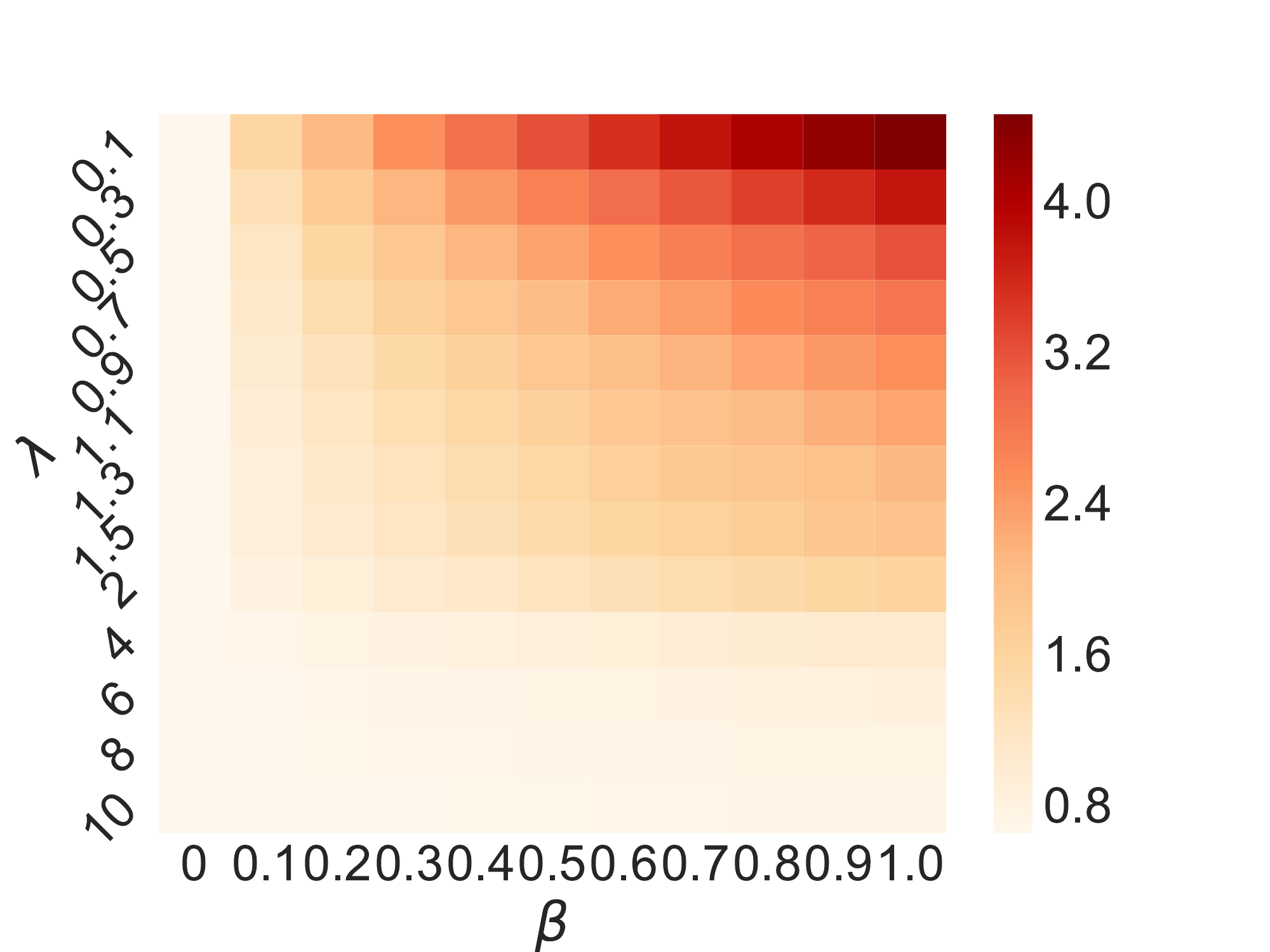} & \includegraphics[width=1.5in]{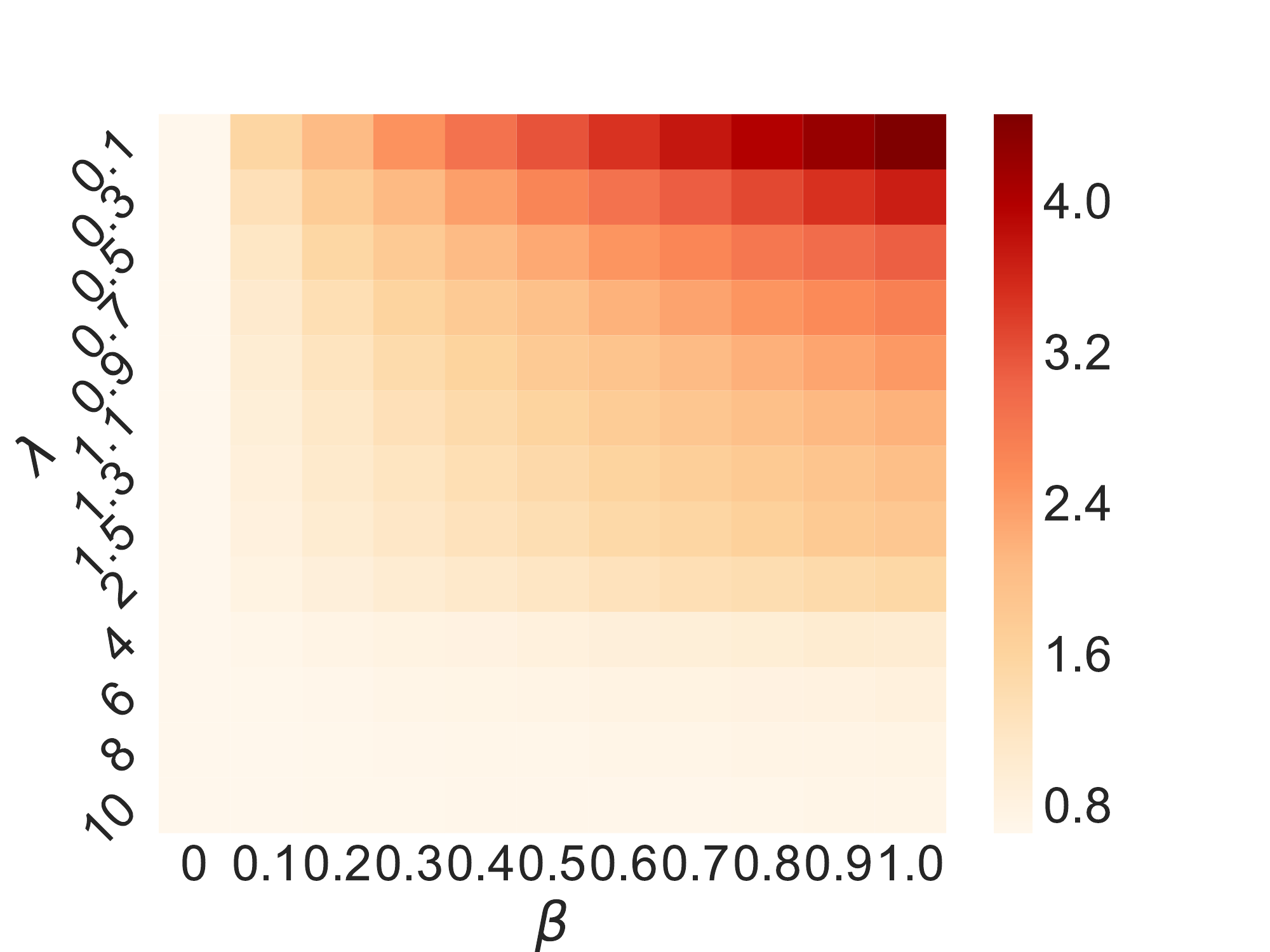}
	\end{tabular}
	\caption{Overestimated $\mathbf{z}$, $\hat{\lambda}=0.5$, $\hat{\beta}=0.8$.The average RMSE across different values of actual $\lambda$ and $\beta$ on redwine dataset. From left to right: \textit{MLSG}, \textit{Lasso}, \textit{Ridge}, \textit{OLS}. }
	\label{fig:redwine_overestimate}
\end{figure}

\begin{figure}[H]
% \centering
	\begin{tabular}{cc}
		\includegraphics[width=1.5in]{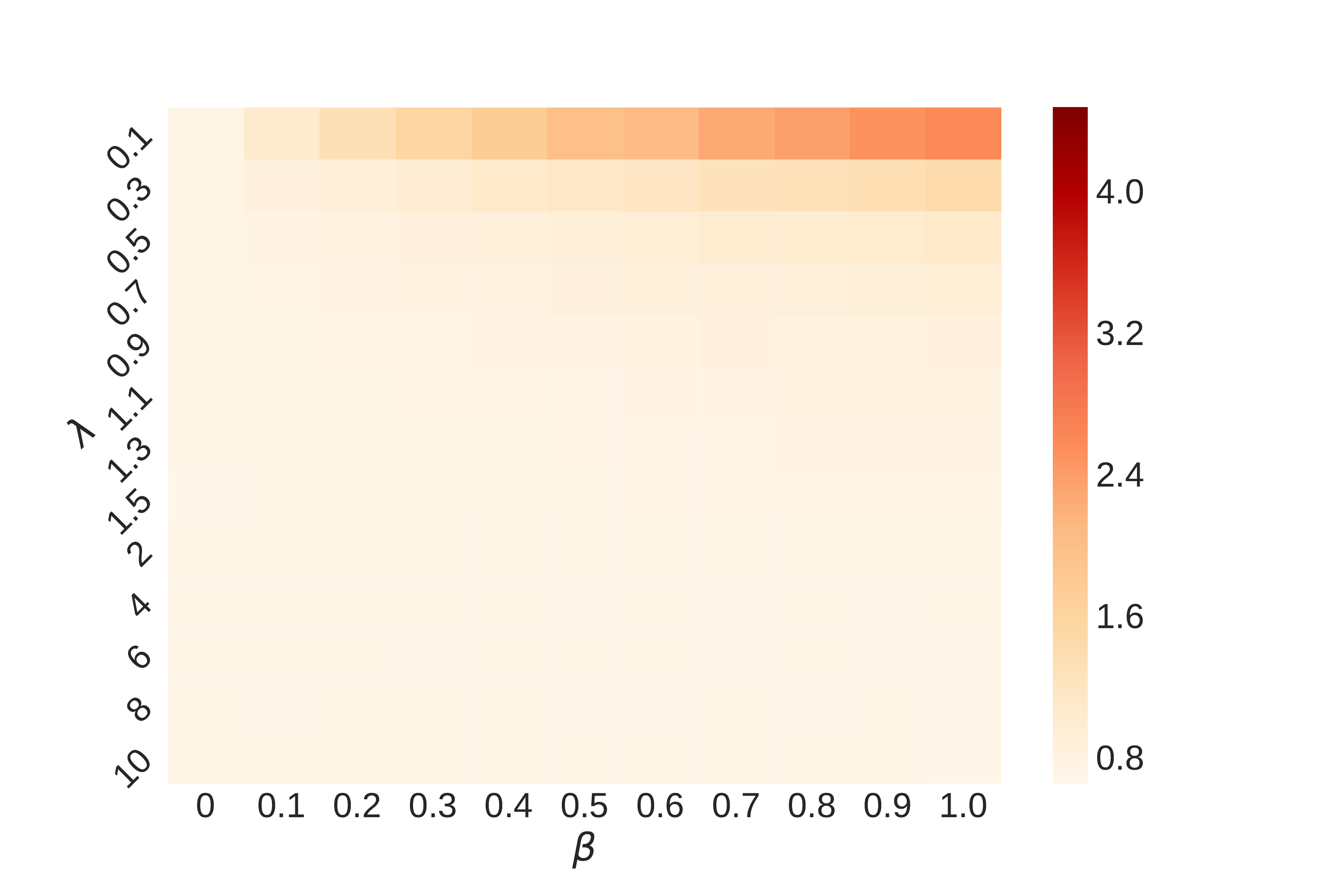} & \includegraphics[width=1.5in]{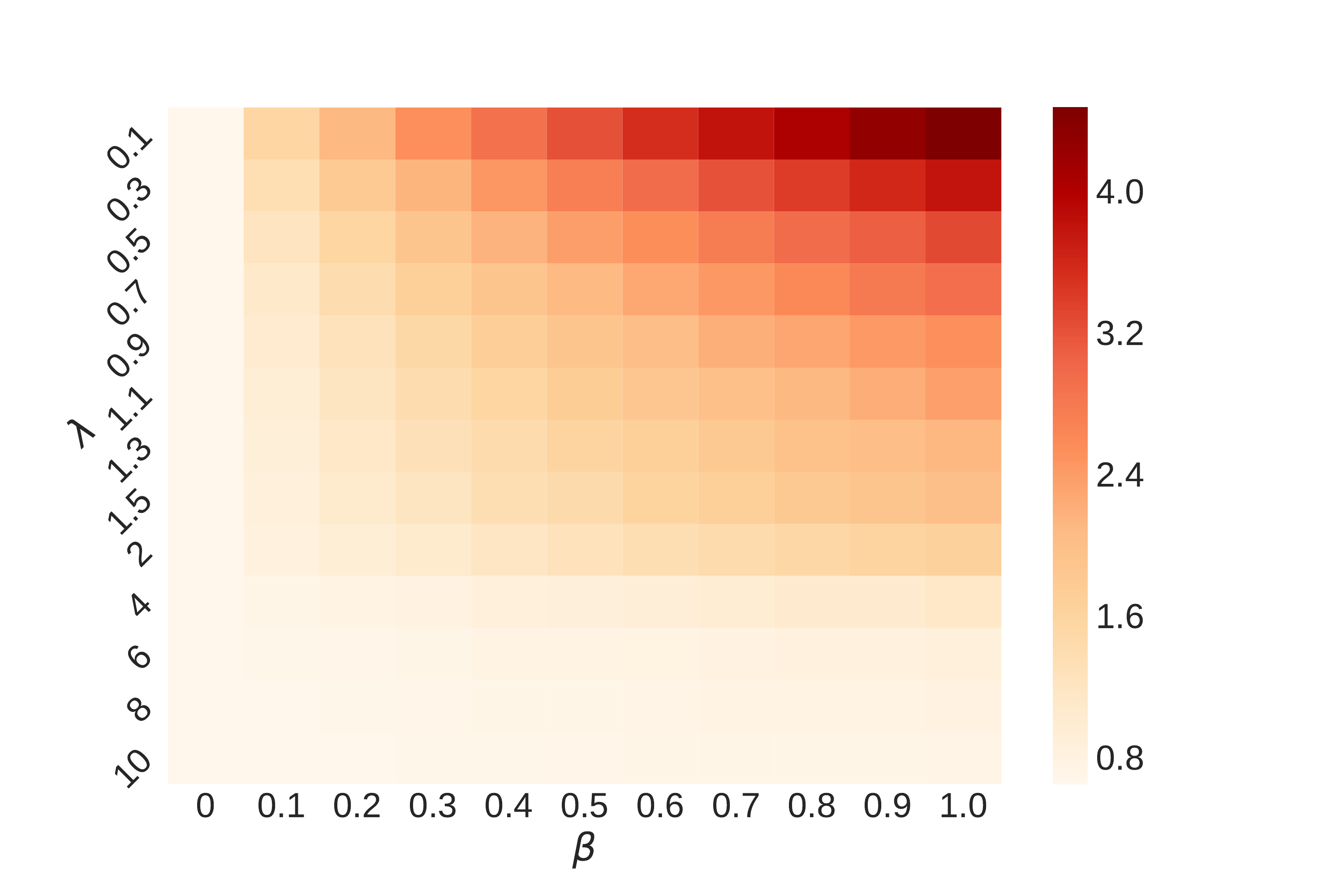}
	\end{tabular}
		\begin{tabular}{cc}
		\includegraphics[width=1.5in]{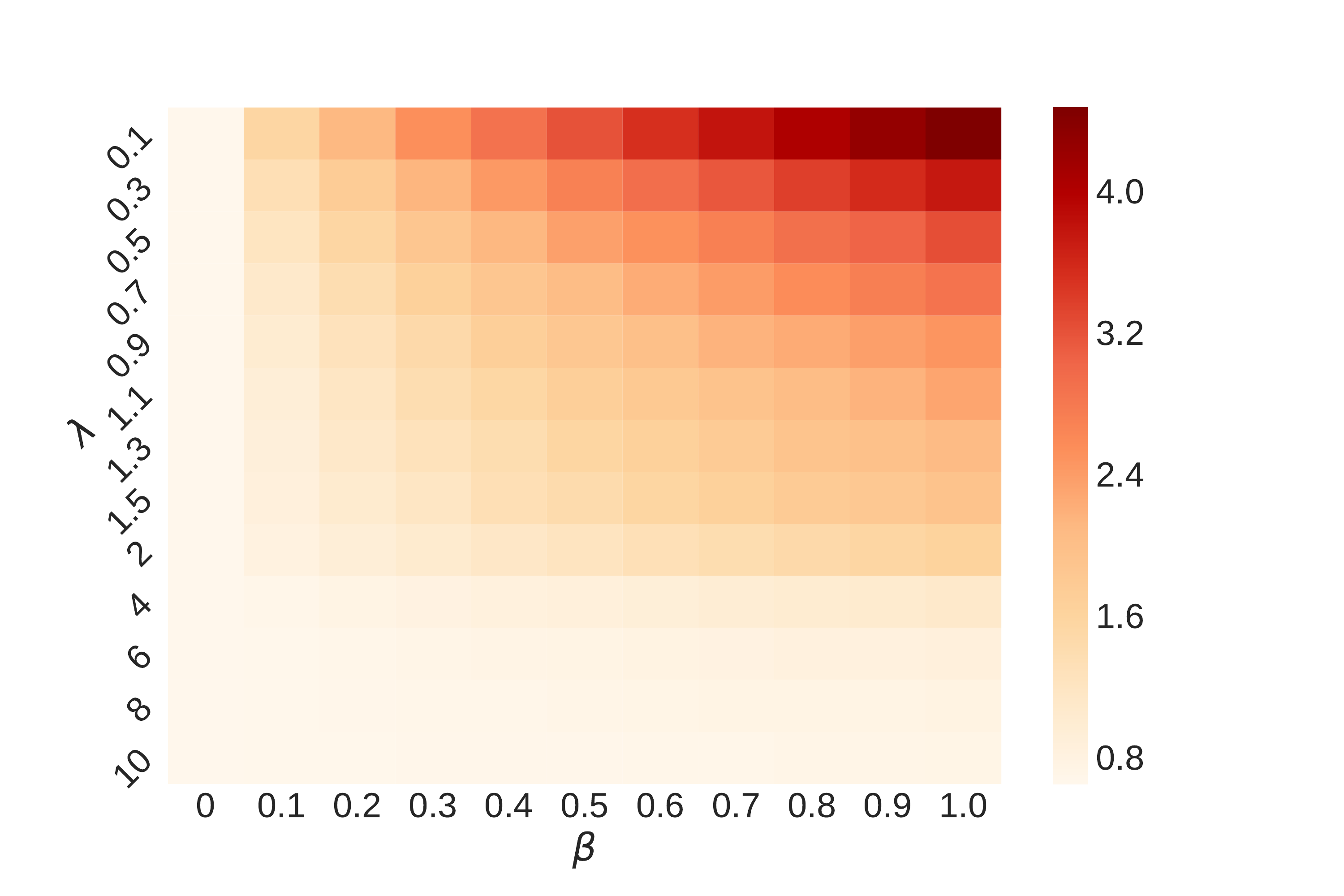} & \includegraphics[width=1.5in]{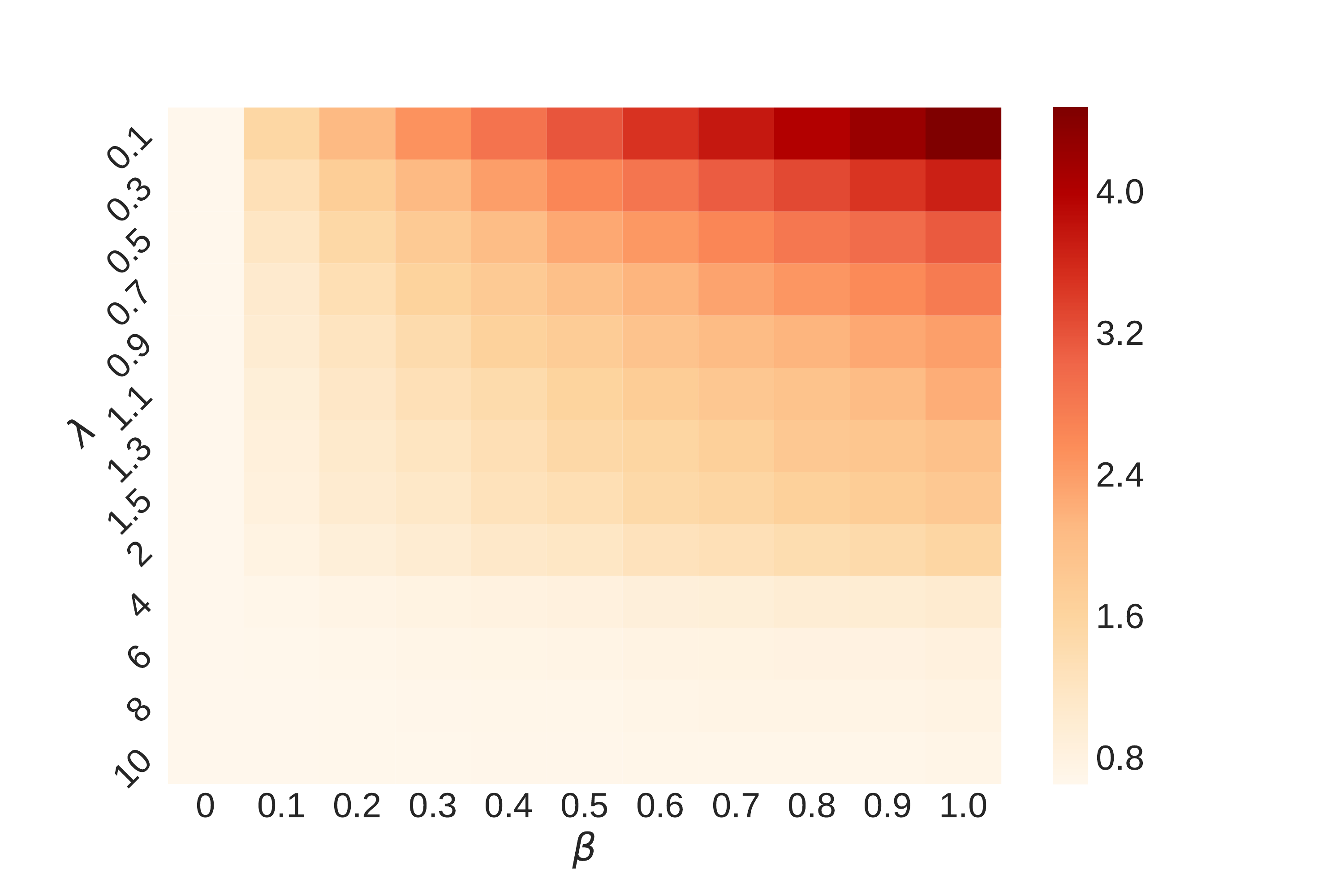}
	\end{tabular}
	\caption{Overestimated $\mathbf{z}$, $\hat{\lambda}=1.5$, $\hat{\beta}=0.8$. The average RMSE across different values of actual $\lambda$ and $\beta$ on redwine dataset. From left to right: \textit{MLSG}, \textit{Lasso}, \textit{Ridge}, \textit{OLS}. }
	\label{fig:redwine_overestimate}
\end{figure}

\begin{figure}[H]
% \centering
	\begin{tabular}{cc}
		\includegraphics[width=1.5in]{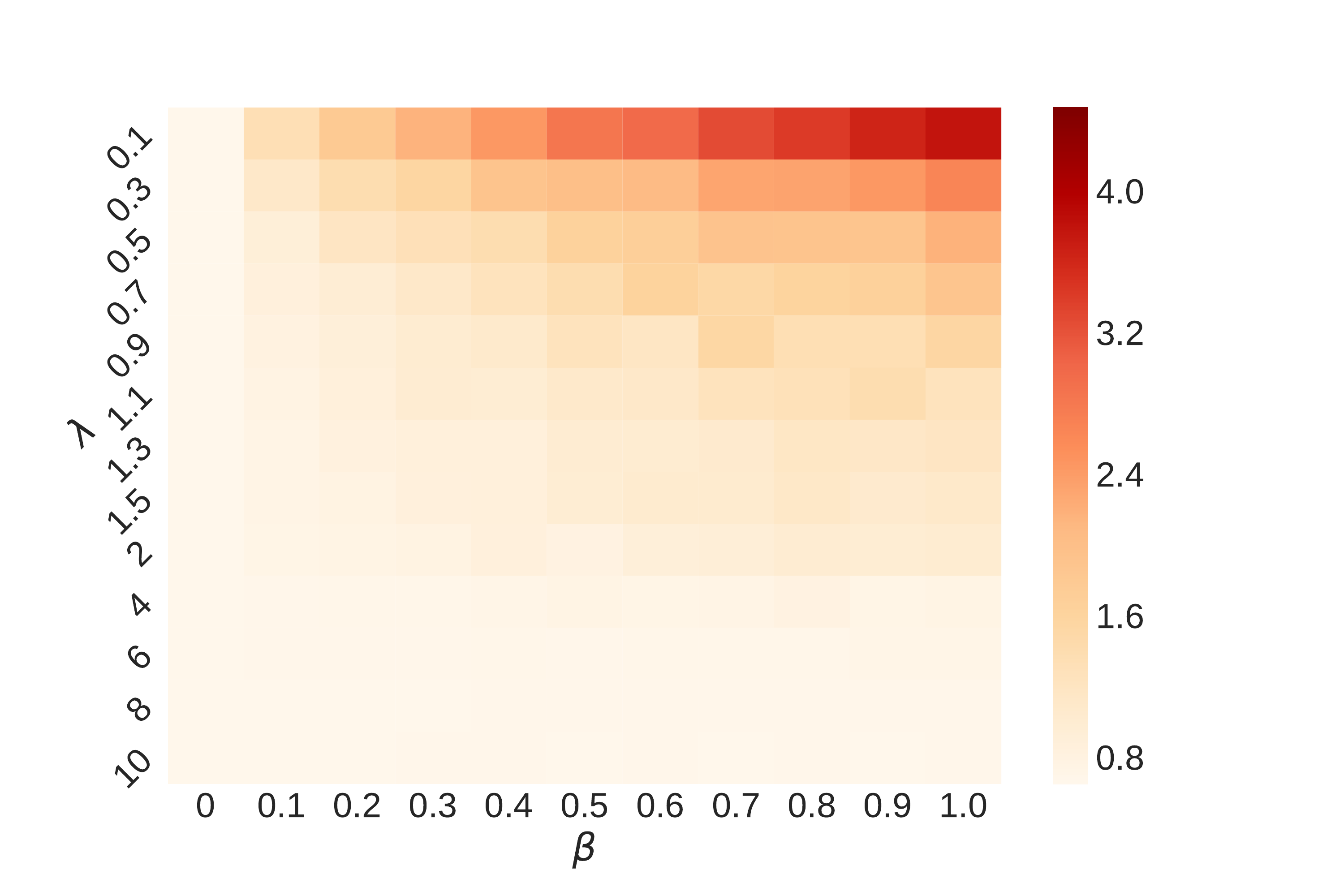} & \includegraphics[width=1.5in]{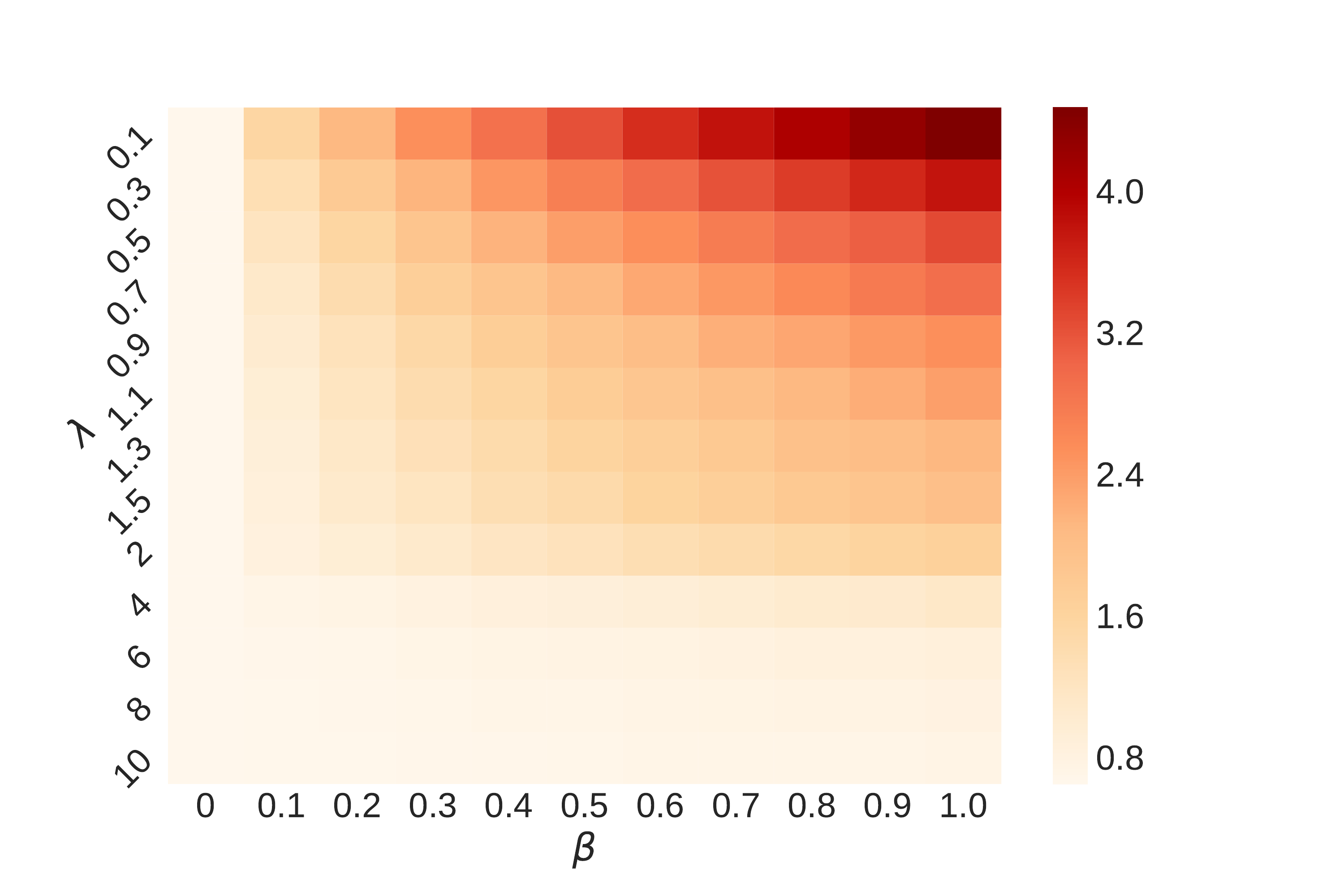}
	\end{tabular}
		\begin{tabular}{cc}
		\includegraphics[width=1.5in]{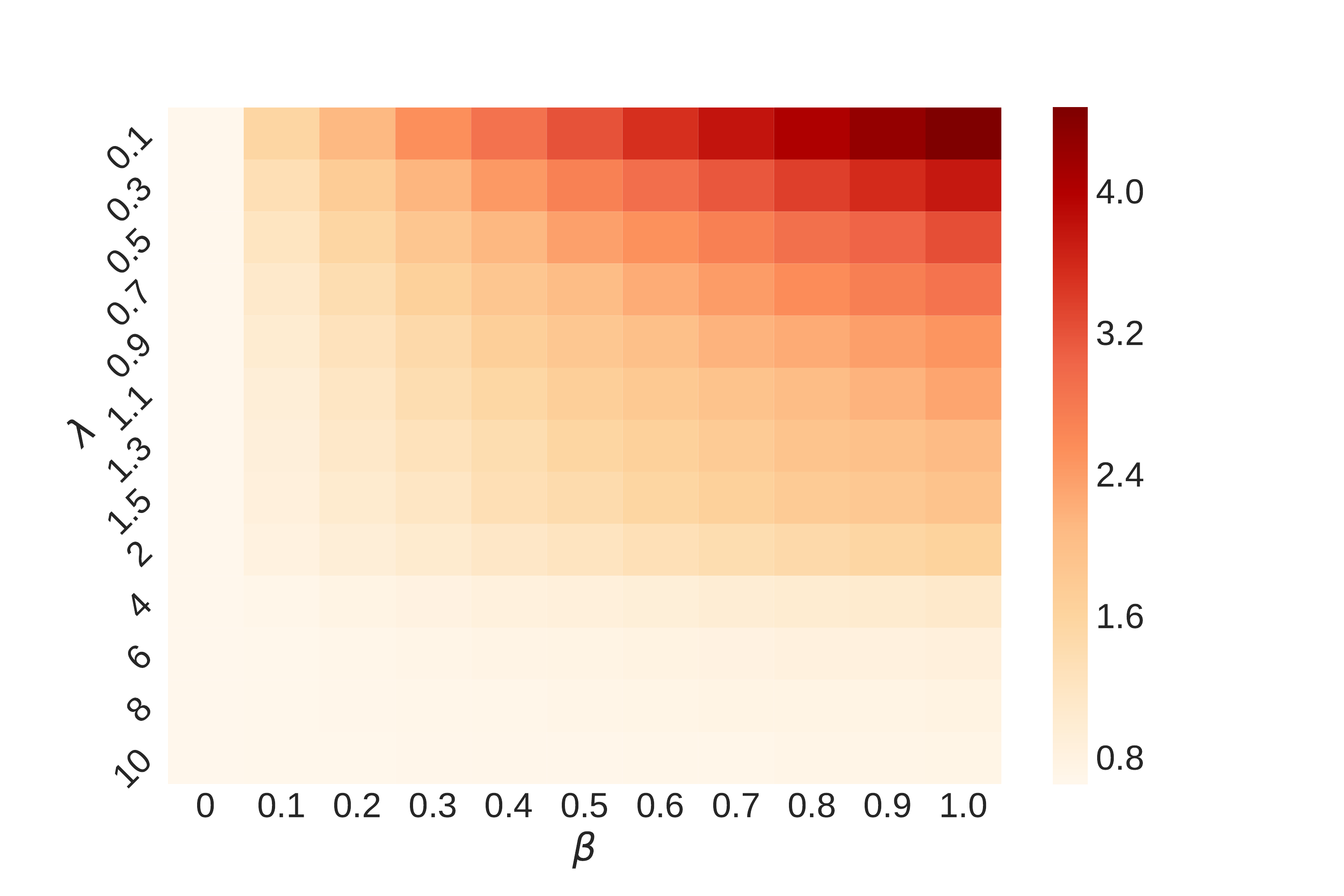} & \includegraphics[width=1.5in]{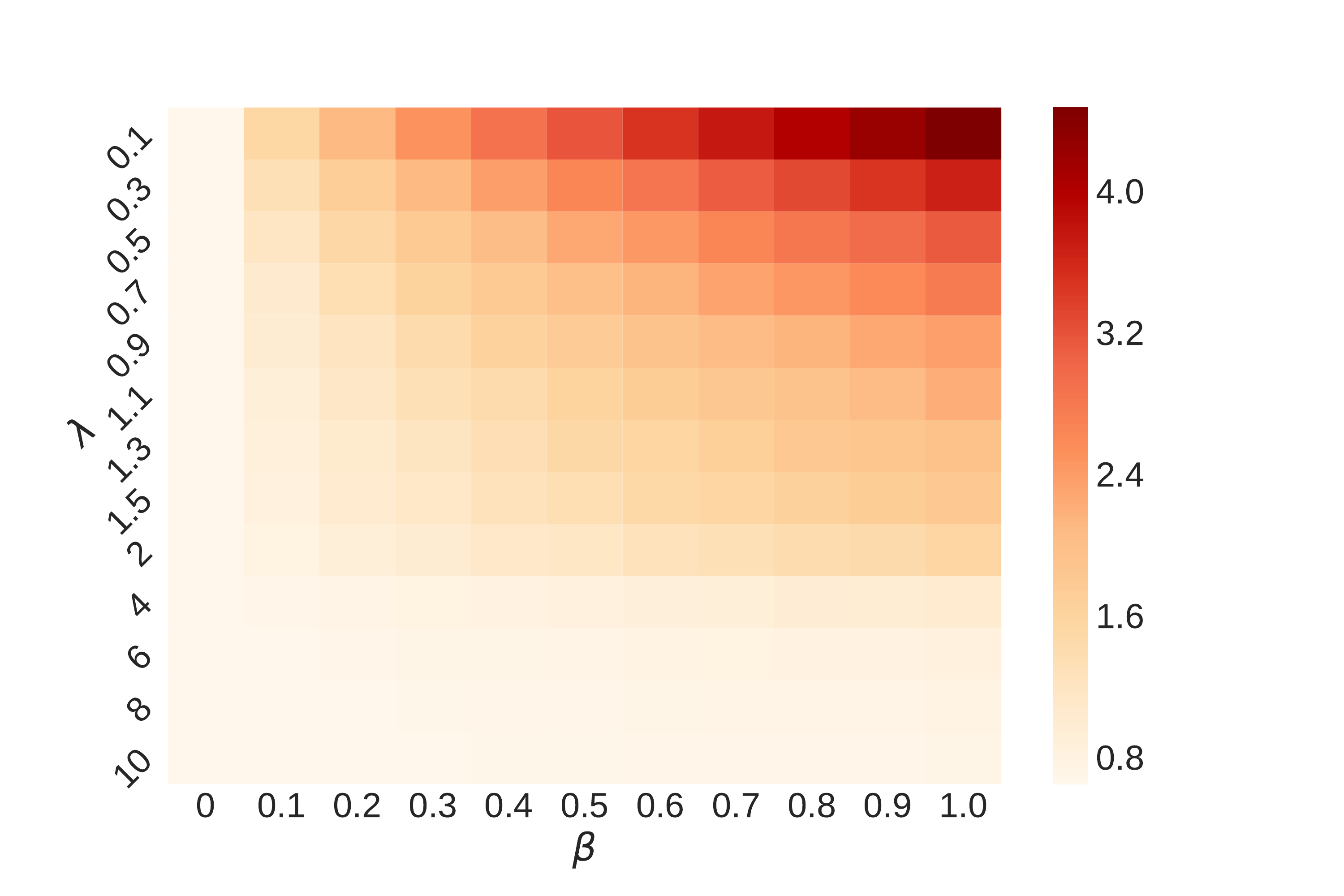}
	\end{tabular}
	\caption{Underestimated $\mathbf{z}$, $\hat{\lambda}=1.5$, $\hat{\beta}=0.8$. The average RMSE across different values of actual $\lambda$ and $\beta$ on redwine dataset. From left to right: \textit{MLSG}, \textit{Lasso}, \textit{Ridge}, \textit{OLS}. }
	\label{fig:redwine_underestimate}
\end{figure}

\subsection{Supplementary results for the boston dataset}
\begin{figure}[H]
% \centering
	\begin{tabular}{cc}
		\includegraphics[width=1.6in]{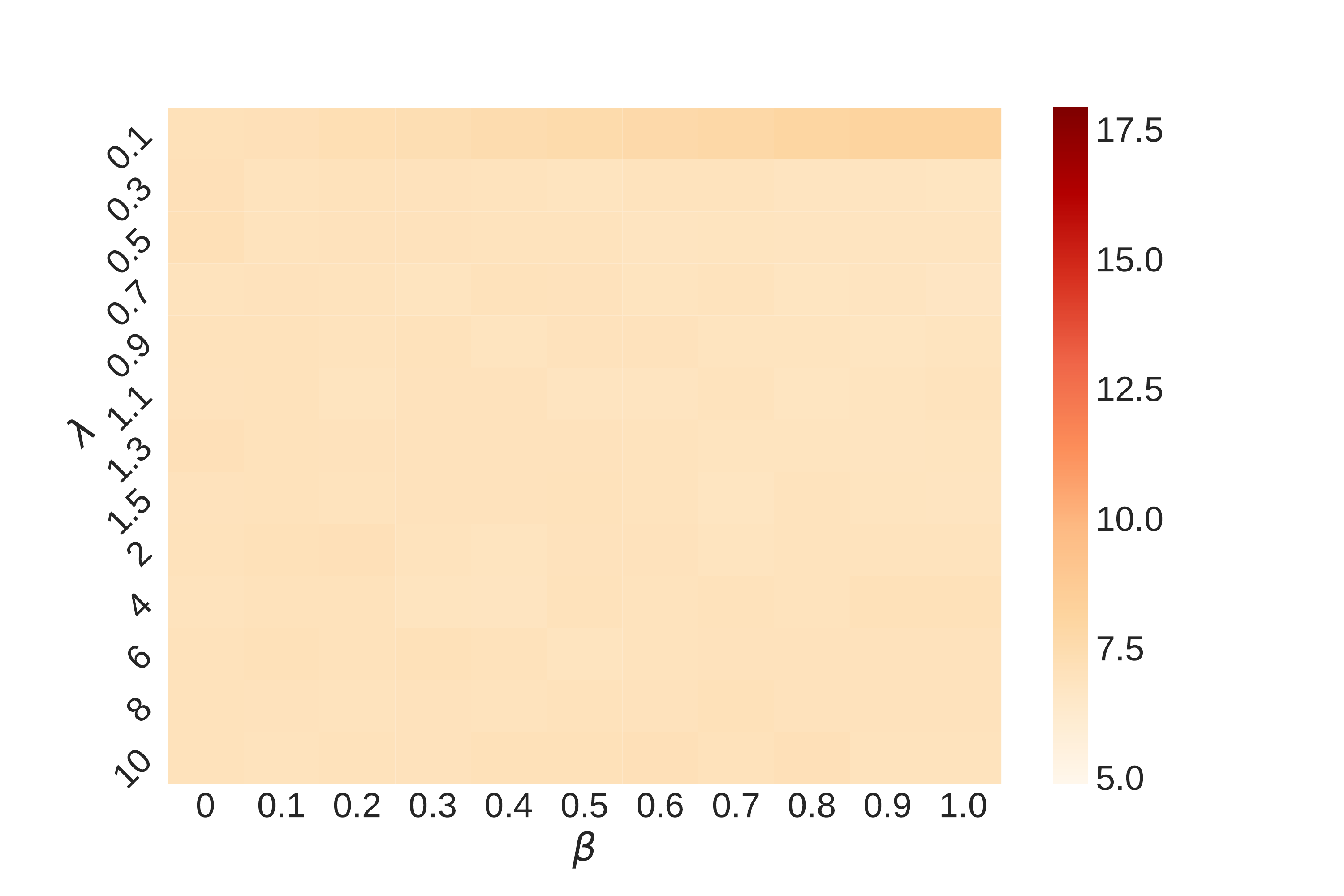} & \includegraphics[width=1.6in]{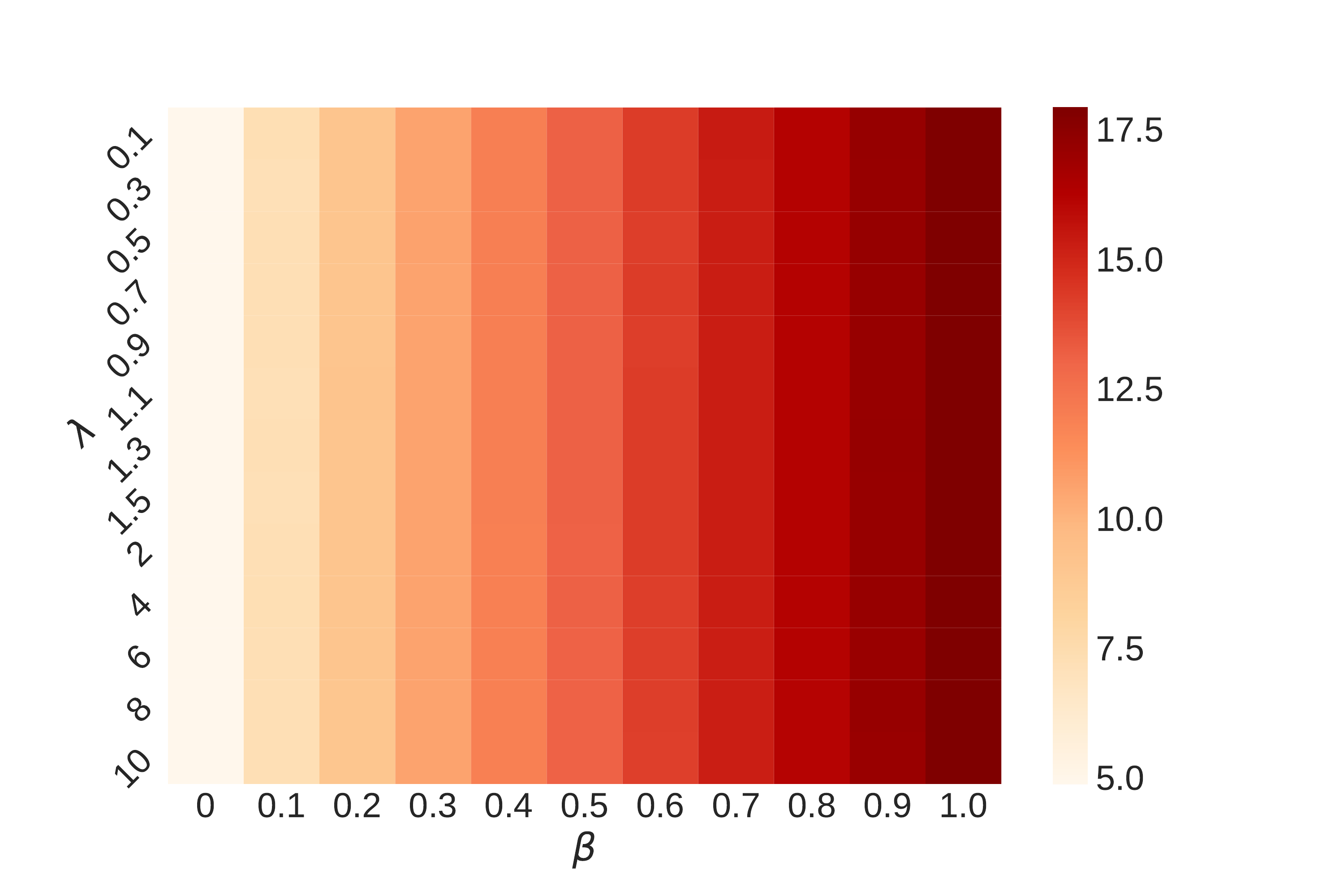}
	\end{tabular}
		\begin{tabular}{cc}
		\includegraphics[width=1.6in]{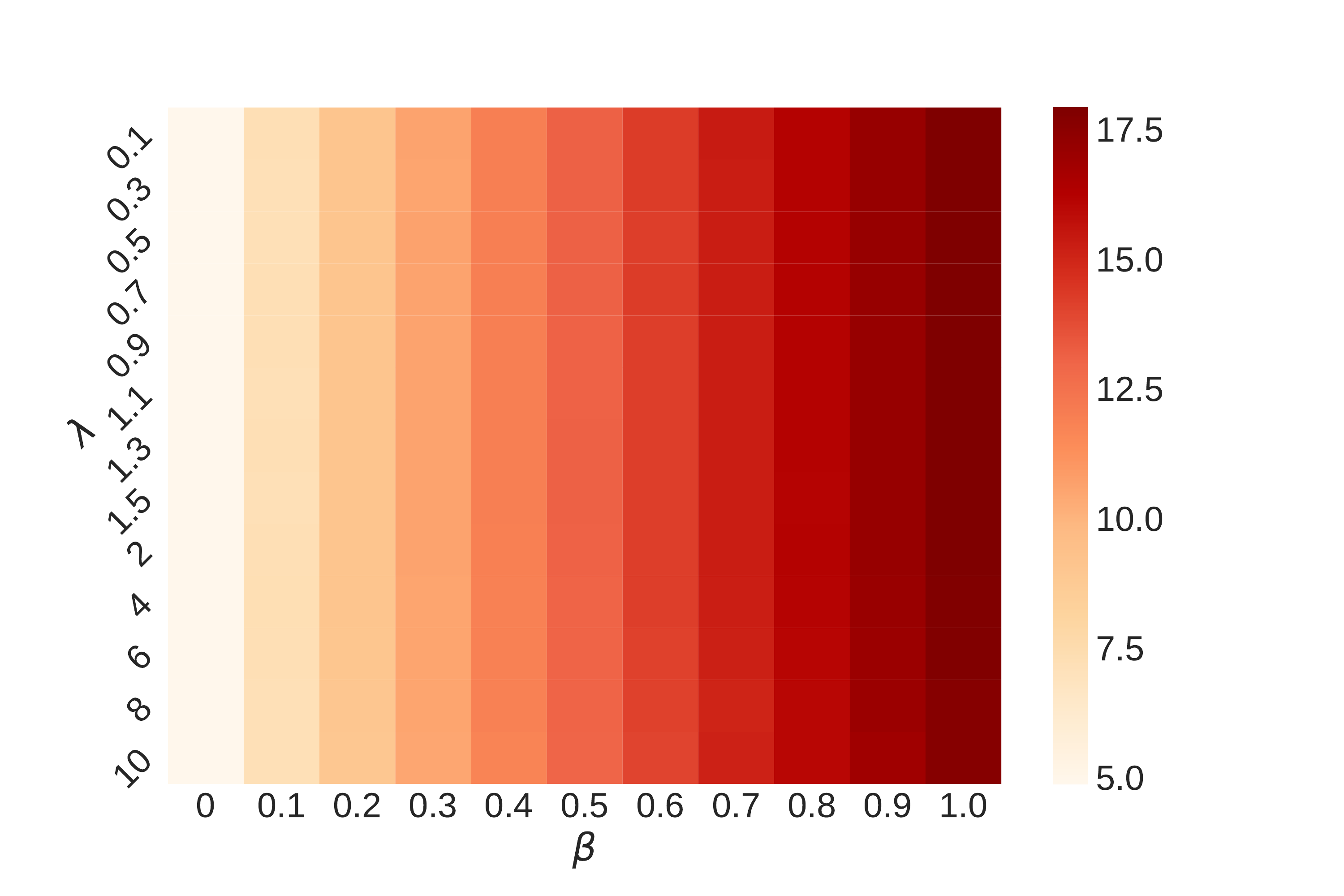} & \includegraphics[width=1.6in]{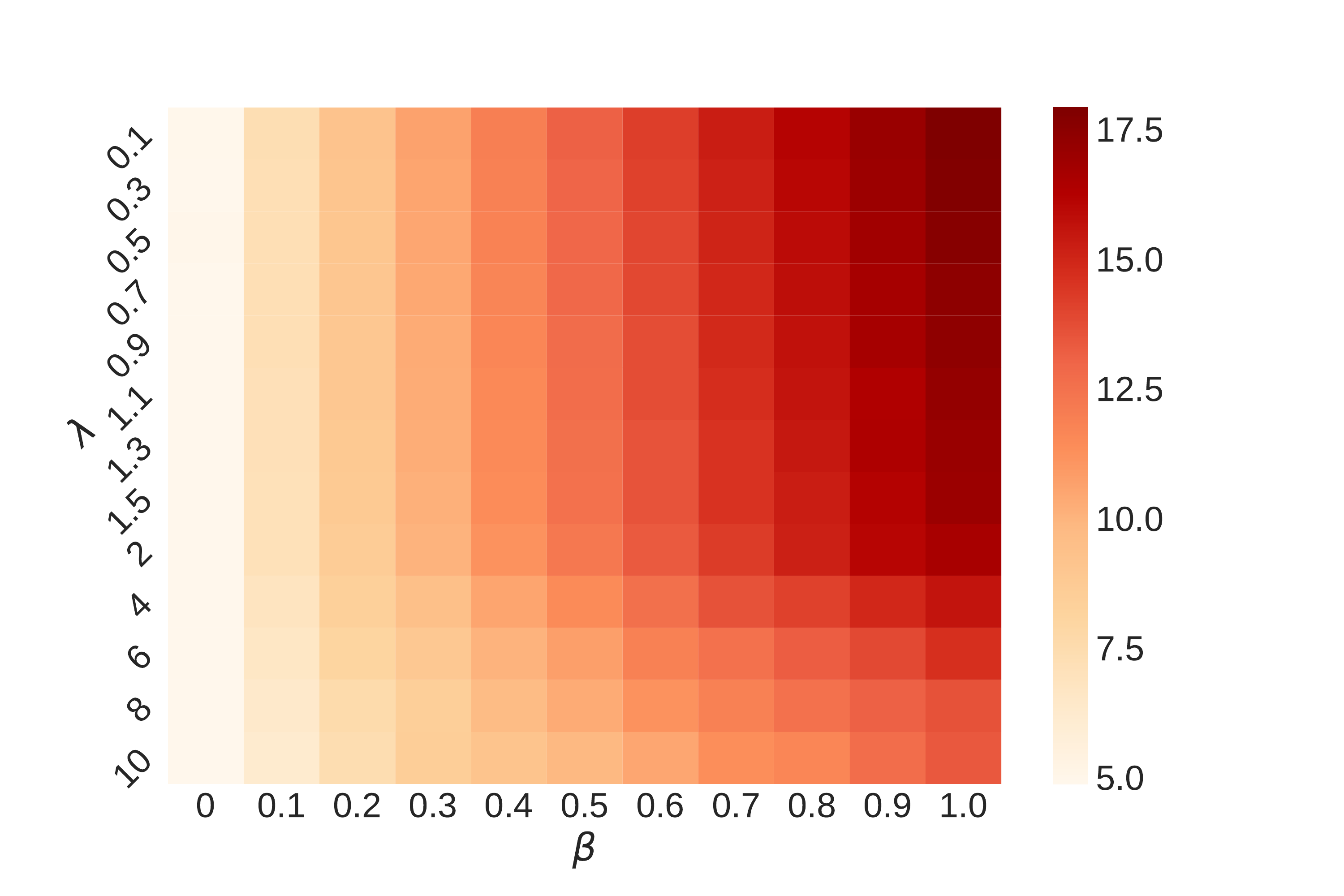}
	\end{tabular}
	\caption{Underestimated $\mathbf{z}$, $\hat{\lambda}=0.3$, $\hat{\beta}=0.8$. The average RMSE across different values of actual $\lambda$ and $\beta$ on boston dataset.  From left to right: \textit{MLSG}, \textit{Lasso}, \textit{Ridge}, \textit{OLS}.}
	\label{fig:boston_underestimate}
\end{figure}

\subsection{Supplementary results for the PDF dataset}

\begin{figure}[H]
% \centering
	\begin{tabular}{cc}
		\includegraphics[width=1.6in]{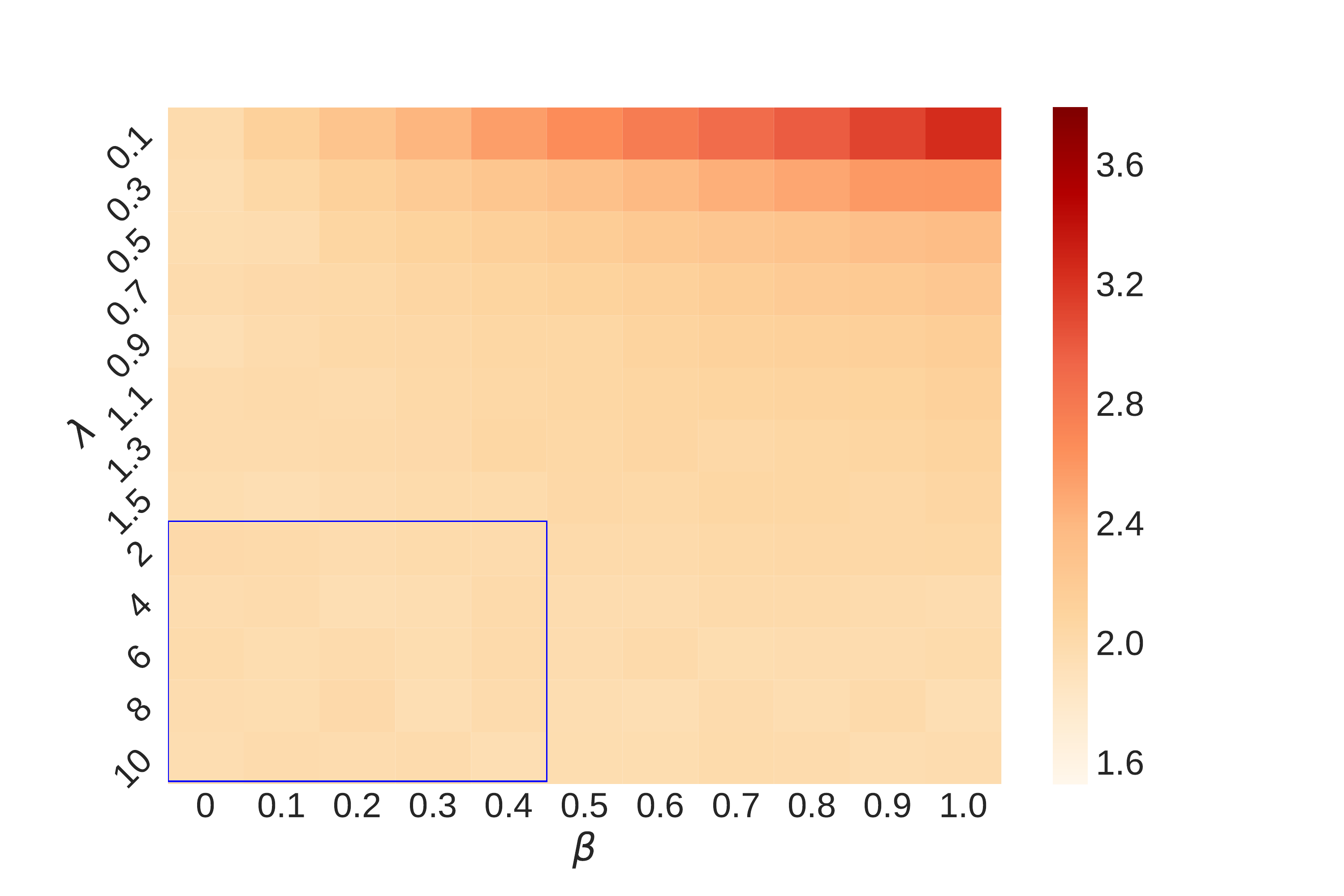} & \includegraphics[width=1.6in]{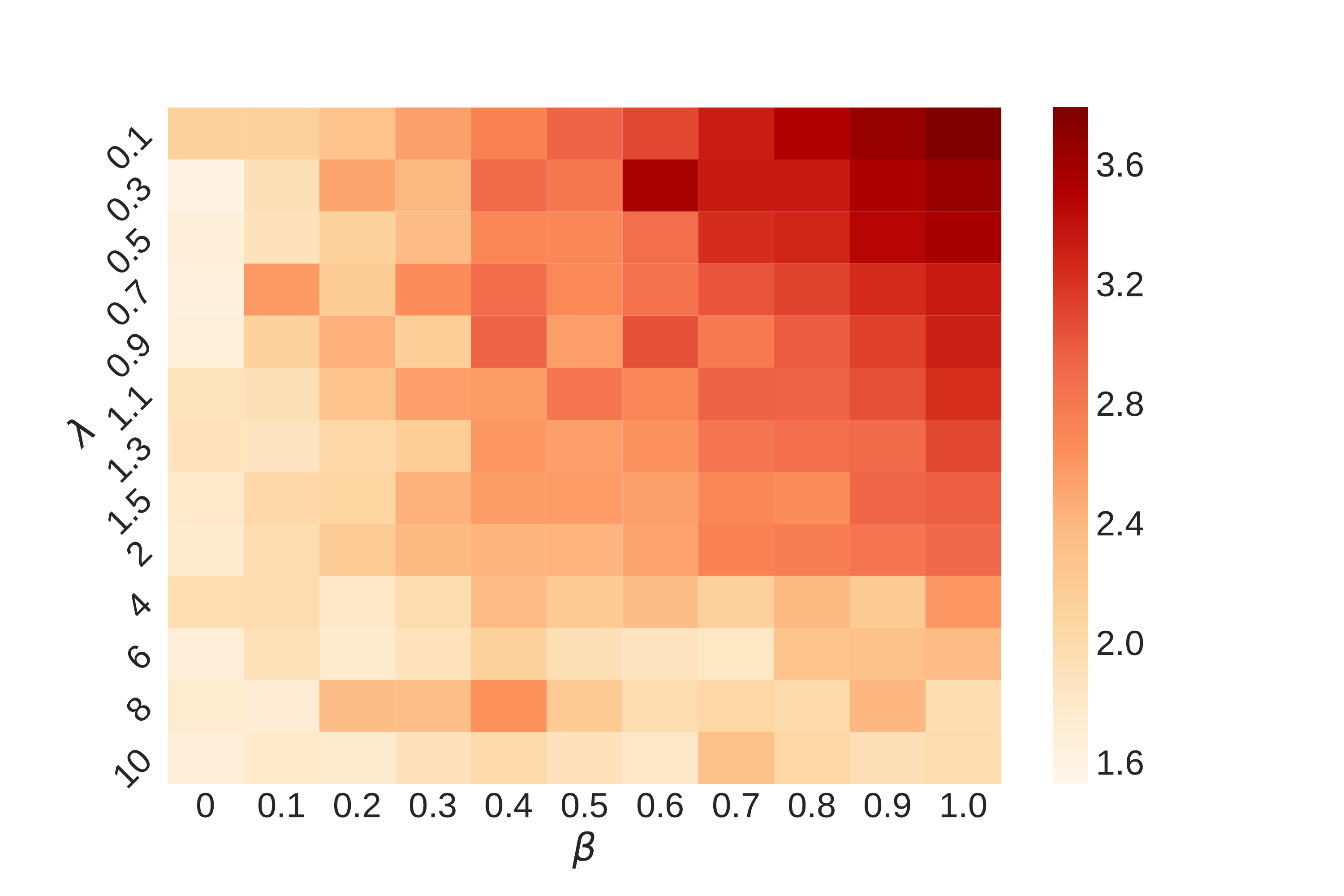}
	\end{tabular}
		\begin{tabular}{cc}
		\includegraphics[width=1.6in]{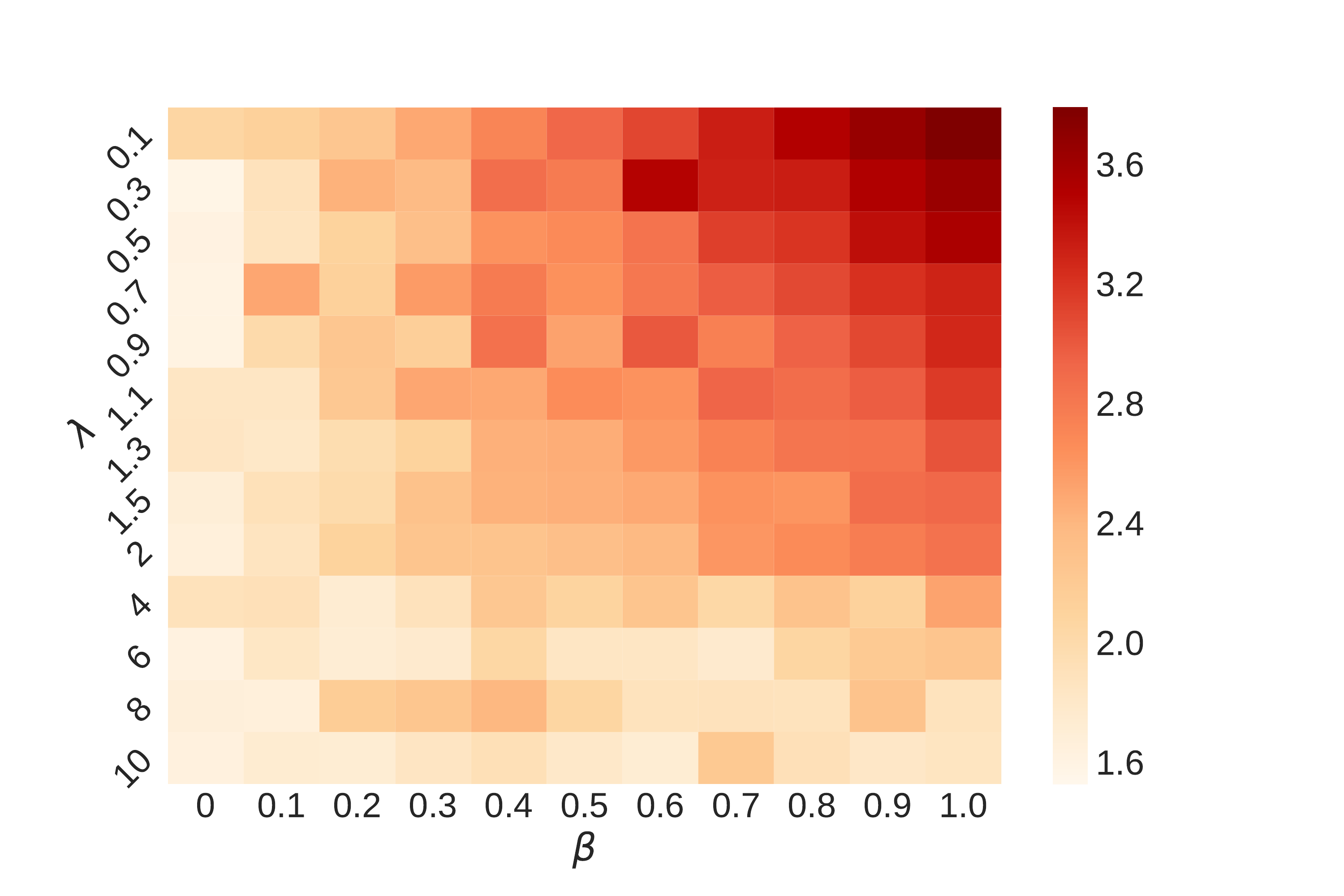} & \includegraphics[width=1.6in]{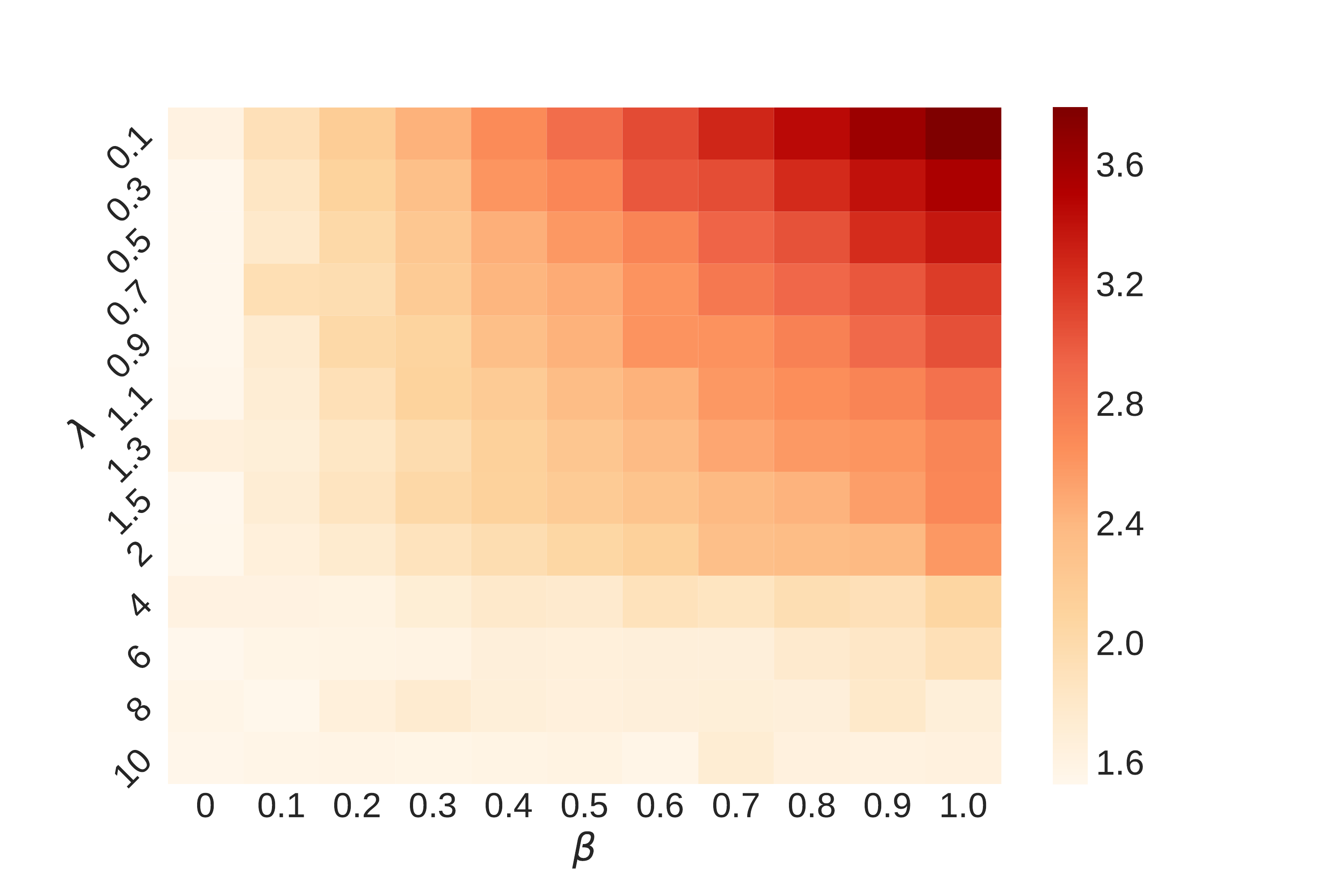}
	\end{tabular}
	\caption{Overestimated $\mathbf{z}$, $\hat{\lambda}=1.5$, $\hat{\beta}=0.5$. The average RMSE across different values of actual $\lambda$ and $\beta$ on PDF dataset.  From left to right: \textit{MLSG}, \textit{Lasso}, \textit{Ridge}, \textit{OLS}.}
	\label{fig:PDF_overestimate}
\end{figure}

\end{document}